%% file: neurips_2025.tex
\documentclass{article}

\PassOptionsToPackage{numbers, compress}{natbib}

\usepackage[final]{neurips_2025}

\usepackage[utf8]{inputenc} % allow utf-8 input
\usepackage[T1]{fontenc}    % use 8-bit T1 fonts
\usepackage{hyperref}       % hyperlinks
\usepackage{url}            % simple URL typesetting
\usepackage{booktabs}       % professional-quality tables
\usepackage{amsfonts}       % blackboard math symbols
\usepackage{nicefrac}       % compact symbols for 1/2, etc.
\usepackage{microtype}      % microtypography
\usepackage{xcolor}         % colors

\usepackage[utf8]{inputenc}
\usepackage{wrapfig}
\usepackage{amsmath}
\usepackage{amssymb}
\usepackage{multirow}
\usepackage{mathtools}
\usepackage{amsthm}
\usepackage{xcolor}
\usepackage{colortbl}
\usepackage{enumitem}
\usepackage{pifont}% http://ctan.org/pkg/pifont
\newcommand{\cmark}{\ding{51}}%
 \newcommand{\ind}{\perp\!\!\!\!\perp} 

\input{customized_math_symbols}

\usepackage[textwidth=1.7cm]{todonotes}

\def\lance#1{\textcolor{black}{#1}}

\newtheorem{theorem}{Theorem}[section]
\newtheorem{proposition}[theorem]{Proposition}

\title{Beyond Masked and Unmasked: \\Discrete Diffusion Models via Partial Masking}

\author{%
  Chen-Hao Chao\textsuperscript{1}, Wei-Fang Sun\textsuperscript{2}, Hanwen Liang\textsuperscript{1}, Chun-Yi Lee\textsuperscript{3}, Rahul G. Krishnan\textsuperscript{1}\\
  \textmd{\textsuperscript{1} University of Toronto, Vector Institute}\\
  \textmd{\textsuperscript{2} NVIDIA AI Technology Center}\\
  \textmd{\textsuperscript{3} National Taiwan University}\\
  \texttt{\{chchao, rahulgk\}@cs.toronto.edu},\,\,\, \texttt{cylee@csie.ntu.edu.tw}
}

\begin{document}

\maketitle

% \vspace{-0.775em}
\begin{abstract}
% \vspace{-0.1em}
Masked diffusion models (MDM) are powerful generative models for discrete data that generate samples by progressively unmasking tokens in a sequence. Each token can take one of two states: \textit{masked} or \textit{unmasked}. We observe that token sequences often remain unchanged between consecutive sampling steps; consequently, the model repeatedly processes identical inputs, leading to redundant computation. To address this inefficiency, we propose the \underline{P}a\underline{r}t\underline{i}al \underline{m}asking schem\underline{e} (Prime), which augments MDM by allowing tokens to take intermediate states interpolated between the masked and unmasked states. This design enables the model to make predictions based on partially observed token information, and facilitates a fine-grained denoising process. We derive a variational training objective and introduce a simple architectural design to accommodate intermediate-state inputs. Our method demonstrates superior performance across a diverse set of generative modeling tasks. On text data, it achieves a perplexity of 15.36 on OpenWebText, outperforming previous MDM (21.52), autoregressive models (17.54), and their hybrid variants (17.58), without relying on an autoregressive formulation. On image data, it attains competitive FID scores of 3.26 on CIFAR-10 and 6.98 on ImageNet-32, comparable to leading continuous generative models.
\end{abstract}
% \vspace{-0.3em}

\input{Sections/1_introduction}
\input{Sections/2_background}

\input{Sections/3_methodology}
\input{Sections/4_experiments}

\input{Sections/5_related_works}
\input{Sections/6_conclusion}
\bibliographystyle{unsrt}
\bibliography{citation}

%%%%%%%%%%%%%%%%%%%%%%%%%%%%%%%%%%%%%%%%%%%%%%%%%%%%%%%%%%%%

\newpage
\section*{NeurIPS Paper Checklist}

\begin{enumerate}

\item {\bf Claims}
    \item[] Question: Do the main claims made in the abstract and introduction accurately reflect the paper's contributions and scope?
    \item[] Answer: \answerYes{} % Replace by \answerYes{}, \answerNo{}, or \answerNA{}.
    \item[] Justification: The main claims presented in the abstract and the introduction section reflect the contributions of this paper. The key theoretical results are discussed in Section~\ref{sec:methodology}. The experiments in Section~\ref{sec:experiment} provide empirical justification for these claims. Finally, Section~\ref{sec:conclusion} summarizes both the theoretical and empirical contributions of this work.

\item {\bf Limitations}
    \item[] Question: Does the paper discuss the limitations of the work performed by the authors?
    \item[] Answer: \answerYes{} % Replace by \answerYes{}, \answerNo{}, or \answerNA{}.
    \item[] Justification: The limitations of this work are summarized in the Appendix~\ref{apx:limitation}. The discussion covers the assumptions made in this paper.

\item {\bf Theory assumptions and proofs}
    \item[] Question: For each theoretical result, does the paper provide the full set of assumptions and a complete (and correct) proof?
    \item[] Answer: \answerYes{} % Replace by \answerYes{}, \answerNo{}, or \answerNA{}.
    \item[] Justification: The assumptions and proofs of our theoretical results are presented in detail in Appendices~\ref{apx:mdm} and \ref{apx:analysis}.

    \item {\bf Experimental result reproducibility}
    \item[] Question: Does the paper fully disclose all the information needed to reproduce the main experimental results of the paper to the extent that it affects the main claims and/or conclusions of the paper (regardless of whether the code and data are provided or not)?
    \item[] Answer: \answerYes{} % Replace by \answerYes{}, \answerNo{}, or \answerNA{}.
    \item[] Justification: This paper discloses the information needed to reproduce the experimental results. The experimental configurations, detailed hyperparameter setups, and hardware requirements for the experiments are elaborated in Appendix~\ref{apx:setups}.

\item {\bf Open access to data and code}
    \item[] Question: Does the paper provide open access to the data and code, with sufficient instructions to faithfully reproduce the main experimental results, as described in supplemental material?
    \item[] Answer: \answerYes{} % Replace by \answerYes{}, \answerNo{}, or \answerNA{}.
    \item[] Justification: All the data used in this paper are publicly available. The code and installation instructions are available in an anonymous repository, with the link provided in Appendix~\ref{apx:setups}.

\item {\bf Experimental setting/details}
    \item[] Question: Does the paper specify all the training and test details (e.g., data splits, hyperparameters, how they were chosen, type of optimizer, etc.) necessary to understand the results?
    \item[] Answer: \answerYes{} % Replace by \answerYes{}, \answerNo{}, or \answerNA{}.
    \item[] Justification: The experimental configurations, detailed hyperparameter setups, and hardware requirements are elaborated in Appendix~\ref{apx:setups}.

\item {\bf Experiment statistical significance}
    \item[] Question: Does the paper report error bars suitably and correctly defined or other appropriate information about the statistical significance of the experiments?
    \item[] Answer: \answerYes{} % Replace by \answerYes{}, \answerNo{}, or \answerNA{}.
    \item[] Justification: The results in Fig.~\ref{fig:empty_step_analytic} show the mean values with shaded regions representing the variances. The caption of Table~\ref{tab:time_eval} offers the confidence interval. Results in Tables~\ref{tab:experiment:zeroshot}-\ref{tab:experiment:benchmark_imagenet} follow prior work by reporting performance from a single run, due to the high computational cost of training multiple variants.
     % \ref{tab:experiment:nll}, \ref{tab:experiment:benchmark_cifar}, and 

\item {\bf Experiments compute resources}
    \item[] Question: For each experiment, does the paper provide sufficient information on the computer resources (type of compute workers, memory, time of execution) needed to reproduce the experiments?
    \item[] Answer: \answerYes{} % Replace by \answerYes{}, \answerNo{}, or \answerNA{}.
    \item[] Justification: The hardware requirements (e.g., the computational hardware configurations and the execution time) for each experiment are elaborated in Appendix~\ref{apx:setups}.
    
\item {\bf Code of ethics}
    \item[] Question: Does the research conducted in the paper conform, in every respect, with the NeurIPS Code of Ethics \url{https://neurips.cc/public/EthicsGuidelines}?
    \item[] Answer: \answerYes{} % Replace by \answerYes{}, \answerNo{}, or \answerNA{}.
    \item[] Justification: All authors have reviewed the NeurIPS Code of Ethics and confirmed that the research conducted in this paper complies with it.

\item {\bf Broader impacts}
    \item[] Question: Does the paper discuss both potential positive societal impacts and negative societal impacts of the work performed?
    \item[] Answer: \answerYes{} % Replace by \answerYes{}, \answerNo{}, or \answerNA{}.
    \item[] Justification: This paper discusses its potential impacts in Appendix~\ref{apx:impacts}.
    
\item {\bf Safeguards}
    \item[] Question: Does the paper describe safeguards that have been put in place for responsible release of data or models that have a high risk for misuse (e.g., pretrained language models, image generators, or scraped datasets)?
    \item[] Answer: \answerNA{} % Replace by \answerYes{}, \answerNo{}, or \answerNA{}.
    \item[] Justification: The paper poses no risk for misuse.

\item {\bf Licenses for existing assets}
    \item[] Question: Are the creators or original owners of assets (e.g., code, data, models), used in the paper, properly credited and are the license and terms of use explicitly mentioned and properly respected?
    \item[] Answer: \answerYes{} % Replace by \answerYes{}, \answerNo{}, or \answerNA{}.
    \item[] Justification: The creators of assets are properly credited through citations, and the license is included in the asset (see Appendix~\ref{apx:setups}).

\item {\bf New assets}
    \item[] Question: Are new assets introduced in the paper well documented and is the documentation provided alongside the assets?
    \item[] Answer: \answerYes{} % Replace by \answerYes{}, \answerNo{}, or \answerNA{}.
    \item[] Justification: The assets (e.g., code, installation instructions, and running commands) are summarized in an anonymous repository, with the link provided in Appendix~\ref{apx:setups}. The datasets are publicly available, and the experiments are performed on them with the default setup.

\item {\bf Crowdsourcing and research with human subjects}
    \item[] Question: For crowdsourcing experiments and research with human subjects, does the paper include the full text of instructions given to participants and screenshots, if applicable, as well as details about compensation (if any)? 
    \item[] Answer: \answerNA{} % Replace by \answerYes{}, \answerNo{}, or \answerNA{}.
    \item[] Justification: The paper does not involve crowdsourcing or research with human subjects.

\item {\bf Institutional review board (IRB) approvals or equivalent for research with human subjects}
    \item[] Question: Does the paper describe potential risks incurred by study participants, whether such risks were disclosed to the subjects, and whether Institutional Review Board (IRB) approvals (or an equivalent approval/review based on the requirements of your country or institution) were obtained?
    \item[] Answer: \answerNA{} % Replace by \answerYes{}, \answerNo{}, or \answerNA{}.
    \item[] Justification: This paper does not involve research with human subjects.

\item {\bf Declaration of LLM usage}
    \item[] Question: Does the paper describe the usage of LLMs if it is an important, original, or non-standard component of the core methods in this research? Note that if the LLM is used only for writing, editing, or formatting purposes and does not impact the core methodology, scientific rigorousness, or originality of the research, declaration is not required.
    %this research? 
    \item[] Answer:\answerNA{} % Replace by \answerYes{}, \answerNo{}, or \answerNA{}.
    \item[] Justification: LLMs are only used for editing grammar.
    % and verifying the correctness of mathematical proofs.

\end{enumerate}

\newpage
\input{Sections/a0_appendix}

\end{document}

%% file: customized_math_symbols.tex
\usepackage{bm}
\newcommand{\vx}{\bm{x}}
\newcommand{\vy}{\bm{y}}

\newcommand{\X}{\mathcal{X}}
\newcommand{\Y}{\mathcal{Y}}

\newcommand{\V}{\mathcal{V}}

\newcommand{\E}{\mathbb{E}}
\newcommand{\N}{\mathbb{N}}
\newcommand{\R}{\mathbb{R}}

%% file: Sections/1_introduction.tex
\section{Introduction}
\label{sec:introduction}
\input{Figures/Introduction/empty_steps}
Discrete data generation has been a central focus of probabilistic modeling since the early development of neural networks~\cite{sutskever2011textrnn, mikolov2011subword, graves2013rnntext}. The goal is to build a model capable of generating sequences of symbolic units, known as \textit{tokens}, which can represent words in natural language data or pixels in images. The field has been largely shaped by autoregressive models (ARM) (e.g., \cite{vaswani2017transformer,radford2019language,touvron2023llama}), which capture the distribution of sequence data by factorizing it according to a prespecified left-to-right order. Recently, promising results from~\cite{lou2024sedd, nie2025llmdiff} have demonstrated that order-agnostic models, such as masked diffusion models (MDM) (e.g., \cite{lou2024sedd, nie2025llmdiff, sahoo2024simplifieddiff, shi2024simplifieddiff}), can be effectively extended to large-scale generation tasks, opening a new venue for discrete generative modeling.

Masked diffusion models are latent-variable generative models that introduce noise by progressively masking tokens within a sequence over time. During the reverse diffusion process, masked tokens are incrementally unmasked according to a predefined ratio. Each token takes one of two states, \textit{masked} or \textit{unmasked}. This binary representation introduces an inefficiency; the entire sequence often remains unchanged across consecutive sampling steps, causing the model to repeatedly process the same input. This phenomenon is illustrated in the green curve in Fig.~\ref{fig:intro:empty_steps}, which quantifies the number of such \textit{idle steps} by simulating the reverse diffusion process of an MDM~\cite{sahoo2024simplifieddiff}. The figure shows that 37\% of the 1,024 steps produce no update to the sequence during the reverse diffusion process. This inefficiency motivates our investigation of redefining the diffusion process to transform these idle steps into informative updates for improving the utilization of the model during generation.

 % (also see the theoretical derivation of this number in Appendix~\ref{apx:mdm:idle})

\input{Figures/Introduction/pm}
We propose a simple yet effective solution that allows each token to take an \textit{intermediate} state, which represents the interpolation between the masked and unmasked states, in the diffusion process. We refer to this method as the \underline{P}a\underline{r}t\underline{i}al \underline{m}asking schem\underline{e} (Prime), and denote MDM augmented with Prime as MDM-Prime. Prime represents each token as a sequence of sub-tokens using a base-$b$ encoding, with masking performed at the sub-token level. Since sub-tokens can be masked independently during the forward diffusion process, this method introduces intermediate states that partially reveal token information. An illustrative example is shown in the top of Fig.~\ref{fig:intro:pm}, where unmasked states `0-3' are first encoded as `00-11', and the intermediate states are obtained by masking one of the sub-tokens in the sub-token sequence. The intermediate states enable MDM-Prime to perform a fine-grained denoising process. For example, as illustrated in the `State Transition Tree' of Fig.~\ref{fig:intro:pm}~(b), a four-choice prediction can be decomposed into two binary decisions during the sampling process (e.g., \texttt{mm}$\to$\texttt{m}1$\to$11). Compared to standard MDM transitions (i.e., Fig.~\ref{fig:intro:pm}~(a)), MDM-Prime is capable of making predictions based on partially observed token information while deferring the final token revelation until later sampling steps. With its ability to transition through intermediate states, MDM-Prime demonstrates improved model utilization during sampling, as reflected in the reduced number of idle steps in Fig.~\ref{fig:intro:empty_steps} (i.e., the purple curves). This in turn leads to enhanced performance, as later presented in Section~\ref{sec:experiment}. The contributions of this work are as follows:

\begin{itemize}[leftmargin=12pt]
\vspace{-0.5em}
\item We propose MDM-Prime, a generalized MDM framework that enables intermediate token transitions. This framework can be optimized through a variational upper bound, which approximates the negative log-likelihood.
\item We present a simple implementation of MDM-Prime built upon the standard MDM architecture. Our design requires only minor modifications to the input embedding layer of the standard MDM.
\item We demonstrate that MDM-Prime achieves superior performance on both image and text generation tasks. On the OpenWebText dataset~\cite{gokaslan2019owt}, MDM-Prime attains an evaluation perplexity of 15.36, outperforming ARM (17.54), MDM~\cite{lou2024sedd, sahoo2024simplifieddiff, shi2024simplifieddiff, xu2025ebmdiff} (21.52), and their hybrid variants~\cite{xu2025ebmdiff, arriola2025block} (17.58). To the best of our knowledge, this is the first MDM-based approach to surpass ARM without relying on the autoregressive formulation. Furthermore, MDM-Prime achieves FID scores~\cite{heusel2017fid} of 3.26 on CIFAR-10~\cite{krizhevsky2009cifar10} and 6.98 on ImageNet-32~\cite{chrabaszcz2017imagenet}, demonstrating competitive performance comparable to leading continuous generative modeling methods~\cite{karras2020stylegan, ho2020ddpm, song2021scoreflow}.
% \vspace{-0.5em}
\end{itemize}

%% file: Figures/Introduction/empty_steps.tex
\begin{wrapfigure}[]{r}{0.43\linewidth}
    \centering
    \vspace{-1.6em}
    \footnotesize
    \includegraphics[width=\linewidth]{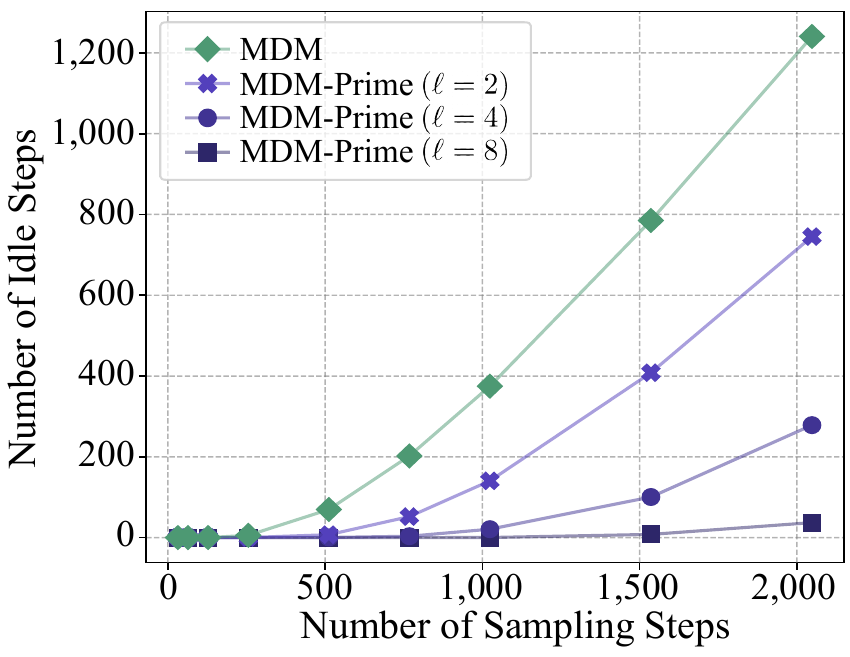}
    \vspace{-2em}
    \caption{Number of idle steps during the reverse diffusion processes of MDM and MDM-Prime. The results are averaged over ten runs. $\ell$ is the sub-token sequence length.}
    \vspace{-1.7em}
    \label{fig:intro:empty_steps}
\end{wrapfigure}

%% file: Figures/Introduction/pm.tex
\begin{figure}[t]
    \centering
    \footnotesize
    \vspace{-1em}
    \includegraphics[width=\linewidth]{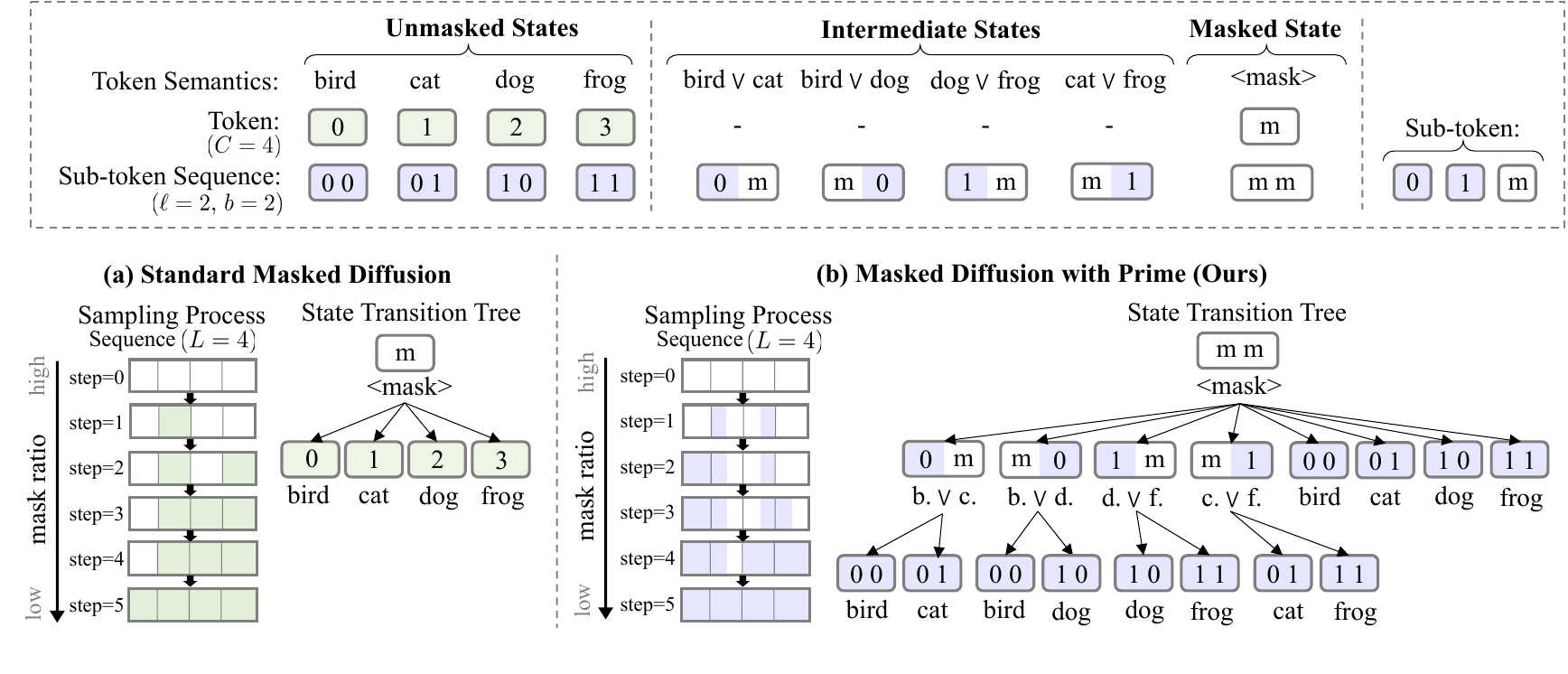}
    \vspace{-2.5em}
    \caption{An illustrative example of (a) standard MDM and (b) MDM-Prime. Each token and its corresponding sub-token sequence (constructed via base-$b$ encoding) can take one of three states: \textit{unmasked}, \textit{masked}, or \textit{intermediate}. The masked and intermediate states serve as the latent representations produced by the forward diffusion process. This example contains $C = 4$ possible token classes, labeled as `bird’, `cat’, `dog’, and `frog’. $\ell = 2$ indicates that each token is represented using two sub-tokens, and $b = \sqrt[\ell]{C} = 2$ denotes the number of classes per sub-token. The symbol \texttt{m} represents a masked token or a masked sub-token. The bottom-right sections of (a) and (b) illustrate the state transition trees. MDM-Prime supports transitions through intermediate states while retaining the ability to directly reach unmasked states. The bottom-left portions depict the sampling process for a token sequence of length $L = 4$. In (a), an idle step occurs between steps 3 and 4. In contrast, (b) demonstrates a sampling process without idle steps, which leads to improved model utilization.}
    \label{fig:intro:pm}
    \vspace{-0.75em}
\end{figure}

%% file: Sections/2_background.tex
\section{Background}
\label{sec:background}
We first provide background on continuous-time masked diffusion processes. Section~\ref{sec:background:forward} describes the forward process, while Section~\ref{sec:background:reverse} elaborates on the reverse process.

\subsection{Forward Diffusion Process}
\label{sec:background:forward}
Let $\X \triangleq \{0,\cdots,C-1\}$ denote a set of $C$ tokens, and let $\vx = [x^1, \cdots, x^L] \in \X^L$ be a sequence of tokens with length $L$, where the superscript indicates the index of each token. Let $\vx_0$ denote the sample drawn from a data distribution $p_{\text{data}}$, which is defined as a probability mass function (pmf) over $\mathcal{X}^L$. Given a continuous time variable $t \in [0,1]$, the latent variable introduced by the forward diffusion process is denoted as $\vx_t\in (\X \cup \{\texttt{m}\} )^L \triangleq \tilde{\X}^L$, where $\texttt{m}$ represents the masked token. In addition, let $\delta_{x'}(x)$ be the Kronecker delta function, which equals $1$ if $x = x'$ and $0$ otherwise. The forward diffusion process is performed through an element-wise conditional sampler $q(\vx_{t}|\vx_0)=\prod_{i=1}^L q(x_{t}^i|x^i_0)$ constructed by interpolating between $\delta_{\texttt{m}}(\cdot)$ and $\delta_{x_0^i}(\cdot)$, defined as follows~\cite{sahoo2024simplifieddiff, shi2024simplifieddiff, austin2021structurediff}:
\input{Equations/Background/diffusion_kernel_t0}%
where $\alpha_t \in [0,1]$ is a strictly decreasing scheduling function with boundary conditions $\alpha_0 \approx 1$ and $\alpha_1 \approx 0$. Intuitively, each perturbed token $x_t^i$ retains the original value $x_0^i$ with probability $\alpha_t$, and is replaced by $\texttt{m}$ with probability $1-\alpha_t$. At time $t=1$, the latent variable $\vx_1=[\texttt{m},\cdots,\texttt{m}]$ consists entirely of masked tokens, exhibiting no randomness and revealing no information about $\vx_0$. This process can also be interpreted as decreasing the mutual information between $x^i_t$ and $x^i_0$ over time according to $\alpha_t$~\cite{dickstein2015diffusion,austin2021structurediff} (i.e., $I(x^i_t;x^i_0)= \alpha_t H(x^i_0)$, where $H(x^i_0)$ denotes the entropy of $x^i_0$) as explained in Appendix~\ref{apx:mdm:mutual_information}. Based on Eq.~(\ref{eq:background:diffusion_kernel_t0}), the forward diffusion process can be accomplished by first drawing a data sample $\vx_0\sim p_\text{data}(\cdot)$ and then applying masking to obtain $\vx_{t}\sim q(\cdot|\vx_0)$.

\subsection{Reverse Diffusion Process} 
\label{sec:background:reverse}
Let $s$ and $t$ be two time variables that satisfy $0 \leq s < t \leq 1$. The reverse diffusion process is performed by iterating through $p(\vx_{s}|\vx_t)$, starting from $\vx_1$. The distribution $p(\vx_{s}|\vx_t)$ can be derived using the conditional distributions $q(\vx_{t}|\vx_{s})$ and $q(\vx_{s}|\vx_{t},\vx_{0})$. In particular, the transition distribution $q(\vx_{t}|\vx_{s})=\prod_{i=1}^L q(x_{t}^i|x^i_s)$ is defined to be \textit{absorbing}~\cite{austin2021structurediff} on the masked state (i.e., a masked token remains masked from $s$ to $t$), and is derived from Eq.~(\ref{eq:background:diffusion_kernel_t0}) as follows~\cite{shi2024simplifieddiff}:
\input{Equations/Background/diffusion_kernel_ts}Based on Eqs.~(\ref{eq:background:diffusion_kernel_t0}) and (\ref{eq:background:diffusion_kernel_ts}), the posterior distribution $q(\vx_{s}|\vx_{t},\vx_{0})=\prod_{i=1}^L q(x^i_{s}|x^i_{t},x^i_{0})$ can be derived using Bayes’ rule and the Markov property of the diffusion process, and is expressed as~\cite{sahoo2024simplifieddiff,shi2024simplifieddiff}:
\input{Equations/Background/diffusion_kernel_st0}%
Eq.~(\ref{eq:background:diffusion_kernel_st0}) indicates that, given $x^i_{0}$ and observing a masked token $x^i_t=\texttt{m}$, the reverse process transitions the masked token to its original value $x^i_{0}$ with probability $\frac{\alpha_{s}-\alpha_{t}}{1-\alpha_t}$ or retains its value with probability $\frac{1-\alpha_{s}}{1-\alpha_t}$. For the unmasked token $x^i_t\in \X$, its value remains unchanged for the remaining steps.

Since $p(\vx_{s}|\vx_t)=\E_{p(\vx_0|\vx_t)}[q(\vx_{s}|\vx_{t},\vx_{0})]$~\cite{austin2021structurediff}, \lance{the reverse diffusion process can be performed by first drawing $\vx_0\sim p(\cdot|\vx_t)$ and then sampling $\vx_{s} \sim q(\cdot |\vx_{t},\vx_{0})$}. Many recent works~\cite{nie2025llmdiff, sahoo2024simplifieddiff, shi2024simplifieddiff, austin2021structurediff, campbell2022ctmc, gat2024dfm, zheng2024mdm} choose to model $p(\vx_0|\vx_t)$ as $p_\theta(\vx_0|\vx_t)$, where $\theta$ denotes the model parameters. To facilitate computational efficiency, this distribution is typically factorized as $p_\theta(\vx_0|\vx_t)=\prod_{i=1}^L p_\theta (x_0^i|\vx_t)$~\cite{nie2025llmdiff, sahoo2024simplifieddiff, shi2024simplifieddiff, austin2021structurediff, campbell2022ctmc, gat2024dfm, zheng2024mdm}. The parameter $\theta$ can be optimized by estimating the negative log-likelihood $-\log p_\theta(\vx_0)$ using a variational upper bound, written as~\cite{sahoo2024simplifieddiff, shi2024simplifieddiff}:
\input{Equations/Background/diffusion_elbo}%
where $\alpha_t'=\frac{d}{dt} \alpha_t$. This objective specifies a cross-entropy loss, weighted by the coefficient $\frac{\alpha'_t}{1-\alpha_t}$, and satisfies $\mathcal{L}_\text{vb}(\boldsymbol{x}_0;\theta) \geq -\log p_\theta(\vx_0)$~\cite{sahoo2024simplifieddiff, shi2024simplifieddiff}. In addition, since unmasked elements in $\vx_t$ retain their values over time according to Eq.~(\ref{eq:background:diffusion_kernel_st0}), the \textit{carry-over} parameterization~\cite{sahoo2024simplifieddiff} (also referred to as the \textit{mean} parameterization in~\cite{shi2024simplifieddiff}) can be applied by explicitly setting the corresponding entries in $\vx_0$ to the unmasked values in $\vx_t$. Formally, this is expressed as $p_\theta(x_0^i | \vx_t) \triangleq \delta_{x_t^i}(x_0^i)$ for position $i$ where $x_t^i \in \X$. This technique ensures that the model avoids redundant predictions on unmasked elements, and therefore can effectively reduce $\mathcal{L}_\text{vb}$~\cite{sahoo2024simplifieddiff}.

%% file: Equations/Background/diffusion_kernel_t0.tex
\begin{equation}
\label{eq:background:diffusion_kernel_t0}
\begin{aligned}
    q(x^i_{t}|x^i_0)= (1-\alpha_t)\delta_{\texttt{m}} (x^i_{t})+\alpha_t \delta_{x^i_0}(x^i_{t}),
\end{aligned}
\end{equation}

%% file: Equations/Background/diffusion_kernel_ts.tex
\begin{equation}
\label{eq:background:diffusion_kernel_ts}
\begin{aligned}
    q(x^i_t|x^i_{s})=\begin{cases}
    \frac{\alpha_{s}-\alpha_t}{\alpha_{s}} \delta_{\texttt{m}}(x^i_{t})+\frac{\alpha_t}{\alpha_{s}}\delta_{x^i_{s}}(x^i_{t})& \text{if }x^i_s \in \X,\\
    \delta_{x^i_s}(x^i_{t}) & \text{if } x^i_s = \texttt{m}.  
    \end{cases}    
\end{aligned}
\end{equation}

%% file: Equations/Background/diffusion_kernel_st0.tex
\begin{equation}
\label{eq:background:diffusion_kernel_st0}
\begin{aligned}
    q(x^i_{s}|x^i_{t},x^i_{0})= \begin{cases}
      \delta_{x^i_t}(x^i_{s}) & \text{if } x^i_t \in \X,\\
      \frac{1-\alpha_{s}}{1-\alpha_t}\delta_{\texttt{m}}(x^i_{s})+\frac{\alpha_{s}-\alpha_{t}}{1-\alpha_t} \delta_{x^i_0}(x^i_{s})& \text{if }x^i_t = \texttt{m}.
    \end{cases}    
\end{aligned}
\end{equation}

%% file: Equations/Background/diffusion_elbo.tex
\begin{equation}
\label{eq:background:diffusion_elbo}
\begin{aligned}
    \mathcal{L}_\text{vb}(\vx_0;\theta)=\int_0^1 \frac{\alpha'_t}{1-\alpha_t} \mathbb{E}_{q(\vx_t|\vx_0)}\left[\sum_{i=1}^L \log p_\theta (x_0^i|\vx_t) \right] dt,
\end{aligned}
\end{equation}

% \mathbf{1}_{x_t^{i}= \texttt{m}}

%% file: Sections/3_methodology.tex
\section{Methodology}
\label{sec:methodology}
In this section, we introduce a new MDM framework that incorporates the \underline{P}a\underline{r}t\underline{i}al \underline{m}asking schem\underline{e} (Prime), referred to as MDM-Prime. We begin by introducing an invertible function for constructing sub-token sequences and define a novel masked diffusion process over these sub-tokens, which enables intermediate state transitions and reduces the number of idle steps (Section~\ref{sec:methodology:diff_pm}). We then propose a novel parameterization that captures sub-token dependencies via joint probability distributions, and describe a simple model architecture design that operates over sub-token inputs (Section~\ref{sec:methodology:parameterization}).

\subsection{Discrete Diffusion via Partial Masking}
\label{sec:methodology:diff_pm}
The core idea of Prime is to represent each token $x^i_0$ with a sub-token sequence $\vy^i_0=[y^{i,1}_0,\cdots,y^{i,\ell}_0]$, allowing the creation of intermediate states during the element-wise masking of the forward diffusion process (i.e., Eq.~(\ref{eq:background:diffusion_kernel_t0})). Given a target length $\ell>1$, this can be achieved through the use of an invertible function $f$, which maps each token $x^i_0 \in \X$ to its base-$b$ encoding $\vy^i_0 \in \Y^\ell$ (i.e., a sequence of $\ell$ sub-tokens), where $\Y \triangleq \{0, \cdots, b-1\}$ denotes the set of sub-tokens with $b=\lceil\sqrt[\ell]{C}\rceil\in \N$.

\input{Figures/Methodology/invertible_func}
The invertibility of $f$ enables the likelihood to be defined on the set of sub-tokens $\Y$. For convenience, we slightly abuse notation and extend $f$ to accept vector inputs: $f(\vx_0)\triangleq [f(x^1_0), \cdots, f(x^L_0)]=[\vy^1_0,\cdots,\vy^L_0]=[(y^{1,1}_0,\cdots,y^{1,\ell}_0),\cdots,(y^{L,1}_0,\cdots,y^{L,\ell}_0)]=\vy_0$. Its inverse is denoted as $f^{-1}$. Under this transformation, the pmf of the data, $p_{\text{data}}(\vx_0)$, can be equivalently written as $p_{\text{data}} \circ f^{-1}(\vy_0)$ according to the change-of-variable principle. Rather than directly modeling $\vx_0$, we model $\vy_0$ through MDM, and reconstruct $\vx_0$ using $f^{-1}(\vy_0)$, resulting in a parameterized pmf $p_\theta(\vy_0)=p_\theta \circ f(\vx_0)$. An illustrative example is provided in Fig.~\ref{fig:methodology:invertible_func}.

To learn $p_\theta(\vy_0)$ using MDM, each data point $\vx_0$ is first transformed into $\vy_0 = f(\vx_0)$, and its corresponding latent variable $\vy_t$ is sampled from $q(\vy_{t}|\vy_0)=\prod_{i=1}^L\prod_{j=1}^{\ell} q(y_{t}^{i,j}|y^{i,j}_0)$ according to Eq.~(\ref{eq:background:diffusion_kernel_t0}) by substituting $x$'s with $y$'s. The reverse diffusion process is parameterized by $p_\theta(\vy_0 | \vy_t)=\prod_{i=1}^L p_\theta(\vy_0^i| \vy_t)$, and the model is trained to minimize the loss $\mathcal{L}_{\text{vb}}(\vy_0; \theta)$ as in Eq.~(\ref{eq:background:diffusion_elbo}):
\input{Equations/Methodology/subtoken_elbo}%
Section~\ref{sec:methodology:parameterization} provides details on the parameterization of $p_\theta(\vy_0^i | \vy_t)$, while Appendix~\ref{apx:analysis:nll} shows that Eq.~(\ref{eq:methodology:subtoken_elbo}) defines a variational bound that approximates negative log-likelihood (NLL).

This new diffusion process operates on an augmented set $\tilde{\Y} \triangleq \Y \cup \{\texttt{m}\}$. For each token, the number of intermediate states introduced by $f$ and the diffusion process is given by $|\tilde{\Y}^\ell| - |\tilde{\X}| = (b+1)^\ell - (C+1)$. \textbf{Proposition}~\ref{prop:num_state} ensures that this number is always positive. This property allows MDM-Prime to learn smooth transitions via a rich set of intermediate states. Moreover, \textbf{Proposition}~\ref{prop:idle_step} and Eq.~(\ref{eq:apx:sub_idle_step}) formally establish that the number of idle steps decreases as $\ell$ increases, which guarantees an improved model utilization of MDM-Prime during the reverse diffusion process.

\subsection{Parameterization}
\label{sec:methodology:parameterization}

In this section, we discuss our implementation for $p_\theta(\vy^{i}_0|\vy_{t})$, which comprises a decoder for modeling the distribution of $\vy^{i}_0$ and an encoder for processing the input $\vy_{t}$. We begin by outlining the decoder design, followed by a description of the encoder.

\input{Figures/Methodology/overview}
\paragraph{Decoder Design via Joint Probability.}~A straightforward approach to implementing $p_\theta(\vy^{i}_0|\vy_{t})$ is to factorize it as $\prod_{j=1}^{\ell} p_\theta(y^{i,j}_0|\vy_{t})$, modeling each component with a softmax distribution. While this factorization allows for easy application of the \textit{carry-over} parameterization~\cite{sahoo2024simplifieddiff,shi2024simplifieddiff} (discussed in Section~\ref{sec:background:reverse}), it introduces two key challenges: (1) an independence assumption and (2) potential generation of invalid samples. First, the factorized form $p_\theta(\vy^{i}_0|\vy_{t})=\prod_{j=1}^{\ell} p_\theta(y^{i,j}_0|\vy_{t})$ imposes an additional independence assumption across sub-tokens, preventing the model from capturing dependencies between them. As shown in Fig.~\ref{fig:methodology:toy}~(a), increasing the sub-token sequence length $\ell$ leads to a deterioration in the sampling distribution due to this limitation. Second, since $f$ is defined as an injective function but not necessarily bijective (i.e., $|f(\X)|\leq |\Y^\ell|$), some samples $\vy_0$ produced by such a factorized model may not correspond to any valid $\vx_0$ under the inverse mapping $f^{-1}$. For instance, when encoding the GPT-2~\cite{radford2019language} vocabulary of $C=50,257$ classes with $\ell=4$ and $b=\lceil \sqrt[\ell]{C}\rceil=15$, the model may generate invalid sub-token sequences such as $\vy^i_0=(14,14,14,14)$ during inference time, even though there is no corresponding $x^i_0=50,624$ for decoding.

To address these two issues, we propose to model the joint distributions $p_\theta(\vy^{i}_0|\vy_t)$ of a sequence of $\ell$ sub-tokens while explicitly zeroing out the probability mass assigned to invalid samples. This is achieved by parameterizing only the logits of base-$b$ encoding $\vy^i_0\in f(\X)$ that correspond to a valid $x_0^i\in \X$, which results in exactly $C$ entries in the logit outputs for each position $i\in\{1,\cdots,L\}$. 

Based on the above joint probability design, to further support carry-over parameterization for MDM-Prime, the element-wise distribution should be defined as $p_\theta(y^{i,j}_0|\vy_t) \triangleq \delta_{y^{i,j}_t}(y^{i,j}_0)$ for all position $i,j$ where $y^{i,j}_t \in \Y$ (following the end of Section~\ref{sec:background:reverse}). Since we parameterize the joint distribution as $p_\theta(\vy^{i}_0|\vy_t)=p_\theta(y^{i,1}_0,\cdots,y^{i,\ell}_0|\vy_{t})$, this condition is imposed on the marginal distribution as follows:
\input{Equations/Methodology/marginalization}%
To meet this condition, the probabilities of $\vy_0^i$ with any element $y_0^{i,j}$ that is inconsistent with $y_t^{i,j}$ should be explicitly set to zero. The parameterized probability can thus be defined as follows:
\input{Equations/Methodology/parameterization}%
where $\V(\vy_t^i)\triangleq \{ \vy^i=[y^{i,1}, \cdots, y^{i,\ell}]  \in f(\X) \,\,\,s.t.\,\,\, (y^{i,j} = y^{i,j}_t) \vee (y^{i,j}_t = \texttt{m})\}$ denotes a set of outputs that is consistent with $y_t^{i,j}$, and $E_\theta:\Y^{\ell}\times \tilde{\Y}^{\ell\times L}\to \R$ is a scalar logit. \textbf{Proposition}~\ref{prop:param_condition} guarantees the correctness of this parameterization (i.e., Eq.~(\ref{eq:methodology:parameterization}) satisfies Eq.~(\ref{eq:methodology:marginalization})). Consider a simple example where $C=7$, $\ell=3$, and $b=2$, then $\V(\vy_t^i)$ corresponds to the following sets:
\begin{itemize}[leftmargin=12pt]
\vspace{-0.5em}
\item If $\vy_t^i=(\texttt{m},\texttt{m},\texttt{m})$, then $\V(\vy_t^i)=\{(0,0,0),(0,0,1),(0,1,0),(0,1,1),(1,0,0),(1,0,1),(1,1,0)\}$
\item If $\vy_t^i=(0,\texttt{m},\texttt{m})$, then $\V(\vy_t^i)=\{(0,0,0),(0,0,1),(0,1,0),(0,1,1)\}$ 
\item If $\vy_t^i=(0,0,\texttt{m})$, then $\V(\vy_t^i)=\{(0,0,0),(0,0,1)\}$ 
\vspace{-0.5em}
\end{itemize}
Note that $(1,1,1)$ is invalid since $C=7$ in this case. Some illustrative examples of $p_\theta(\vy^i_0|\vy_{t})$ are provided in Fig.~\ref{fig:methodology:marginalization}. As the reverse diffusion process progresses, the number of unmasked sub-tokens in $\vy_t^i$ increases, leading to a substantial reduction in $|\V(\vy_t^i)|$. This results in a decreasing number of candidate classes of $\vy^i_0$ over time, and thus explicitly reducing uncertainty in the prediction task. From an implementation perspective, Eq.~(\ref{eq:methodology:parameterization}) can be efficiently derived using precomputed filters indexed by $y_t^{i,j}$. During the forward pass of the model, the logits of $\vy^{i}_0\notin \V(\vy_t^i)$ are excluded according to these filters. Further implementation details are provided in Appendix~\ref{apx:analysis:carry_over} and Fig.~\ref{fig:implementation}.

\input{Figures/Methodology/architecture_comp}
\paragraph{Encoder Design for Processing Sub-tokens.}~In contrast to the decoder, where the distribution over sub-tokens $\vy^i_0 \in f(\X)$ is represented jointly using logit outputs with $C$ entries (see Figs.~\ref{fig:methodology:marginalization}~and~\ref{fig:methodology:architecture_comp}), the encoder receives noised inputs $\vy_t^i$ that lie in the augmented set $\tilde{\Y}^\ell$. Since the size of this set, $|\tilde{\Y}^\ell|$, may grow with $\ell$ and typically exhibits $|\tilde{\Y}^\ell| \gg C$ (also see Appendix~\ref{apx:analysis:intermediate}), creating an embedding lookup table for $\vy_t^i$ is impractical due to the resulting growth in the number of parameters in it. To address this issue, we propose to model each sub-token embedding separately (i.e., creating a lookup table for individual $y_t^{i,j}\in\tilde{\Y}$), followed by a merging operation to produce a token embedding.

In our approach, a simple merging operation based on concatenation is employed. Let $D$ denote the dimensionality of the token embedding vector. Each sub-token is first embedded into a vector of size $D/\ell$, and the resulting $\ell$ embeddings are concatenated to form a $D$-dimensional token embedding vector. This token embedding can then be processed by an arbitrary downstream neural network, which allows us to reuse the standard MDM architecture. A comparison between this design and a standard MDM architecture is shown in Fig.~\ref{fig:methodology:architecture_comp}. Alternative merging strategies, such as the Perceiver~\cite{jaegle2021perceiver} cross-attention mechanism, are discussed and evaluated in Appendices~\ref{apx:architecture} and~\ref{apx:experiments:ablation}, where we show that simple concatenation yields the best performance.

In summary, adapting a standard MDM to MDM-Prime requires only minimal architectural modifications on the embedding layer. This simple strategy preserves the overall architectural design of the standard MDM, enabling a fair comparison with our baseline in the following experiments.

%% file: Figures/Methodology/invertible_func.tex
\begin{wrapfigure}[]{r}{0.48\linewidth}
    \centering
    \vspace{-1.6em}
    \footnotesize
    \includegraphics[width=\linewidth]{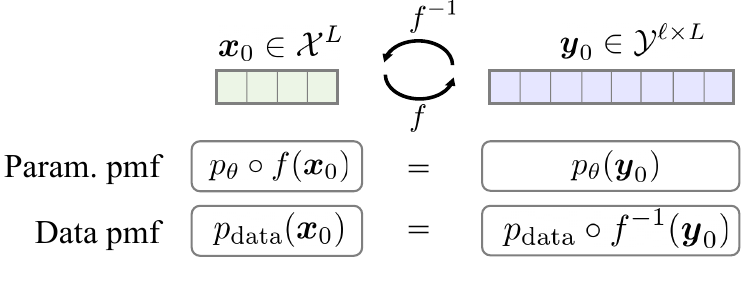}
    \vspace{-2em}
    \caption{Illustration of the data and parameterized pmf with an invertible $f$. `Param. pmf' denotes the parameterized pmf captured using an MDM with parameter $\theta$. In this example, $\ell=2$.}
    \vspace{-1.2em}
    \label{fig:methodology:invertible_func}
\end{wrapfigure}

%% file: Equations/Methodology/subtoken_elbo.tex
\begin{equation}
\label{eq:methodology:subtoken_elbo}
\begin{aligned}
    \mathcal{L}_\text{vb}(\vy_0;\theta)=\int_0^1 \frac{\alpha'_t}{1-\alpha_t} \mathbb{E}_{q(\vy_t|\vy_0)}\left[\sum_{i=1}^{L}  \log p_\theta(\vy^{i}_0|\vy_{t}) \right] dt.
\end{aligned}
\end{equation}

%% file: Figures/Methodology/overview.tex
\begin{figure}
    % \vspace{-0.5em}
    \begin{minipage}{0.595\linewidth}
	\centering
    \vspace{0.4em}
    \footnotesize
    \includegraphics[width=0.985\linewidth]{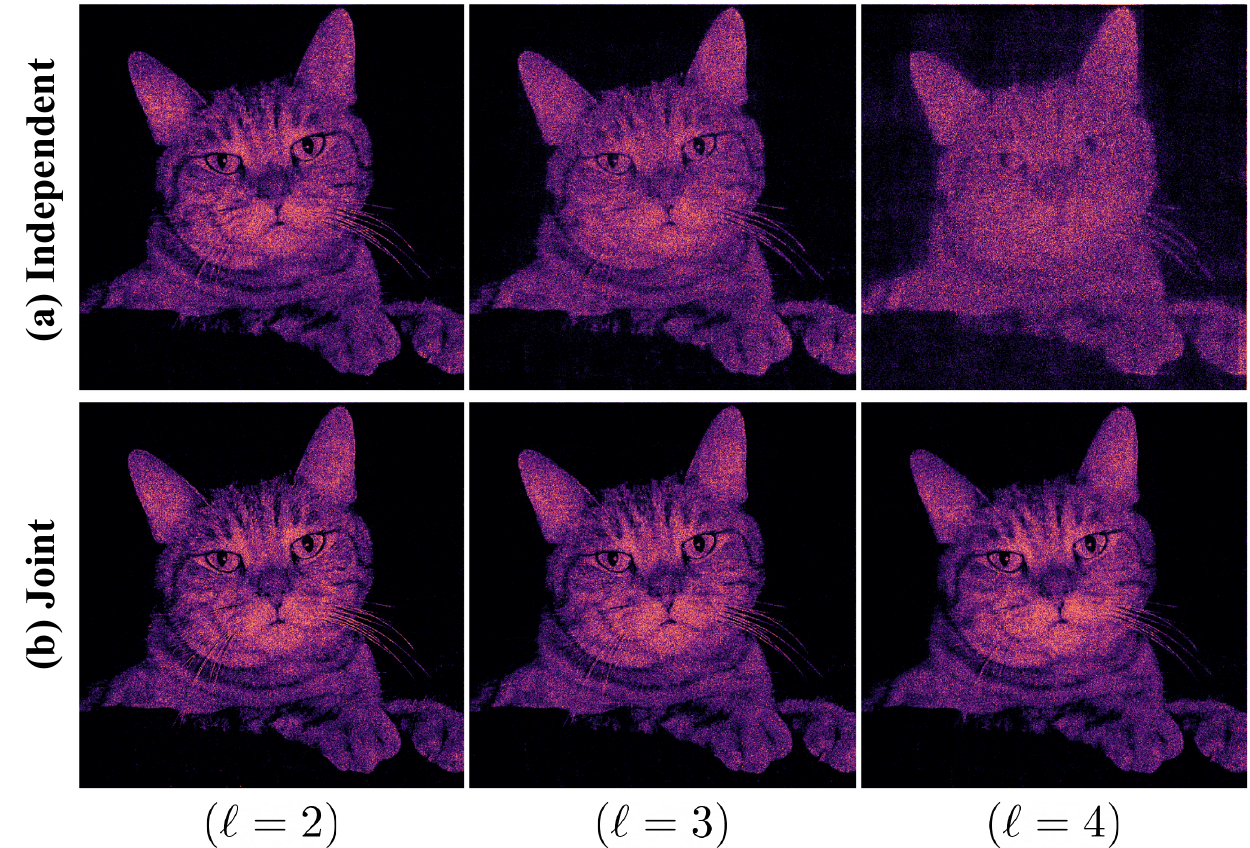}
    \vspace{-0.6em}
    \caption{Distributions modeled by MDM-Prime using (a) independent and (b) joint parameterizations. Models are trained on a two-dimensional synthetic dataset with $\vx_0\in[0,\cdots,511]^2$ representing the coordinate of the figure ($512\times 512$). Brighter regions indicate higher probabilities. Experimental details are offered in Appendix~\ref{apx:setups:toy}.}
    \label{fig:methodology:toy}
	\end{minipage}\hfill
	\begin{minipage}{0.38\linewidth}
	\centering
    \vspace{1em}
    \footnotesize
    \includegraphics[width=\linewidth]{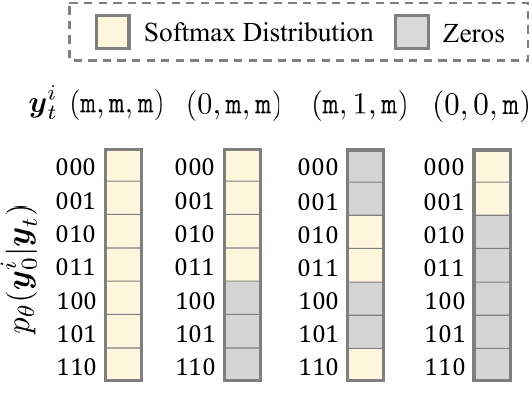}
    \vspace{-0.35em}
    \caption{An illustration of the proposed carry-over parameterization technique. In this example, $C = 7$, $\ell = 3$, and $b = 2$. The conditional distribution $p_\theta (\vy_0^i |\vy_t)$ is defined in Eq.~(\ref{eq:methodology:parameterization}). Softmax distributions are formed by normalizing the corresponding logit outputs highlighted in yellow.}
    \label{fig:methodology:marginalization}
	\end{minipage}
    % \vspace{-0.25em}
\end{figure}

%% file: Equations/Methodology/marginalization.tex
\begin{equation}
\label{eq:methodology:marginalization}
\begin{aligned}
p_\theta(y^{i,j}_0|\vy_{t})=\sum_{y_0^{i,1},\,\cdots,\,y_0^{i,j-1},\,y_0^{i,j+1},\,\cdots,\,y_0^{i,\ell}\in \Y} p_\theta(y^{i,1}_0,\cdots,y^{i,\ell}_0|\vy_{t}) \triangleq \delta_{y^{i,j}_t}(y^{i,j}_0).
\end{aligned}
\end{equation}

% \joj{
% \begin{equation}
% \begin{aligned}
% p_\theta(y^{i,j}_0|\vy_{t})=\sum_{\vy^i \in \Y^{\ell} \text{ and } y^{i,j} = y^{i,j}_0} p_\theta(\vy^i|\vy_{t}) \triangleq \delta_{y^{i,j}_t}(y^{i,j}_0).
% \end{aligned}
% \end{equation}
% }

% p_\theta(y^{i,j}_0|\vy_t)=\sum_{\vy_0^{i, \setminus j}} p_\theta(\vy^i_0|\vy_{t})

% \begin{equation}
% \begin{aligned}
% p_\theta(y^{i,j}_0|\vy_{t})=\sum_{\vy_0^{i, \setminus j}} p_\theta(y^{i,j}_0, \vy^{i, \setminus j}_0|\vy_{t})
% \end{aligned}
% \end{equation}

%% file: Equations/Methodology/parameterization.tex
\begin{equation}
\label{eq:methodology:parameterization}
\begin{aligned}
    p_\theta(\vy^{i}_0|\vy_{t})=
    \begin{cases}
        \frac{\exp (E_\theta (\vy_0^i|\vy_t))}{ \sum_{\vy^i\in \V(\vy_t^i)} \exp (E_\theta (\vy^i|\vy_t)) },& \text{if $\vy^{i}_0\in \V(\vy_t^i)$},\\
        0,& \text{if $\vy^{i}_0\notin \V(\vy_t^i)$},
    \end{cases}
\end{aligned}
\end{equation}

%% file: Figures/Methodology/architecture_comp.tex
\begin{wrapfigure}[]{r}{0.525\linewidth}
    \centering
    \vspace{-1.4em}
    \footnotesize
    \includegraphics[width=\linewidth]{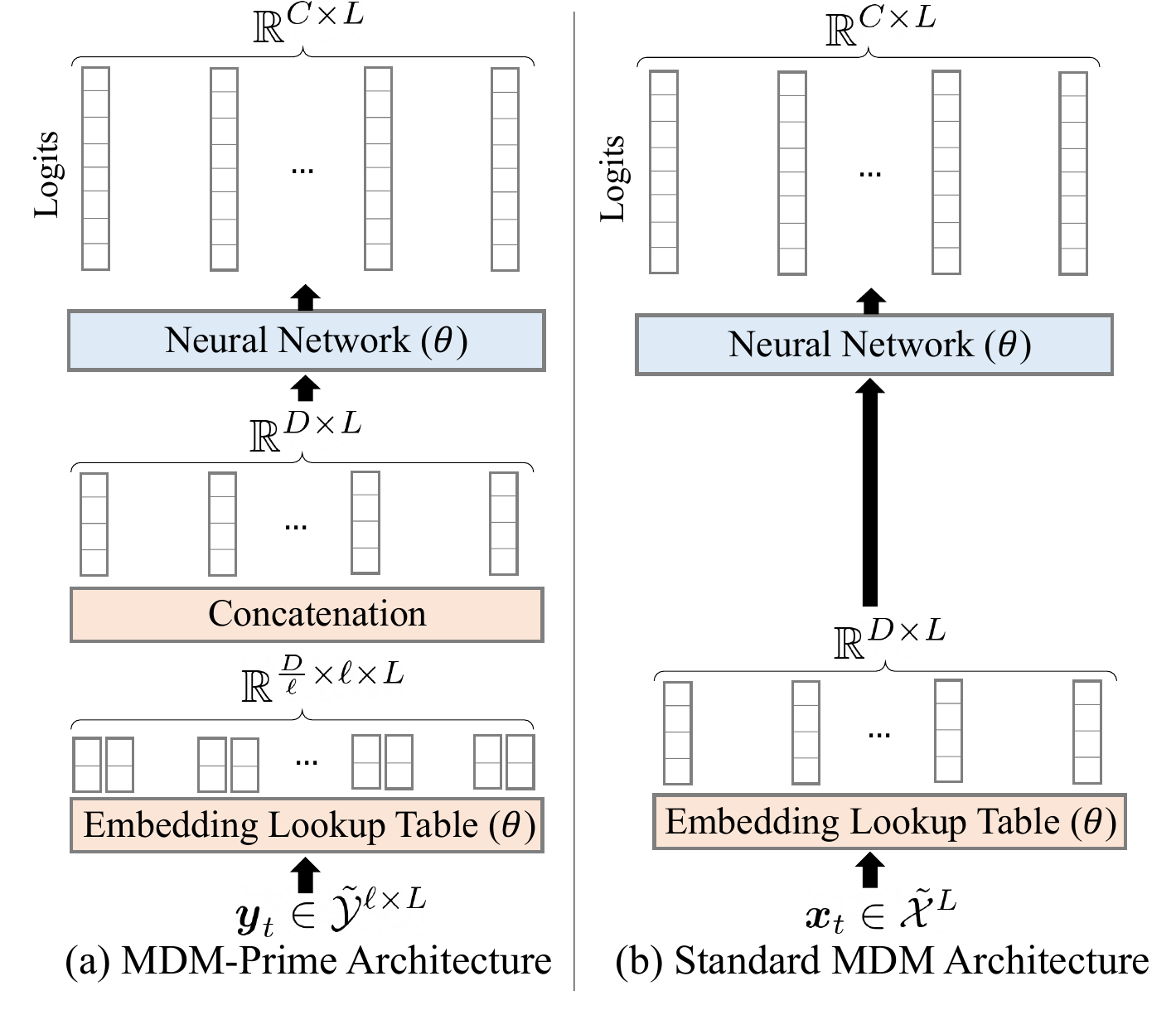}
    \vspace{-2.3em}
    \caption{Comparison between (a) MDM-Prime and (b) standard MDM architectures. The embedding lookup table in (a) has fewer learnable parameters since $|\tilde{\Y}|< |\tilde{\X}|$ and $D/\ell< D$.}
    \vspace{-1.7em}
    \label{fig:methodology:architecture_comp}
\end{wrapfigure}

%% file: Sections/4_experiments.tex
\input{Tables/Experiments/zeroshot}

\section{Experiments}
\label{sec:experiment}
This section presents empirical evaluations to examine the effectiveness of the proposed method. We first report results in the text generation domain in Section~\ref{sec:experiment:text}. Then, we provide comparisons on the image generation benchmarks in Section~\ref{sec:experiment:image}.

\subsection{Text Generation}
\label{sec:experiment:text}

\paragraph{Configuration.}~In this set of experiments, models are trained on the OpenWebText (OWT) dataset~\cite{gokaslan2019owt}. The data are tokenized using the GPT-2 tokenizer~\cite{radford2019language}, which defines $L = 1,024$ and $C = 50,257$. The Masked Diffusion Language Model (MDLM)~\cite{sahoo2024simplifieddiff} is adopted as our baseline, and the experimental setup is consistent with~\cite{sahoo2024simplifieddiff}. Prime with different $\ell$ is applied to enhance its performance, and our method is denoted as MDLM-Prime in this section. We also include comparisons with several recent works~\cite{lou2024sedd, shi2024simplifieddiff, xu2025ebmdiff, arriola2025block}. In MDLM~\cite{sahoo2024simplifieddiff}, MDLM-Prime, and all other recent methods~\cite{lou2024sedd, sahoo2024simplifieddiff, shi2024simplifieddiff, xu2025ebmdiff, arriola2025block}, the core model architecture, i.e., the `Neural Network ($\theta$)' component in Fig.~\ref{fig:methodology:architecture_comp}, is a diffusion transformer \lance{(DiT-B)}~\cite{Peebles2022DiT} with rotary positional embeddings~\cite{su2023rotary}. \lance{The model contains 130M non-embedding parameters.} No temperature is applied to the output probabilities during training and test time. Detailed hyperparameters are offered in Appendix~\ref{apx:setups:text}. Additional results\lance{, including ablation studies and efficiency evaluation,} are provided in Appendices~\ref{apx:experiments:ablation} and \ref{apx:experiments:qualitative}.

\input{Tables/Experiments/ppl}
\paragraph{Improvements to Likelihood Evaluation.}~We evaluate the models’ ability to capture the data distribution using the perplexity (PPL) metric~\cite{jurafsky2025slp3}. Table~\ref{tab:experiment:nll} reports PPL on OWT, along with the idle step ratio (ISR), which is defined as the proportion of idle steps relative to the total sampling steps and is computed using Eq.~(\ref{eq:apx:isr}). We observe that as $\ell$ increases, MDLM-Prime achieves lower PPL, with performance converging when $\ell \geq 4$. Since ISR also converges when $\ell \geq 4$, this trend suggests that ISR can serve as an indicator of improved likelihood modeling ability. We provide a further analysis of the relationship between performance and ISR, with a guideline for selecting $\ell$, in Appendix~\ref{apx:analysis:hyperparameter}. Moreover, MDLM-Prime with $\ell \geq 3$ outperforms ARM, MDM-based approaches~\cite{lou2024sedd, sahoo2024simplifieddiff, shi2024simplifieddiff, xu2025ebmdiff}, and their hybrid variants~\cite{arriola2025block, xu2025ebmdiff} by a noticeable margin in terms of PPL, indicating that incorporating intermediate state representations allows MDLM-Prime to model data likelihood more effectively. Instead of following recent approaches~\cite{xu2025ebmdiff,arriola2025block} that leverage an autoregressive formulation to enhance MDM performance, MDLM-Prime maintains an order-agnostic framework while achieving superior performance on textual data.

\paragraph{Improvements to Generalizability to Unseen Text Data.}~With the models trained on OWT, we then examine their generalizability to unseen textual datasets. To assess the models’ generalizability across diverse text domains, we report PPL on a suite of commonly used zero-shot benchmarks, including LAMBADA~\cite{paperno2016lambada}, WikiText~\cite{merity2016pointer}, Penn Treebank (PTB)~\cite{marcus1993building}, 1 Billion Word Benchmark (LM1B)~\cite{chelba2013one}, AG News~\cite{zhang2015character}, and Scientific Papers (PubMed and ArXiv subsets~\cite{cohan-etal-2018-discourse}). The results are reported in Table~\ref{tab:experiment:zeroshot}. MDLM-Prime exhibits superior results on LAMBADA, PTB, and ArXiv, and achieves comparable performance to ARM on WikiText. While it underperforms ARM on AG News, the overall results highlight its superior generalizability across multiple domains. Furthermore, our ablation study in Appendix~\ref{apx:experiments:ablation} reveals that the carry-over parameterization plays an important role in enhancing zero-shot performance, offering improvements on both LAMBADA and PubMed.

\subsection{Image Generation}
\label{sec:experiment:image}
\input{Tables/Experiments/cifar_imagenet}

\paragraph{Configuration.}~In this set of experiments, models are trained and evaluated on the CIFAR-10~\cite{krizhevsky2009cifar10} and ImageNet-32~\cite{chrabaszcz2017imagenet} datasets. For both datasets, the dimensionality is set to $L = 32 \times 32 \times 3$, with $C = 256$ corresponding to pixel intensity values. The core model architecture is adapted from the ablated diffusion model (ADM)~\cite{dhariwal2021diffusion}, which \lance{contains 114M parameters and} is the same as that used in~\cite{gat2024dfm}. Sample quality is evaluated using the widely adopted Fréchet Inception Distance (FID)~\cite{heusel2017fid} and Inception Score (IS)~\cite{barratt2018is} metrics. Experimental details are provided in Appendix~\ref{apx:setups:image}. Additional results are presented in Appendices~\ref{apx:experiments:qualitative} and~\ref{apx:experiments:traj}.

\input{Figures/Experiment/timestep_fid}
\paragraph{Improvements to Sample Quality.}~We first compare MDM-Prime with varying values of $\ell$ against the baseline configuration (i.e., $\ell = 1$). Due to the relatively small number of classes in the image experiments, we additionally explore the case of $\ell = 2/3$, where pixel values are merged into super-pixels, resulting in $b=\sqrt[2/3]{256}=4,096$ discrete classes in one of the experimental settings. As shown in Fig.~\ref{fig:experiment:timestep_fid}, the configuration with $\ell = 2$ achieves the best FID scores and exhibits a relatively low ISR compared to that of $\ell=1$ and $2/3$. While the settings with $\ell = 3$ and $\ell = 4$ perform comparably to the baseline, we observe that models with lower ISR are less sensitive to the number of function evaluations (NFE) during sampling, as reflected by the smaller FID performance gaps between NFE=$128$ and $512$. 

The benchmark results are reported in Tables~\ref{tab:experiment:benchmark_cifar} and~\ref{tab:experiment:benchmark_imagenet}, which include two baselines, MDM and MDM-Mixture, as well as several existing generative modeling approaches~\cite{karras2020stylegan, ho2020ddpm, song2021scoreflow, austin2021structurediff, campbell2022ctmc, gat2024dfm, oord2016pixelcnn, nisonoff2025guidance, song2019generative, lipman2023flowmatching, chen2023bitdiff, chao2023investigating, bartosh2024neuraldiff, tran2019msgan, zheng2023ode, albergo2023stochastic, kim2022soft}. The MDM baseline corresponds to the standard configuration with $\ell = 1$, while MDM-Mixture extends this baseline by incorporating a mixture distribution using an auxiliary variable, similar to~\cite{hayakawa2024distillation}. In this comparison, MDM-Prime adopts $\ell = 2$.

As shown in the tables, MDM and MDM-Mixture are inferior to MDM-Prime. On CIFAR-10, MDM-Prime achieves better results than the other discrete generative models while requiring fewer NFE, and attains performance comparable to StyleGAN+ADA~\cite{karras2020stylegan}. On ImageNet-32, MDM-Prime demonstrates improved performance over existing continuous diffusion and score-based models (i.e., \cite{ho2020ddpm, song2021scoreflow, chao2023investigating, bartosh2024neuraldiff, zheng2023ode, albergo2023stochastic, kim2022soft}), achieving an FID improvement of 1.36 over ScoreFlow (VP)~\cite{song2021scoreflow}.

\input{Figures/Experiment/imputation}
\paragraph{Imputation under Partial Masking.}~To further evaluate MDM-Prime’s capability for conditional image generation, we conduct experiments on image imputation. In contrast to text generation, where the positional index $j$ of a sub-token $y^{i,j}_0$ lacks explicit semantic interpretation, sub-token positions in the image domain have a direct influence on the resulting pixel values. Specifically, the first sub-token $y^{i,1}_0$ has the greatest impact on determining the final pixel intensity. In this experiment, each conditional image (denoted as `Condition' in Fig.~\ref{fig:experiment:imputation}) is first encoded into its corresponding sub-token representation. The first sub-token is retained, while the remaining sub-tokens are masked and subsequently predicted by our model. Some generated examples are shown in the `Imputation' section of Fig.~\ref{fig:experiment:imputation}. The results demonstrate that our method can generate visually coherent images conditioned on the preserved sub-tokens, highlighting its effectiveness in controlled image synthesis.

%% file: Tables/Experiments/zeroshot.tex
\begin{table}[t]
    \renewcommand{\arraystretch}{1}
    \newcommand{\boldtoprule}{\toprule[1.2pt]}
    \newcommand{\boldbottomrule}{\bottomrule[1.2pt]}
    \centering
    \small
    \vspace{-1.5em}
    \caption{Zero-shot validation perplexities evaluated on seven textual datasets. Lower values correspond to better performance. Methods marked with * incorporate an autoregressive formulation. MDLM-Prime exhibits improved results on LAMBADA, PTB, and ArXiv.}
    \vspace{-0.75em}
    \resizebox{\linewidth}{!}{%
    \begin{tabular}{cccccccc}
        \boldtoprule
                                            & LAMBADA & WikiText & PTB & LM1B & AG News & PubMed & ArXiv \\
        \hline
        ARM*~\cite{sahoo2024simplifieddiff}  & 51.28 & 25.75 & 82.05  & 51.25 & \textbf{52.09}  & 49.01 & 41.73 \\
        BD3-LM*~\cite{arriola2025block}     & $\leq$50.03 & $\leq$31.31 & $\leq$96.81  & $\leq$60.88 & $\leq$61.67 & $\leq$42.52 & $\leq$39.20 \\
        EDLM-coAR*~\cite{xu2025ebmdiff}     & $\leq$50.04 & $\leq$28.31 & $\leq$89.73  & $\leq$60.23 & $\leq$57.94 & $\leq$46.31 & $\leq$39.02 \\
        \hline
        SEDD~\cite{lou2024sedd}             & $\leq$49.86 & $\leq$34.28 & $\leq$100.09 & $\leq$68.20 & $\leq$62.09 & $\leq$41.89  & $\leq$38.48 \\
        EDLM-NCE~\cite{xu2025ebmdiff}       & $\leq$46.92 & $\leq$30.77 & $\leq$93.21  & $\leq$63.19 & $\leq$60.02 & $\leq$41.80 & $\leq$36.63 \\
        \hline
        MDLM~\cite{sahoo2024simplifieddiff} & $\leq$47.52 & $\leq$32.83 & $\leq$95.26  & $\leq$67.01 & $\leq$61.15 & $\leq$41.89 & $\leq$37.37 \\
        MDLM-Prime ($\ell=2$) & $\leq$30.91 &  $\leq$27.93 & $\leq$74.81 & $\leq$55.50 & $\leq$63.21 & $\leq$\textbf{33.32} & $\leq$25.44 \\
        MDLM-Prime ($\ell=3$) & $\leq$27.75 & $\leq$27.42 & $\leq$60.13 & $\leq$42.69 & $\leq$62.58 & $\leq$41.14  & $\leq$\textbf{24.71} \\
        MDLM-Prime ($\ell=4$) & $\leq$\textbf{24.44} & $\leq$\textbf{23.86} & $\leq$53.98 & $\leq$38.02 & $\leq$59.44  & $\leq$48.64 & $\leq$25.83 \\
        MDLM-Prime ($\ell=6$) & $\leq$25.80 & $\leq$26.87 & $\leq$62.62 & $\leq$45.36 & $\leq$60.16  & $\leq$59.09 & $\leq$25.19 \\
        MDLM-Prime ($\ell=8$) & $\leq$25.23 & $\leq$25.77 & $\leq$\textbf{53.77} & $\leq$\textbf{38.00} & $\leq$64.89 & $\leq$54.50 &  $\leq$25.79 \\
        \boldbottomrule
    \end{tabular}}
    \label{tab:experiment:zeroshot}
    \vspace{-0.75em}
\end{table}

%% file: Tables/Experiments/ppl.tex
\begin{wraptable}{r}{16.5em}
    \renewcommand{\arraystretch}{0.99}
    \newcommand{\boldtoprule}{\toprule[1.2pt]}
    \newcommand{\boldbottomrule}{\bottomrule[1.2pt]}
    \centering
    \small
    \vspace{-2em}
    \caption{PPL and ISR evaluation on OWT. Methods marked with * incorporate an autoregressive formulation. The symbol $\downarrow$ represents that lower values correspond to better performance. MDLM-Prime with $\ell\geq 3$ outperforms prior methods.}
    \vspace{0.5em}
    \resizebox{\linewidth}{!} {%
    \begin{tabular}{ccc}
        \boldtoprule
             & PPL ($\downarrow$) & ISR \\
        \hline
        ARM*~\cite{sahoo2024simplifieddiff}   & 17.54 & - \\
        BD3-LMs*~\cite{arriola2025block} & $\leq$20.73 & -\\
        EDLM-coAR*~\cite{xu2025ebmdiff} & $\leq$17.58 & -\\
        \hline
        SEDD~\cite{lou2024sedd}  & $\leq$24.10 & -\\
        GenMD4~\cite{shi2024simplifieddiff}  & $\leq$21.80 & -\\
        EDLM-NCE~\cite{xu2025ebmdiff}  & $\leq$21.52 & -\\
        \hline
        MDLM~\cite{sahoo2024simplifieddiff} & $\leq$22.98 & 36.77\%\\
        MDLM-Prime ($\ell=2$) & $\leq$17.90 & 13.52\%\\
        MDLM-Prime ($\ell=3$) & $\leq$16.36 & 4.97\%\\
        MDLM-Prime ($\ell=4$) & $\leq$15.62 & 1.83\%\\
        MDLM-Prime ($\ell=6$) & $\leq$\textbf{15.36} & 0.25\%\\
        MDLM-Prime ($\ell=8$) & $\leq$15.48 & 0.03\%\\
        \boldbottomrule
    \end{tabular}}
    \label{tab:experiment:nll}
    \vspace{-1em}
\end{wraptable}

% \begin{wraptable}{r}{16.5em}
%     \renewcommand{\arraystretch}{0.99}
%     \newcommand{\boldtoprule}{\toprule[1.2pt]}
%     \newcommand{\boldbottomrule}{\bottomrule[1.2pt]}
%     \centering
%     \small
%     \vspace{-2em}
%     \caption{PPL and ISR evaluation on OWT. Methods marked with * incorporate an autoregressive formulation. The symbol $\downarrow$ represents that lower values correspond to better performance. MDLM-Prime with $\ell\geq 3$ outperforms prior methods.}
%     \vspace{0.5em}
%     \resizebox{\linewidth}{!} {%
%     \begin{tabular}{ccc}
%         \boldtoprule
%              & PPL ($\downarrow$) & ISR \\
%         \hline
%         ARM* [3]   & 17.54 & - \\
%         BD3-LMs* [8] & $\leq$20.73 & -\\
%         EDLM-coAR* [6] & $\leq$17.58 & -\\
%         \hline
%         SEDD [1]  & $\leq$24.10 & -\\
%         GenMD4 [4]  & $\leq$21.80 & -\\
%         EDLM-NCE [6]  & $\leq$21.52 & -\\
%         \hline
%         MDLM [3] & $\leq$22.98 & 36.77\%\\
%         MDLM-Prime ($\ell=2$) & $\leq$17.90 & 13.52\%\\
%         MDLM-Prime ($\ell=3$) & $\leq$16.36 & 4.97\%\\
%         MDLM-Prime ($\ell=4$) & $\leq$15.62 & 1.83\%\\
%         MDLM-Prime ($\ell=6$) & $\leq$\textbf{15.36} & 0.25\%\\
%         MDLM-Prime ($\ell=8$) & $\leq$15.48 & 0.03\%\\
%         \boldbottomrule
%     \end{tabular}}
%     \label{tab:experiment:nll}
%     \vspace{-1em}
% \end{wraptable}

%% file: Tables/Experiments/cifar_imagenet.tex
\begin{table}
    \newcommand{\boldtoprule}{\toprule[1.2pt]}
    \newcommand{\boldbottomrule}{\bottomrule[1.2pt]}
    \renewcommand{\arraystretch}{1}
    % \vspace{-0.5em}
	\begin{minipage}{0.49\linewidth}
        \centering
        \footnotesize
        \caption{FID and IS evaluation on CIFAR-10. The arrow symbols $\uparrow / \downarrow$ represent that higher / lower results correspond to better performance.}
        \resizebox{\linewidth}{!}{%
        \begin{tabular}{lrr}
            \boldtoprule
            \multicolumn{3}{c}{CIFAR-10} \\ 
            \hline
            \multicolumn{3}{c}{Discrete} \\ 
            \hline
            Method                     & FID ($\downarrow$) & IS ($\uparrow$)\\ 
            \hline
            MDM         \scriptsize{(NFE=512)} &  4.66 & 9.09  \\
            MDM-Mixture \scriptsize{(NFE=512)} &  4.80 & 9.22  \\
            MDM-Prime \scriptsize{(NFE=512)} &  \textbf{3.26} & \textbf{9.67}  \\
            \hline 
            PixelCNN~\cite{oord2016pixelcnn}  & 65.93  &  4.60  \\
            D3PM Absorb~\cite{austin2021structurediff} \scriptsize{(NFE=1,000)}  & 30.97  &  6.78  \\
            D3PM Gauss.~\cite{austin2021structurediff} \scriptsize{(NFE=1,000)}  & 7.34  &  8.56  \\
            CTDD-DG~\cite{nisonoff2025guidance} \scriptsize{(NFE=1,000)}  & 7.86  &  8.91  \\
            Tau-LDR~\cite{campbell2022ctmc}  \scriptsize{(NFE=1,000)}     & 3.74  &  9.49  \\
            Discrete FM~\cite{gat2024dfm} \scriptsize{(NFE=1,024)}          &  3.63  &  - \\
            \hline
            \multicolumn{3}{c}{Continuous} \\ 
            \hline
            NCSN~\cite{song2019generative}	& 25.32  &  8.87  \\
            Continuous FM~\cite{lipman2023flowmatching}	 & 6.35	 &   -    \\
            Bit Diffusion~\cite{chen2023bitdiff}  & 3.48  &  -  \\
            StyleGAN+ADA~\cite{karras2020stylegan} & 3.26   &  \textbf{9.74}  \\
            DDPM~\cite{ho2020ddpm}  & \textbf{3.17}   &  9.46  \\
            \boldbottomrule
        \end{tabular}}
        \label{tab:experiment:benchmark_cifar}
	\end{minipage}\hfill
	\begin{minipage}{0.49\linewidth}
        \centering
        \footnotesize
        \caption{FID and IS evaluation on ImageNet-32. The arrow symbols $\uparrow / \downarrow$ represent that higher / lower results correspond to better performance.}
        \vspace{-0.1em}
        \resizebox{\linewidth}{!}{%
        \begin{tabular}{lrr}
            \boldtoprule
            \multicolumn{3}{c}{ImageNet-32} \\ 
            \hline
            \multicolumn{3}{c}{Discrete} \\ 
            \hline
            Method                     & FID ($\downarrow$) & IS ($\uparrow$)\\ 
            \hline
            MDM         \scriptsize{(NFE=1,024)} &  7.91  &  11.60 \\
            MDM-Mixture \scriptsize{(NFE=1,024)}~~~~~~~~~ &  8.08  &  11.56 \\
            MDM-Prime \scriptsize{(NFE=1,024)} &  \textbf{6.98} & \textbf{11.65}  \\
            \hline
            \multicolumn{3}{c}{Continuous} \\ 
            \hline
            QC-NCSN++~\cite{chao2023investigating}	&  19.62  &  9.94 \\
            NDM~\cite{bartosh2024neuraldiff}   &  17.02  &   -   \\
            DDPM~\cite{ho2020ddpm}     &  16.18	 &   -   \\
            MSGAN~\cite{tran2019msgan} &  12.30  &   -   \\
            i-DODE (SP)~\cite{zheng2023ode}  &  10.31  &   -   \\
            i-DODE (VP)~\cite{zheng2023ode}  &  9.09	 &   -   \\
            Stochastic Interp.~\cite{albergo2023stochastic}  &  8.49   &   -   \\
            Soft Trunc. DDPM~\cite{kim2022soft}  &  8.42   &  \textbf{11.82} \\
            ScoreFlow (subVP)~\cite{song2021scoreflow} &  8.87   &   -   \\
            ScoreFlow (VP)~\cite{song2021scoreflow}  &  8.34   &   -   \\
            Continuous FM~\cite{lipman2023flowmatching}  &  \textbf{5.02}   &   -    \\
            \boldbottomrule
        \end{tabular}}
        \label{tab:experiment:benchmark_imagenet}
	\end{minipage}
    \vspace{-0.5em}
\end{table}

%% file: Figures/Experiment/timestep_fid.tex
\begin{wrapfigure}[]{r}{20.75em}
    \centering
    \footnotesize
    \vspace{-1.5em}
    \includegraphics[width=\linewidth]{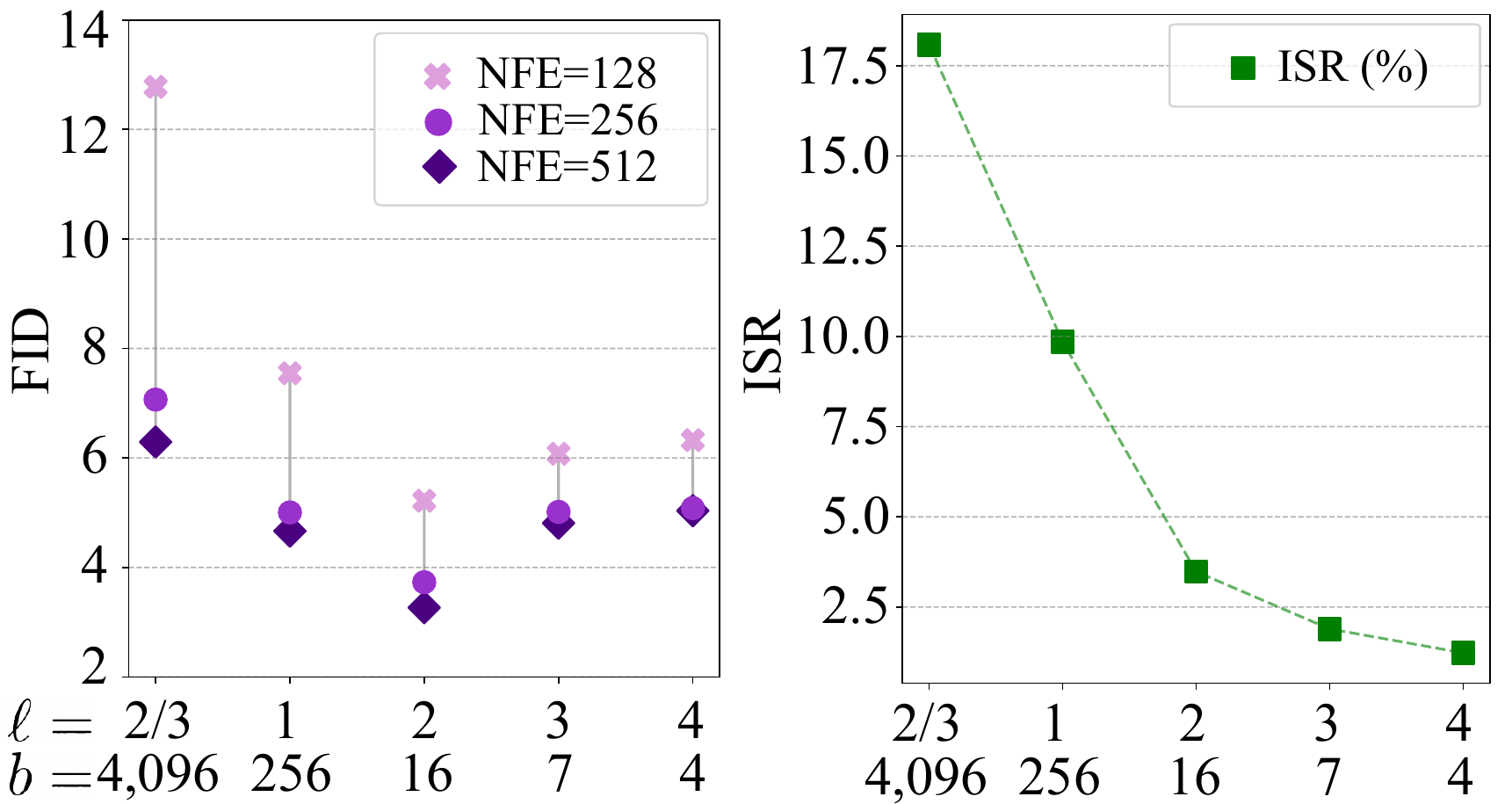}
    \vspace{-1.75em}
    \caption{FID and ISR evaluated under different $\ell$ on CIFAR-10. NFE denotes the number of function evaluations during sampling.}
    % ISR is computed based on a total of 512 steps. 
    \label{fig:experiment:timestep_fid}
    \vspace{-1em}
\end{wrapfigure}

%% file: Figures/Experiment/imputation.tex
\begin{wrapfigure}[]{r}{18.7em}
    \centering
    \footnotesize
    \vspace{-1.75em}
    \includegraphics[width=\linewidth]{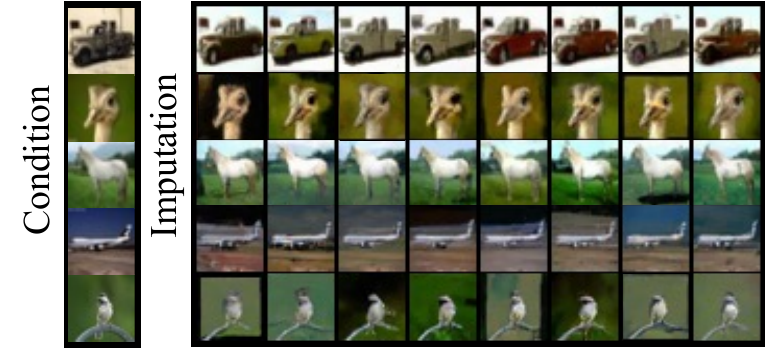}
    \vspace{-2em}
    \caption{Imputation results with conditional images obtained from CIFAR-10. MDM-Prime generates visually coherent image variants. More examples are provided in Appendix~\ref{apx:experiments:qualitative}.}
    \label{fig:experiment:imputation}
    \vspace{-1.5em}
\end{wrapfigure}

%% file: Sections/5_related_works.tex
\section{Related Works}
\label{sec:related_works}

The general framework of discrete diffusion models was first introduced by~\cite{dickstein2015diffusion}, where the authors explored modeling binary data using a diffusion process. This framework was extended to real-world applications, such as text and image generation, by the authors in~\cite{austin2021structurediff}, who proposed various perturbation strategies to implement diffusion processes with discrete noise. The authors in~\cite{campbell2022ctmc} proposed generalizing this framework as Continuous-Time Markov Chains (CTMC). Inspired by the success of scaling MDM for text generation~\cite{lou2024sedd}, some works~\cite{sahoo2024simplifieddiff, shi2024simplifieddiff, ou2025radd} explored simplifications of the training objective of discrete diffusion models with latent variables represented using masked tokens. Other studies~\cite{gat2024dfm, shaul2025kinetic} investigated learning approaches based on a broader class of latent representations formulated through flow matching. 

Building on the theoretical foundation of MDM, several enhancement techniques~\cite{xu2025ebmdiff, arriola2025block, gat2024dfm, hayakawa2024distillation, zhao2024informed, wang2025remasking, deschenaux2025bydiff} have been proposed. In~\cite{arriola2025block}, the authors proposed an interpolation between ARM and MDM to capture the left-to-right structure in textual data. In~\cite{xu2025ebmdiff}, an ARM was employed as an energy-based function to guide the sampling process of an MDM, resulting in improved performance. Other works~\cite{gat2024dfm, zhao2024informed, wang2025remasking} modified the sampling process of MDM to selectively remask certain unmasked predictions to enhance sample quality. In addition, the authors in~\cite{hayakawa2024distillation, deschenaux2025bydiff} explored distillation techniques designed to reduce the number of sampling steps while maintaining sample quality.

Another line of research has explored the use of diffusion models with continuous noise distributions (i.e., Gaussian) for modeling discrete data. Representative methods include~\cite{chen2023bitdiff, dieleman2022cdcd, li2022diffllm, gulrajani2023plaid, zhang2024flowllm, hu2024flowllm}. 
% Specifically,~\cite{dieleman2022cdcd, li2022diffllm, gulrajani2023plaid} adopted latent diffusion models to generate word embeddings, while \cite{zhang2024flowllm, hu2024flowllm} extended this framework using flow-matching techniques. 
Among these works, Bit Diffusion~\cite{chen2023bitdiff} shares similarities with our approach. In their method, discrete data are first encoded into bit representations, and a continuous diffusion model (with Gaussian kernels) is trained to generate these encodings. The generated outputs are then quantized back into discrete tokens. However, due to its reliance on quantization, the model's likelihood becomes intractable, which leads to its inability to directly capture the distribution of discrete data. 

%% file: Sections/6_conclusion.tex
\section{Conclusion}
\label{sec:conclusion}
Scientific progress has continually reshaped our understanding of what constitutes the most basic units of matter. Physicists initially believed that atoms were elementary units of matter. This view changed with the discoveries of the electron, the atomic nucleus, and eventually the development of the \textit{standard model}~\cite{Cottingham_Greenwood_2007}, which describes fundamental particles, their interactions, and how they combine to form atoms. In the context of generative models, we proposed Prime, a method to decompose the elementary unit of discrete data--tokens--into fine-grained subcomponents. MDM-Prime establishes a principled framework for perturbing and reconstructing discrete data using sub-token representations. Experimental results on both text and image generation tasks demonstrated that sub-token representations provide a more expressive modeling paradigm. We believe that this framework holds potential for addressing real-world problems that require fine-grained and precise modeling of discrete data.

\section*{Acknowledgements}
RGK is supported by a Canada CIFAR AI Chair and a Canada Research Chair Tier II in Computational Medicine (CRC-2022-00049). This research was supported by an NFRF Special Call NFRFR2022-00526. Resources used in preparing this research were provided, in part, by the Province of Ontario, the Government of Canada through CIFAR, and companies sponsoring the Vector Institute. The authors gratefully acknowledge the support from the National Science and Technology Council (NSTC) in Taiwan under grant numbers NSTC 114-2221-E-002-069-MY3, NSTC 113-2221-E-002-212-MY3, and NSTC 114-2218-E-A49-026. We also express our sincere appreciation to NVIDIA Corporation and the NVIDIA AI Technology Center (NVAITC) for the donation of GPUs and access to the Taipei-1 supercomputer. Furthermore, the authors extend their gratitude to the National Center for High-Performance Computing for providing the necessary computational and storage resources. Finally, we thank David Pellow, Vahid Belazadeh Meresht, and the anonymous reviewers for valuable feedback.

%% file: Sections/a0_appendix.tex
\newpage
\appendix

% Reset the value of the item counter
\setcounter{section}{0}
\setcounter{equation}{0}
\setcounter{figure}{0}
\setcounter{table}{0}

% Add a prefix for each tag
\renewcommand{\thefigure}{A\arabic{figure}}
\renewcommand{\thetable}{A\arabic{table}}
\renewcommand{\theequation}{A\arabic{equation}}

\section{Appendix}
\label{apx}

In this appendix, we provide additional analyses and experiments. Section~\ref{apx:mdm} presents some theoretical properties of masked diffusion models (MDM). Section~\ref{apx:analysis} offers an analysis of the proposed \underline{P}a\underline{r}t\underline{i}al \underline{m}asking schem\underline{e} (Prime). Section~\ref{apx:architecture} outlines a number of architectural designs of MDM augmented with Prime (i.e., MDM-Prime). Section~\ref{apx:setups} details the experimental configurations. Section~\ref{apx:experiments} reports additional experimental results. Section~\ref{apx:limitation} summarizes the limitations of this work. Finally, Section~\ref{apx:impacts} discusses the potential impacts of this work. The following table of contents summarizes the structure of the main manuscript and this appendix.

\setcounter{tocdepth}{2}
\renewcommand{\contentsname}{Table of Contents}
\tableofcontents

\clearpage
\input{Sections/a1_diffusion}

\input{Sections/a2_analysis}

\input{Sections/a4_architecture}
\input{Sections/a5_setups}
\input{Sections/a6_experiments}

%% file: Sections/a1_diffusion.tex
\subsection{Analyses of Masked Diffusion Processes}
\label{apx:mdm}
In this section, we examine two properties of MDM. In Section~\ref{apx:mdm:idle}, we derive an analytical expression for the expected number of idle steps in the reverse diffusion process. In Section~\ref{apx:mdm:mutual_information}, we show that the mutual information between the latent variables and data exhibits a linearly decaying trend with respect to the scheduling function over time.

\subsubsection{Expected Number of Idle Steps}
\label{apx:mdm:idle}
In Section~\ref{sec:introduction}, we show that MDM may have significant number of idle steps during the sampling process. In this section, we derive a closed-form formula for calculating the expected number of idle steps and provide the reason why Prime must reduce this number.

\begin{proposition}
\label{prop:idle_step}
Let $L$ be the token sequence length, $T$ be the total number of discretized timesteps for the sampling process, and $\alpha_t\in [0,1]$ be a strictly decreasing scheduling function in $t\in [0,1]$. Suppose the sampling timesteps are indexed by $k\in \{0,\cdots,T-1\}$, the expected number of idle steps $\eta$, i.e., the entire token sequence remains unchanged between consecutive steps, is given by:
\begin{equation}
\label{eq:apx:idle_step}
    \eta=\sum_{k=0}^{T-1} \left[1 - \left(\alpha_{1-\frac{k+1}{T}} - \alpha_{1-\frac{k}{T}}\right)\right]^L.
\end{equation}

\end{proposition}
\begin{proof}
At reverse step $k \to k+1$, a token remains unchanged if it is already unmasked at $t=1-\frac{k}{T}$ or it is masked at $t=1-\frac{k}{T}$ while remaining masked $t=1-\frac{k+1}{T}$. According to the forward diffusion process, the probability that a single token remains unchanged is given by:
\begin{equation}
\label{eq:apx:idle_step_k}
\underbrace{\alpha_{1-\frac{k}{T}}}_{(i)} + \underbrace{(1-\alpha_{1-\frac{k}{T}})}_{(ii)} \cdot \underbrace{\frac{1-\alpha_{1-\frac{k+1}{T}}}{1-\alpha_{1-\frac{k}{T}}}}_{(iii)}  = 1 - \left(\alpha_{1-\frac{k+1}{T}} -\alpha_{1-\frac{k}{T}}\right),
\end{equation}
where $(i)$ is the probability of observing an unmasked token at time $t=1-\frac{k}{T}$ (i.e., Eq.~(\ref{eq:background:diffusion_kernel_t0})), $(ii)$ is the probability of observing a masked token at time $t=1-\frac{k}{T}$ (i.e., Eq.~(\ref{eq:background:diffusion_kernel_t0})), and $(iii)$ is the probability that the masked token remains masked at time $t=1-\frac{k+1}{T}$ (i.e., Eq.~(\ref{eq:background:diffusion_kernel_st0})). For all $L$ tokens to remain unchanged, the joint probability is $\left[1 - \left(\alpha_{1-\frac{k+1}{T}} -\alpha_{1-\frac{k}{T}}\right)\right]^L$. Due to the linearity of expectation, the expected number of idle steps can be calculated by accumulating Eq.~(\ref{eq:apx:idle_step_k}) for all $k\in\{0,\cdots,T-1\}$, and is written as:
\begin{equation}
\sum_{k=0}^{T-1} \left[1 - \left(\alpha_{1-\frac{k+1}{T}} -\alpha_{1-\frac{k}{T}}\right)\right]^L.
\end{equation}
\end{proof}

Eq.~(\ref{eq:apx:idle_step}) enables the computation of idle steps given the sequence length $L$, the number of discretized steps $T$, and the scheduling function $\alpha_t$. Consider the commonly used linear schedule $\alpha_t = 1 - t$~\cite{sahoo2024simplifieddiff, shi2024simplifieddiff}. Substituting the difference $\alpha_{1 - \frac{k+1}{T}} - \alpha_{1 - \frac{k}{T}} = \frac{1}{T}$ into the formula yields:
\begin{equation}
\sum_{k=0}^{T-1} \left(1 - \frac{1}{T}\right)^L=T\left(1 - \frac{1}{T}\right)^L\overset{(i)}{\approx} T e^{-\frac{L}{T}},
\end{equation}
where $(i)$ holds when $T$ is large. For example, when $T = 1{,}024$ and $L= 1{,}024$, we obtain $T e^{-L/T} \approx 1024 \cdot e^{-1} \approx 376$, indicating that approximately 37\% of the sampling steps are idle in the discrete diffusion process.

Moreover, since $1 - \left(\alpha_{1 - \frac{k+1}{T}} - \alpha_{1 - \frac{k}{T}}\right)\in [0,1]$, the quantity in Eq.~(\ref{eq:apx:idle_step}) decreases as the token sequence length $L$ increases. In MDM-Prime, the sequence length is extended by a factor of $\ell > 1$ due to the introduction of sub-token sequences, and the number of idle steps $\eta_\text{\,Prime}$ is given by:
\begin{equation}
\label{eq:apx:sub_idle_step}
\eta_\text{\,Prime}=\sum_{k=0}^{T-1} \left[1 - \left(\alpha_{1-\frac{k+1}{T}} - \alpha_{1-\frac{k}{T}}\right)\right]^{L\times \ell}.
\end{equation}
Given identical $T$ and $\alpha_t$, $\eta_\text{\,Prime}$ in Eq.~(\ref{eq:apx:sub_idle_step}) is always smaller than $\eta$ in Eq.~(\ref{eq:apx:idle_step}). Therefore, employing Prime must result in fewer idle steps. To support this observation and verify the correctness of Eqs.~(\ref{eq:apx:idle_step}) and~(\ref{eq:apx:sub_idle_step}), Fig.~\ref{fig:empty_step_analytic} compares our simulation results with the analytical values computed using $\eta$ and $\eta_\text{\,Prime}$. The results demonstrate agreement between the theoretical and empirical estimates.
\input{Figures/Appendix/empty_step_analytic}

Based on the above definition of expected idle steps, we define the idle step ratio (ISR) as the proportion of expected idle steps relative to the total number of sampling steps $T$:
\begin{equation}
\label{eq:apx:isr}
\frac{1}{T}\sum_{k=0}^{T-1} \left[1 - \left(\alpha_{1-\frac{k+1}{T}} - \alpha_{1-\frac{k}{T}}\right)\right]^{L\times \ell}.
\end{equation}
ISR quantifies the model utilization in diffusion processes. To evaluate it, we can fix the scheduling function $\alpha_t$ and choose a large $T$ (i.e., to approximate the continuous-time limit $T \to \infty$). In our experiments presented in Section~\ref{sec:experiment}, we set $T$ as $1,024$.

\subsubsection{Mutual Information Scheduling}
\label{apx:mdm:mutual_information}

In this section, we show that the MDM framework implements mutual information scheduling (i.e., $I(x^i_t; x^i_0) = \alpha_t H(x^i_0)$). This relationship was initially established in prior works~\cite{austin2021structurediff,dickstein2015diffusion}, where a linear scheduling function $\alpha_t = 1 - t$ was considered as a special case. In our work, we extend this result to arbitrary scheduling functions \(\alpha_t \in [0, 1]\), as formally stated in \textbf{Proposition}~\ref{prop:schedule}.

\begin{proposition}
\label{prop:schedule}
Let $\alpha_t \in [0,1]$ be the scheduling function and $q(x^i_t|x^i_0)$ be the distribution defined in Eq.~(\ref{eq:background:diffusion_kernel_t0}). The mutual information between $x^i_0$ and $x^i_t$ satisfies:
\begin{equation}
\label{eq:apx:mutual_info_scheduling}
    I(x^i_t;x^i_0)=\alpha_t H(x^i_0).
\end{equation}
\end{proposition}
\begin{proof}
The mutual information can be decomposed as follows:
\begin{equation}
\label{eq:apx:mutual_info}
I(x^i_t; x^i_0) = H(x_t^i) - H(x_t^i | x_0^i)
\end{equation}
The following proof expands both entropy terms to derive Eq.~(\ref{eq:apx:mutual_info_scheduling}).

\paragraph{(i) $H(x_t^i | x_0^i)$:}~Given $x_0^i$, the distribution over $x_t^i$ is:
\begin{equation}
q(x^i_t|x^i_0) = \alpha_t \, \delta_{x_0^i}(x_t^i) + (1 - \alpha_t) \, \delta_{\texttt{m}}(x_t^i).
\end{equation}
Its entropy can be expressed as the entropy of a Bernoulli:
\begin{equation}
\label{eq:apx:conditional_entropy}
H(x_t^i | x_0^i) = H(\text{Bernoulli}(\alpha_t)) = -\alpha_t \log \alpha_t - (1 - \alpha_t) \log(1 - \alpha_t).
\end{equation}

\paragraph{(ii) $H(x_t^i)$:}~The element-wise distribution $p_t(x_t^i)$ can be written as follows:
\begin{equation}
\label{eq:apx:marginal}
\begin{aligned}
p_t(x_t^i) &= \sum_{x_0^i \in \X} p_0(x_0^i) \cdot q(x_t^i | x_0^i) \\
&= \sum_{x_0^i \in \X} p_0(x_0^i) \left[ \alpha_t \delta_{x_0^i}(x_t^i) + (1 - \alpha_t) \delta_{\texttt{m}}(x_t^i) \right]=
\begin{cases}
\alpha_t \, p_0(x_t^i), & \text{if } x_t^i \in \mathcal{X}, \\
1 - \alpha_t, & \text{if } x_t^i = \texttt{m}
\end{cases}.
\end{aligned}
\end{equation}
The element-wise entropy can be expanded using Eq.~(\ref{eq:apx:marginal}) as follows:
\begin{equation}
\label{eq:apx:marginal_entropy}
\begin{aligned}
H(x_t^i)
&= - \sum_{x \in \tilde{\X}} p_t(x) \log p_t(x) \\
&= - \left(\sum_{x \in \X} \alpha_t p_0(x) \log(\alpha_t p_0(x)) \right)- (1 - \alpha_t) \log(1 - \alpha_t) \\
&= -\left(\alpha_t \sum_{x \in \X} p_0(x) \log p_0(x) \right) - \alpha_t \log \alpha_t - (1 - \alpha_t) \log(1 - \alpha_t) \\
&= \alpha_t H(x_0^i) + H(\text{Bernoulli}(\alpha_t)).
\end{aligned}
\end{equation}

According to Eqs.~(\ref{eq:apx:conditional_entropy}) and (\ref{eq:apx:marginal_entropy}), the mutual information is expressed as follows:
\begin{equation}
\begin{aligned}
I(x^i_t; x^i_0) &= H(x_t^i) - H(x_t^i | x_0^i)\\
&= \alpha_t H(x_0^i) + H(\text{Bernoulli}(\alpha_t)) - H(\text{Bernoulli}(\alpha_t)) \\
&= \alpha_t H(x_0^i).
\end{aligned}
\end{equation}
\end{proof}

%% file: Figures/Appendix/empty_step_analytic.tex
\begin{figure}[t]
    \centering
    \footnotesize
    \includegraphics[width=\linewidth]{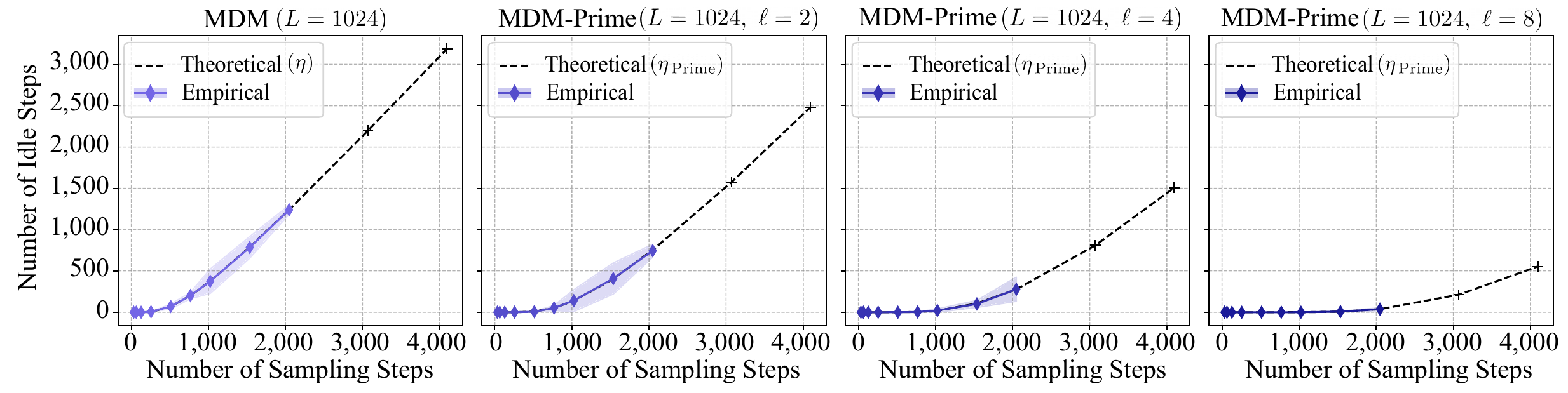}
    \vspace{-1.5em}
    \caption{Comparison between idle steps obtained from simulation results and analytical computation. Solid curves and their corresponding shaded areas represent the mean and variance from ten independent simulation runs. Dashed lines indicate the theoretical values.}
    \label{fig:empty_step_analytic}
\end{figure}

%% file: Sections/a2_analysis.tex
\subsection{Analyses of the Partial Mask Scheme}
\label{apx:analysis}

In this section, we provide thorough analyses of MDM-Prime. Section~\ref{apx:analysis:nll} derives a formula for computing the negative log-likelihood (NLL). Section~\ref{apx:analysis:intermediate} shows that Prime yields a positive number of intermediate states. Section~\ref{apx:analysis:carry_over} examines the correctness of the \textit{carry-over} parameterization for Prime. Finally, Section~\ref{apx:analysis:hyperparameter} offers a guideline for selecting the target length $\ell$.

\subsubsection{Negative Log Likelihood Calculation}
\label{apx:analysis:nll}
In Section~\ref{sec:methodology}, we introduce an invertible function $f:\X^L \to \Y^{\ell \times L}$ to transform between two vectors, $\vx_0$ and $\vy_0$. In this section, we detail the computation of the expected negative log-likelihood (NLL) of MDM after applying this transformation. We begin by revisiting the derivation of Eq.~(\ref{eq:background:diffusion_elbo}), and then extend it to our proposed objective in Eq.~(\ref{eq:methodology:subtoken_elbo}).

The expected NLL of $\vx_0$ is expressed as $\E_{p_{\text{data}}(\vx_0)}[-\log p_\theta(\vx_0)]$, which can be approximated using a variational upper bound~\cite{sahoo2024simplifieddiff, shi2024simplifieddiff, xu2025ebmdiff} as expressed as follows:
\begin{equation}
\label{eq:apx:nll_derivation}
\begin{aligned}
    \E_{p_{\text{data}}(\vx_0)}[-\log p_\theta(\vx_0)] &\leq \E_{p_{\text{data}}(\vx_0)}\left[\int_0^1 \frac{\alpha'_t}{1-\alpha_t} \mathbb{E}_{q(\vx_t|\vx_0)}\left[\log p_\theta (\vx_0|\vx_t) \right] dt \right]\\
    &\overset{(i)}{=}\E_{p_{\text{data}}(\vx_0)}\left[\int_0^1 \frac{\alpha'_t}{1-\alpha_t} \mathbb{E}_{q(\vx_t|\vx_0)}\left[\sum_{i=1}^L \log p_\theta (x_0^i|\vx_t) \right] dt\right],
\end{aligned}
\end{equation}
where $(i)$ is derived from $p_\theta(\vx_0|\vx_t)=\prod_{i=1}^L p_\theta (x_0^i|\vx_t)$. Due to the introduction of $f$, the expected NLL expressed by MDM-Prime is written as $\E_{p_{\text{data}}(\vx_0)}[-\log p_\theta\circ f(\vx_0)]$, where $p_\theta\circ f(\vx_0)$ is the parameterized pmf as discussed in Section~\ref{sec:methodology:diff_pm}. This expectation can be estimated via Eq.~(\ref{eq:methodology:subtoken_elbo}), as derived below:
\begin{equation}
\label{eq:apx:nll}
\begin{aligned}
    \E_{p_{\text{data}}(\vx_0)}[-\log p_\theta\circ f(\vx_0)]&=\E_{p_{\text{data}}\circ f^{-1}(\vy_0)}[-\log p_\theta(\vy_0)] \\
    &\leq \E_{p_{\text{data}}\circ f^{-1}(\vy_0)} \left[\int_0^1 \frac{\alpha'_t}{1-\alpha_t} \mathbb{E}_{q(\vy_t|\vy_0)}\left[\log p_\theta(\vy_0|\vy_{t}) \right] dt \right] \\
    &\overset{(i)}{=} \E_{p_{\text{data}}\circ f^{-1}(\vy_0)} \left[\int_0^1 \frac{\alpha'_t}{1-\alpha_t} \mathbb{E}_{q(\vy_t|\vy_0)}\left[\sum_{i=1}^{L}  \log p_\theta(\vy^{i}_0|\vy_{t}) \right] dt \right],
\end{aligned}
\end{equation}
where $(i)$ is due to $p_\theta(\vy_0|\vy_t)=\prod_{i=1}^L p_\theta (\vy_0^i|\vy_t)$ following the derivation in Eq.~(\ref{eq:apx:nll_derivation}). To sample data points from the distribution $p_{\text{data}}\circ f^{-1}$, one can first sample $\vx_0 \sim p_{\text{data}}$ and then transform it to $\vy_0 = f(\vx_0)$ according to the change-of-variable principle for probability distributions. This formulation enables us to estimate the expected NLL of MDM-Prime.

\subsubsection{Number of Intermediate States}
\label{apx:analysis:intermediate}
In Section~\ref{sec:methodology:diff_pm}, we claim that the number of intermediate states can be quantified as $|\tilde{\Y}^\ell| - |\tilde{\X}| = (b+1)^\ell - (C+1)$ and that this number is always positive. To verify this, we provide \textbf{Proposition}~\ref{prop:num_state}.

\begin{proposition}
\label{prop:num_state}
Let $\tilde{\Y}=\{0,\cdots, b-1\} \cup \{\mathtt{m}\}$ and $\tilde{\X}=\{0,\cdots, C-1\} \cup \{\mathtt{m}\}$. Let $\ell>1$ be a positive integer and let $b=\lceil \sqrt[\ell]{C} \rceil$. The number of intermediate states is a positive integer:
\begin{equation}
\begin{aligned}
|\tilde{\Y}^\ell| - |\tilde{\X}| > 0.
\end{aligned}
\end{equation}
\end{proposition}

\begin{proof}
The number of original tokens with the mask token is:
\begin{equation}
|\tilde{\X}| = C + 1.
\end{equation}
The total number of possible sub-token sequences (including the mask) is:
\begin{equation}
|\tilde{\Y}^\ell| = (b+1)^\ell \overset{(i)}{=} \sum_{k=0}^{\ell} \binom{\ell}{k} b^{\ell-k}= b^\ell + \underbrace{\binom{\ell}{1}b^{\ell-1} + \cdots}_{>0} + 1 > b^\ell + 1,
\end{equation}
where $(i)$ is derived by the binomial theorem. Since $b = \lceil \sqrt[\ell]{C} \rceil$ by construction, we have $b^\ell \geq C$, and thus:
\begin{equation}
(b+1)^\ell > b^\ell + 1 \geq C + 1.
\end{equation}
By rearranging this equation, we conclude that $(b+1)^\ell - (C + 1) > 0$. As a result, 
\begin{equation}
|\tilde{\Y}^\ell| - |\tilde{\X}| = (b+1)^\ell - (C + 1) > 0.
\end{equation}
\end{proof}
% \textbf{Proposition}~\ref{prop:num_state} verifies that Prime expands the state space. 
To illustrate how the number of intermediate states (i.e., $|\tilde{\Y}^\ell| - |\tilde{\X}|$) may increase with $\ell$, we define a function $M(\ell, C) = (b + 1)^\ell - (C + 1)$ and evaluate its value under different $\ell$. As a concrete example, consider $C = 256$ and $\ell \in \{2, 4, 8\}$. Substituting these values yields $M(2, 256) = 32$, $M(4, 256) = 368$, and $M(8, 256) = 6304$. This example highlights that Prime can produce a substantial number of intermediate states by selecting $\ell$. This growth of $M(\ell, C)$ motivates our design choice for the embedding lookup table discussed in Section~\ref{sec:methodology:parameterization}, where we mention the infeasibility of directly modeling the embeddings for $\vy_t^i$.

\subsubsection{Carry-over Parameterization}
\label{apx:analysis:carry_over}

In Section~\ref{sec:methodology:parameterization}, we introduced the carry-over parameterization for MDM-Prime. In this section, we justify this approach by presenting \textbf{Proposition}~\ref{prop:param_condition} and discussing its practical implementation.

\begin{proposition}
\label{prop:param_condition}
Let $\V(\vy_t^i)$ denote the set of $\vy_0^i$ such that each $y_0^{i,j}$ is consistent with $y_t^{i,j}\in \Y$, i.e.,
\begin{equation}
\label{eq:apx:valid_set}
\V(\vy_t^i)\triangleq \{ \vy^i=[y^{i,1}, \cdots, y^{i,\ell}]  \in f(\X) \,\,\,s.t.\,\,\, (y^{i,j} = y^{i,j}_t) \vee (y^{i,j}_t = \textnormal{\texttt{m}}) \}.
\end{equation}
Given the parameterized distribution $p_\theta(\vy^i_0 | \vy_t)$ expressed as follows: 
\begin{equation}
\label{eq:apx:parameterization}
\begin{aligned}
    p_\theta(\vy^{i}_0|\vy_{t})=
    \begin{cases}
        \frac{\exp (E_\theta (\vy_0^i|\vy_t))}{ \sum_{\vy^i\in \V(\vy_t^i)} \exp (E_\theta (\vy^i|\vy_t)) },& \text{if $\vy^{i}_0\in \V(\vy_t^i)$},\\
        0,& \text{if $\vy^{i}_0\notin \V(\vy_t^i)$},
    \end{cases}
\end{aligned}
\end{equation}
the marginal distribution $p_\theta(y^{i,j}_0|\vy_{t})$ of Eq.~(\ref{eq:apx:parameterization}) satisfies the carry-over condition for all position $i,j$ where $y^{i,j}_t\in \Y$:
\begin{equation}
\label{eq:apx:carry_over_cond}
p_\theta(y^{i,j}_0|\vy_{t})=\sum_{y_0^{i,1},\,\cdots,\,y_0^{i,j-1},\,y_0^{i,j+1},\,\cdots,\,y_0^{i,\ell}\in \Y} p_\theta(y^{i,1}_0,\cdots,y^{i,\ell}_0|\vy_{t}) = \delta_{y^{i,j}_t}(y^{i,j}_0).
\end{equation}
\end{proposition}
\begin{proof}
Given $i,j$ such that $y_t^{i,j}\in\Y$, 
% then for all $\vy_0^i\in \V(\vy_t^i)$, we have $y_0^{i,j}=y_t^{i,j}$
according to Eq.~(\ref{eq:apx:valid_set}), $y_0^{i,j}=y_t^{i,j}$ is a necessary condition for non-zero probability:
\begin{equation}
\label{eq:apx:non_zero}
p_\theta(\vy^{i}_0|\vy_{t})>0 \,\,\,\Rightarrow\,\,\, y_0^{i,j}=y_t^{i,j}.
% ,\,\,\, \forall j\in \{1,\cdots,\ell\}.
\end{equation}
In other word, if $y_0^{i,j}\neq y_t^{i,j}$ for any $j$ is observed, then $p_\theta(\vy^{i}_0|\vy_{t})=0$. This indicates that the marginal $p_\theta(y^{i,j}_0|\vy_{t})=\sum_{y_0^{i,1},\,\cdots,\,y_0^{i,j-1},\,y_0^{i,j+1},\,\cdots,\,y_0^{i,\ell}\in \Y} p_\theta(\vy^i_0|\vy_{t})=0$ for any $y_0^{i,j}\neq y_t^{i,j}$.

Since $\sum_{y^{i,j}_0\in \Y} p_\theta(y^{i,j}_0|\vy_{t})=\sum_{\vy^{i}_0\in \Y^{\ell}} p_\theta(\vy^{i}_0|\vy_{t})=1$ by Eq.~(\ref{eq:apx:parameterization}), $p_\theta(y^{i,j}_0|\vy_{t})=1$ for $y_0^{i,j}= y_t^{i,j}$. Therefore, the marginal distribution corresponds to the Kronecker delta function as follows:
\begin{equation}
p_\theta(y_0^{i,j}| \vy_t) =
\begin{cases}
1, & \text{if } y_0^{i,j} = y_t^{i,j}, \\
0, & \text{if } y_0^{i,j} \neq y_t^{i,j}.
\end{cases}= \delta_{y_t^{i,j}}(y_0^{i,j}).
\end{equation}
% which corresponds to the Kronecker delta function:
% \begin{equation}
% p_\theta(y_0^{i,j} | \vy_t) 
% \end{equation}
\end{proof}

To understand this parameterization method, we examine its implementation in practice. The goal is to ensure that the marginal distribution $p_\theta(y_0^{i,j}| \vy_t)=\sum_{y_0^{i,1},\,\cdots,\,y_0^{i,j-1},\,y_0^{i,j+1},\,\cdots,\,y_0^{i,\ell}\in \Y} p_\theta(\vy^i_0|\vy_{t})$ is zero whenever the candidate state $\vy_0^i$ contains any $y_0^{i,j}$ that is inconsistent with $y_t^{i,j}$ (i.e., when $y_0^{i,j} \ne y_t^{i,j}$, as discussed in the proof of \textbf{Proposition}~\ref{prop:num_state}). Since the marginal probability is a sum over probabilities, it follows that each individual probability $p_\theta(\vy^i_0|\vy_{t})$ should be zero for $\vy^i_0$ with $y_0^{i,j}$ that is inconsistent with $y_t^{i,j}$. To enforce this, we use filters with $C$ entries to exclude the output logits of invalid $\vy^i_0$. For example, suppose $C = 7$, $\ell = 3$, and $b = 2$. If $\vy_t^i = (y_t^{i,1}, y_t^{i,2}, y_t^{i,3}) = (1, \texttt{m}, 0)$, we create $\ell$ filters (i.e., Filters 1-3) as follows:
\begin{itemize}[leftmargin=12pt]
\vspace{-0.5em}
\item For $y_t^{i,1} = 1$, Filter 1 excludes $\vy^i_0$ where the first position is not 1.
\item For $y_t^{i,2}=\texttt{m}$, Filter 2 excludes no state since no condition is needed to be satisfied at this position.
\item For $y_t^{i,3} = 0$, Filter 3 excludes $\vy^i_0$ where the third position is not 0.
\vspace{-0.5em}
\end{itemize}
\input{Figures/Appendix/implementation}

We then combine these element-wise filters using a logical AND operation, since any violation of an element-wise condition leads to an invalid $\vy^i_0$ that has inconsistency with $y_t^{i,j}$. The resulting combined filter excludes $\vy^i_0\notin\V(\vy_t^i)$, retaining only $\vy_0^i\in\V(\vy_t^i)$. Subsequently, we apply a softmax over the logits corresponding to $\vy_0^i\in\V(\vy_t^i)$ to define the probability distribution $p_\theta(\vy^{i}_0|\vy_{t})$. An illustrative example of this procedure is depicted in Fig.~\ref{fig:implementation}.

For efficient implementation, we precompute and cache a lookup table with size $|\tilde{\Y}|\times |f(\X)|=(b+1)\times C$, which contains all possible filters for $y^{i,j}_t$. During the forward pass, we query this lookup table using $y^{i,j}_t$ to retrieve the corresponding filters, which are then applied to the logits to zero out the probability of $\vy^i_0\notin\V(\vy_t^i)$.

\subsubsection{Analysis of the Target Length}
\label{apx:analysis:hyperparameter}
Prime introduces a parameter $\ell$ that controls the target length of sub-token sequences. This section analyzes its impact on performance and provides a practical guideline for selecting an appropriate value of $\ell$.

\input{Figures/Appendix/isr_performance}
\input{Figures/Appendix/l_performance}
Section~\ref{sec:experiment} demonstrates that, although increasing $\ell$ generally correlates with improved performance, the relationship is not strictly monotonic, i.e., larger values of $\ell$ do not always lead to better results. As larger values of $\ell$ typically induce more intermediate states (see Section~\ref{apx:analysis:intermediate}), we hypothesize that the associated increase in learning complexity may degrade performance under fixed model capacity (e.g., model size). This observation motivates our development of a strategy for identifying an effective $\ell$.

Through empirical analysis across both image and text datasets, we find that selecting $\ell$ near the elbow point of the Idle Step Ratio (ISR) curve (i.e., Eq.~(\ref{eq:apx:isr})) yields strong performance. Fig.~\ref{fig:apx:isr_performance} presents ISR curves for various values of $\ell$, with elbow points indicated by red stars. The relationship between ISR elbow points and MDM-Prime's performance is depicted in Fig.~\ref{fig:l_performance}, where these points align with near-optimal results across most evaluations. Based on the above findings, we recommend $\ell = 2$ for image datasets and $\ell = 4$ for text datasets.

%% file: Figures/Appendix/implementation.tex
\begin{figure}[t]
    \centering
    \footnotesize
    \includegraphics[width=\linewidth]{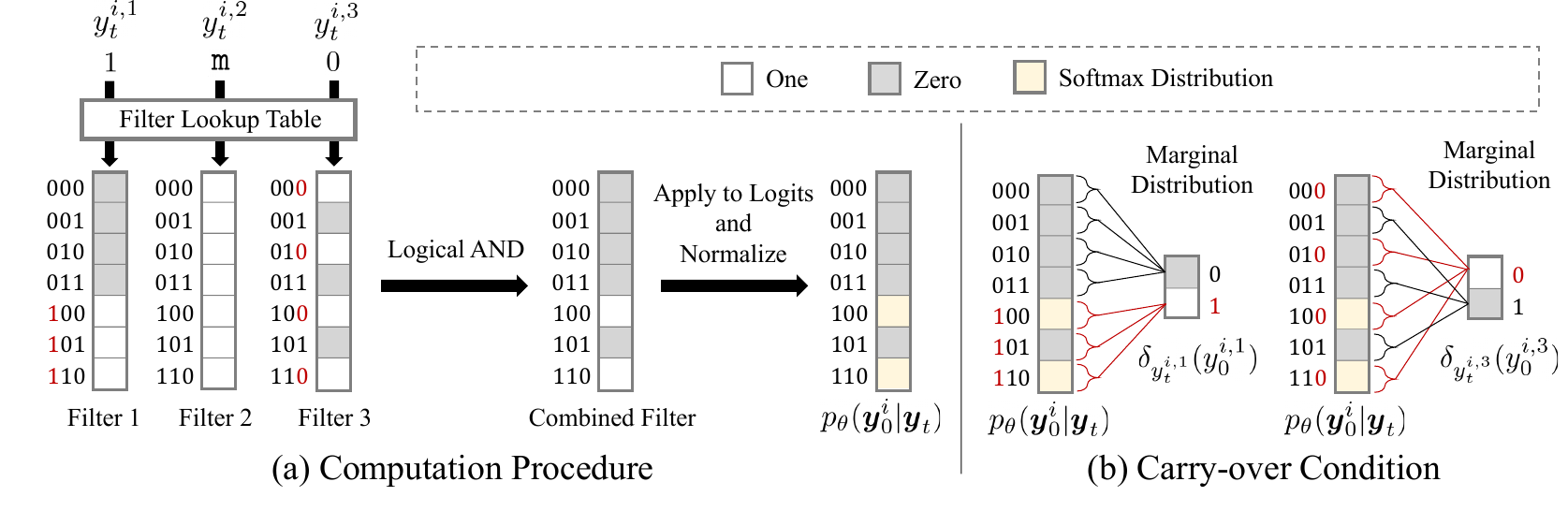}
    \vspace{-2em}
    \caption{An illustration of (a) the computation procedure of the carry-over parameterization and (b) the corresponding carry-over condition (i.e., Eq.~(\ref{eq:apx:carry_over_cond})) on the marginal distributions. In this example, $C = 7$, $\ell = 3$, and $b = 2$ (i.e., binary encodings). In (a), the filters are precomputed and stored in a lookup table indexed by $y_t^{i,j}$. These filters exclude logits corresponding to $\vy_0^i$ with $y_0^{i,j}$ that are inconsistent with $y^{i,j}_t\in\Y$ (i.e., $y^{i,1}_t=1$ and $y^{i,3}_t=0$). For each position $j\in\{1,\cdots,\ell\}$, a filter is queried from the lookup table using $y^{i,j}_t$. Then, a combined filter is then constructed by applying a logical AND operation across $\ell$ filters. Finally, the softmax distribution $p_\theta(\vy_0^i|\vy_t)$ is established by normalizing the filtered logits. In (b), the resulting distribution $p_\theta(\vy_0^i|\vy_t)$ satisfies the carry-over condition for $y_t^{i,j}\in \Y$ as defined in Eq.~(\ref{eq:apx:carry_over_cond}).}
    \label{fig:implementation}
\end{figure}

%% file: Figures/Appendix/isr_performance.tex
\begin{wrapfigure}[]{r}{0.48\linewidth}
    \centering
    \vspace{-0.25em}
    \footnotesize
    \includegraphics[width=\linewidth]{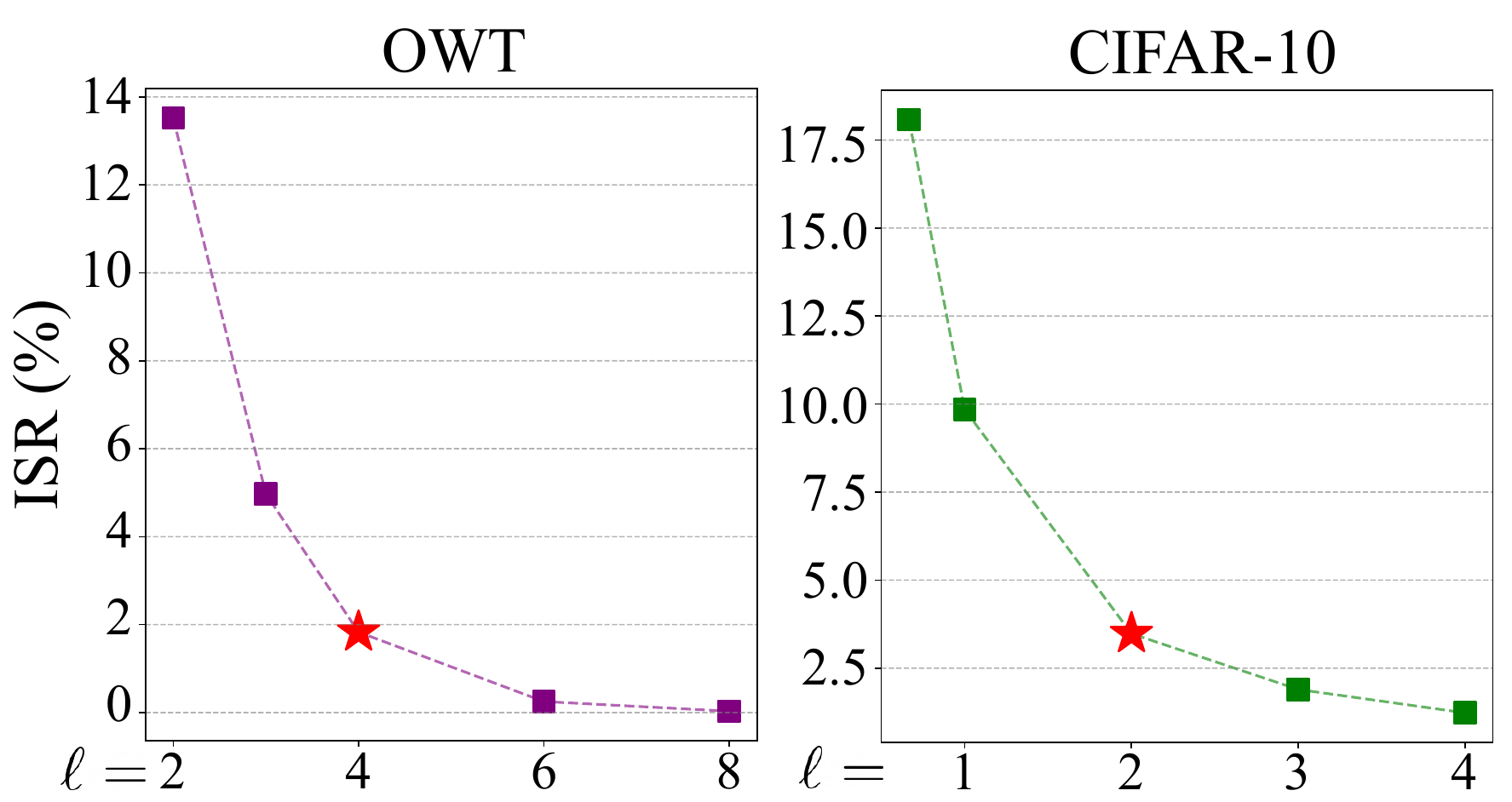}
    \vspace{-1.5em}
    \caption{ISR under different choices of $\ell$. The elbow points are highlighted using red stars.}
    \vspace{-1.5em}
\label{fig:apx:isr_performance}
\end{wrapfigure}

%% file: Figures/Appendix/l_performance.tex
\begin{figure}[t]
    \centering
    \footnotesize
    \includegraphics[width=\linewidth]{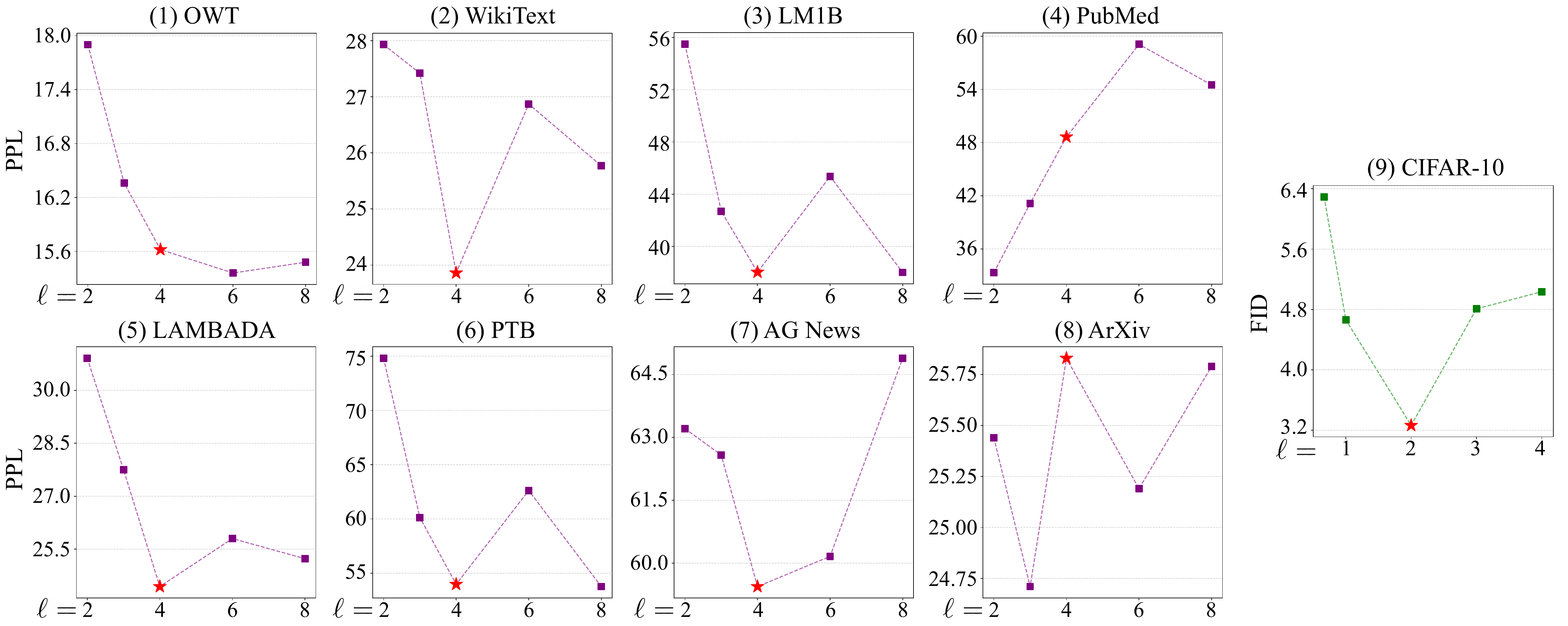}
    \caption{(1)-(8) PPL of MDLM-Prime with different $\ell$ evaluated on OWT and seven zeroshot datasets. (9) FID of MDM-Prime across varying $\ell$ evaluated on CIFAR-10. Lower values indicate better performance. Red stars highlight the $\ell$ values corresponding to the elbow points of the ISR curves (see Fig.~\ref{fig:apx:isr_performance}). These points yield near-optimal performance across most evaluations.}
    \label{fig:l_performance}
\end{figure}

%% file: Sections/a4_architecture.tex
\subsection{Model Architecture Designs}
\label{apx:architecture}
This section compares a number of architectural variants of MDM-Prime, with a focus on the design choices for the output logit layer and the input embedding layer. For reference, the standard MDM architecture is shown in Figs.~\ref{fig:architecture}~(o1) and (i1), which illustrate its output logit and embedding layers, respectively. In standard MDM, the output layer produces $L$ logits with $C$ entries, while the embedding layer processes inputs $\vx_t \in \tilde{\X}^L$ to produce $L$ $D$-dimensional embeddings.

\paragraph{Output Logit Layer.}~In Section~\ref{sec:methodology:parameterization}, we introduce two alternative designs for the output layer. The first design assumes independence among sub-tokens, leading to a factorized form of the probability distribution: $p_\theta(\vy^{i}_0 | \vy_t) = \prod_{j=1}^{\ell} p_\theta(y^{i,j}_0 | \vy_t)$. This formulation can be implemented using $\ell \times L$ logits, each with $b$ entries, as illustrated in Fig.~\ref{fig:architecture}~(o3). The second design (i.e., Fig.~\ref{fig:architecture}~(o2)) models a joint distribution over sub-token sequences, producing $L$ logits with $C$ entries, consistent with the standard MDM architecture shown in Fig.~\ref{fig:architecture}~(o1). As discussed in Section~\ref{sec:methodology:parameterization}, we adopt the joint distribution design over the factorized alternative (Fig.~\ref{fig:architecture}~(o3)) due to two key limitations of the latter: (1) the independence assumption, and (2) the potential to generate invalid samples. These issues hinder the model’s ability to capture complex data distributions (see Fig.~\ref{fig:methodology:toy}) and limit its applicability to real-world generative tasks.

\paragraph{Embedding Layer.}~As discussed in Section~\ref{sec:methodology:parameterization}, creating an embedding lookup table for $\vy_t^i$ is impractical due to the resulting growth in the number of parameters in it (also see Section~\ref{apx:analysis:intermediate}). To address this, we model sub-token embeddings by creating a lookup table for $y_t^{i,j} \in \tilde{\Y}$.  Since the input and output sequences differ in length, we introduce a simple merging strategy based on concatenation, as illustrated in Fig.~\ref{fig:architecture}~(i2). In this approach, each sub-token is embedded into a vector of size $\frac{D}{\ell}$, and the $\ell$ sub-token embeddings are then concatenated to form $D$-dimensional token-level embeddings.

\input{Figures/Appendix/architecture}
Fig.~\ref{fig:architecture}~(i3) shows an alternative design based on Perceiver~\cite{jaegle2021perceiver}, which employs cross-attention with learnable latent queries to merge sub-token embeddings. In this design, each sub-token is embedded into a $D$-dimensional vector, producing $\ell \times L$ sub-token embeddings. These sub-token embeddings are then used as the key and value matrices in a cross-attention layer, while a learnable query matrix $Q_\theta \in \mathbb{R}^{D \times L}$ retrieves the token-level embeddings. The resulting outputs are $L$ $D$-dimensional token embeddings. Although this approach is well-established for processing high-dimensional inputs in transformer-based models, it introduces substantial computational overhead~\cite{jaegle2021perceiver}. Moreover, empirical results in Appendix~\ref{apx:experiments:ablation} show that it does not lead to performance gains. %in either text or image generation tasks. 
Based on these observations, we adopt the simple concatenation-based merging strategy in MDM-Prime's embedding layer.

%% file: Figures/Appendix/architecture.tex
\begin{figure}[t]
    \centering
    \footnotesize
    \includegraphics[width=\linewidth]{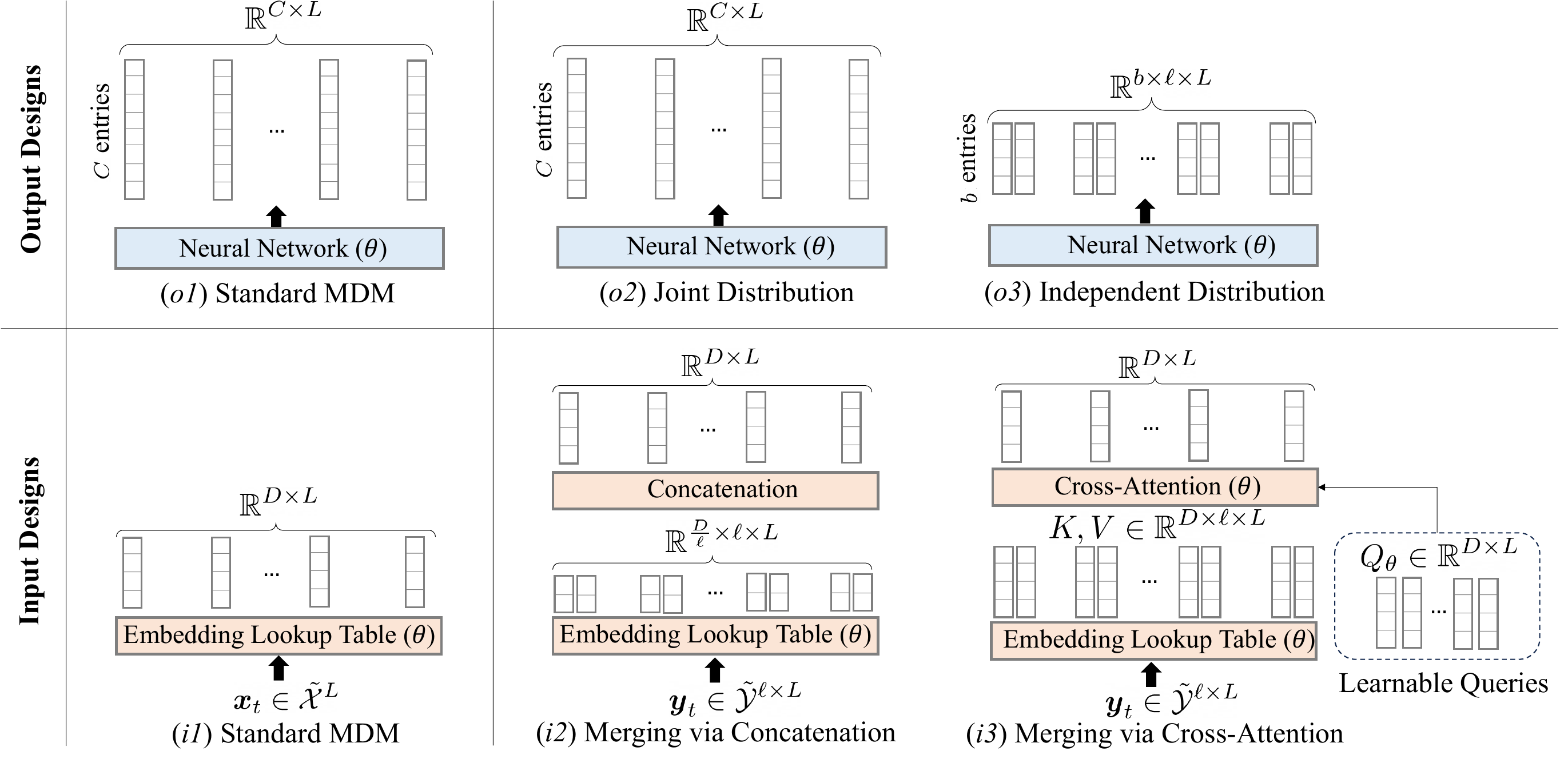}
    \vspace{-2em}
    \caption{Comparison of input and output layer designs in MDM and MDM-Prime. The top row illustrates three types of output layer designs: (o1) standard MDM, (o2) MDM-Prime with joint distribution, and (o3) MDM-Prime with independent distribution. The bottom row shows three embedding layer designs: (i1) standard MDM embedding layer, (i2) MDM-Prime embedding layer with concatenation, and (i3) MDM-Prime embedding layer with Perceiver~\cite{jaegle2021perceiver} cross-attention.}
    \label{fig:architecture}
\end{figure}

%% file: Sections/a5_setups.tex
\subsection{Experimental Setups}
\label{apx:setups}
In this section, we provide additional implementation details and the hyperparameter settings used in our experiments. The configurations for the two-dimensional synthetic experiment, image generation experiment, and text generation experiment are presented in Sections~\ref{apx:setups:toy}, \ref{apx:setups:image}, and \ref{apx:setups:text}, respectively.

\subsubsection{Two-Dimensional Example}
\label{apx:setups:toy}
\textbf{Dataset.}~The experiment in Fig.~\ref{fig:methodology:toy} is conducted on a synthetic dataset constructed by converting pixel values from an image in the Cat dataset~\cite{zhang2008cat} into a two-dimensional probability histogram. The image is first cropped to $512\times512$ pixels and converted to grayscale. A probability distribution is then constructed according to the normalized pixel intensity values. In this example, $\X^L=\{0, \dots, 511\}^2$, where $L=2$ and $C=512$. The sample $\vx_0\in\X^L$ represents the coordinate of the figure ($512\times 512$).

\textbf{Training and Implementation Details.}~The network architecture is a four-layered multilayer perceptron (MLP) with hidden dimension $512$ and Swish~\citep{ramachandran2017searching} as its activation function. The models are trained using the Adam optimizer~\citep{kingma2015adam} with a learning rate of $1\times 10^{-3}$ and a batch size of $4,096$. The training is performed on a single NVIDIA A40 GPU with 48 GB memory.

\subsubsection{Image Generation}
\label{apx:setups:image}
\textbf{Datasets and Evaluation Methods.}~The experiments described in Section~\ref{sec:experiment} are conducted on the CIFAR-10~\cite{krizhevsky2009cifar10} and ImageNet-32~\cite{chrabaszcz2017imagenet} datasets. For both datasets, $L=32\times 32 \times 3$ and $C=256$. The CIFAR-10 training set contains $50,000$ images, while the ImageNet-32 dataset comprises $1,281,149$ training images and $49,999$ validation images. The sample quality is assessed using the Fréchet Inception Distance (FID)~\cite{heusel2017fid} and Inception Score (IS)~\cite{barratt2018is}, implemented via the \texttt{torchmetrics.image.fid} and \texttt{torchmetrics.image.inception} libraries, respectively. The corrector steps~\cite{campbell2022ctmc, gat2024dfm} are adopted during sampling. For CIFAR-10, FID is computed using the training set as the reference distribution, whereas for ImageNet-32, the validation set serves as the reference distribution. This evaluation protocol is consistent with prior works~\cite{karras2020stylegan, ho2020ddpm, song2021scoreflow, austin2021structurediff, campbell2022ctmc, gat2024dfm, oord2016pixelcnn, nisonoff2025guidance, song2019generative, lipman2023flowmatching, chen2023bitdiff, chao2023investigating, bartosh2024neuraldiff, tran2019msgan, zheng2023ode, albergo2023stochastic, kim2022soft}.

\textbf{Training and Implementation Details.}~Our model architecture follows~\cite{gat2024dfm} to adapt the ablated diffusion model (ADM)~\cite{dhariwal2021diffusion}, which has a U-Net structure with a symmetric design. It comprises a downsampling module, a bottleneck module, and an upsampling module. Each module contains multiple residual convolutional blocks, with attention layers incorporated at selected resolutions. The downsampling and upsampling modules use an embedding dimension of $96$ and channel multipliers of $[3, 4, 4]$, respectively. The depth of each block is set to $5$. For attention layers, the number of head channels is set to $64$. The network is optimized using the Adam optimizer~\cite{kingma2015adam} with $\beta_1 = 0.9$, $\beta_2 = 0.999$, and a learning rate of $1\times10^{-4}$. The scheduling function is defined as a third-order polynomial, given by $\alpha_t = (1 - t)^3$. \lance{Following prior work~\cite{austin2021structurediff, gat2024dfm}, the coefficient $\frac{\alpha'_t}{1 - \alpha_t}$ in the loss term is omitted during training to improve sample quality. As for sampling, temperature scheduling~\cite{gat2024dfm}, corrector steps~\cite{gat2024dfm}, and timestep scheduling~\cite{shi2024simplifieddiff} are incorporated into both MDM and MDM-Prime to enhance performance. The sampling parameters are selected via grid search (see Section~\ref{apx:experiments:fid}).} The model is trained with a batch size of $512$ for $4,250$ epochs on CIFAR-10 and $1,000$ epochs on ImageNet-32. The training is performed on eight NVIDIA L40 GPUs with 48 GB memory.

\subsubsection{Text Generation}
\label{apx:setups:text}
\textbf{Datasets and Evaluation Methods.}~The training is performed on the OpenWebText (OWT) dataset~\cite{gokaslan2019owt}. Text data is tokenized using the GPT-2 tokenizer~\cite{radford2019language}, which defines $L = 1,024$ and $C = 50,257$. Since some sequences in OWT exceed the maximum length $L$, they are concatenated with the \texttt{<eos>} token as a separator and then wrapped into segments of length $1,024$. In addition, the first and last tokens of each sequence are set to \texttt{<eos>}. As OWT does not include an official validation split, we follow the procedure in~\cite{sahoo2024simplifieddiff} by reserving the last $100,000$ samples for validation. The above data preprocessing and evaluation setup is consistent with~\cite{sahoo2024simplifieddiff}.

\textbf{Training and Implementation Details.}~Following prior works~\cite{lou2024sedd, sahoo2024simplifieddiff, shi2024simplifieddiff, xu2025ebmdiff, arriola2025block}, our model architecture is a diffusion transformer (DiT)~\cite{Peebles2022DiT} with rotary positional embeddings~\cite{su2023rotary}. The model consists of $12$ DiT blocks with a hidden dimension of $768$ and $12$ attention heads, consistent with the configuration used in~\cite{sahoo2024simplifieddiff}. A dropout rate of $0.1$ is applied throughout training. The scheduling function is defined as $\alpha_t = 1 - t$. The network is optimized using the AdamW optimizer~\cite{loshchilov2019adamw} with $\beta_1 = 0.9$, $\beta_2 = 0.999$, and a linear learning rate warm-up from $0$ to $3 \times 10^{-4}$. The model is trained with a batch size of $128$ for $1$ million steps, corresponding to a total of $262$ billion tokens~\cite{sahoo2024simplifieddiff}. To facilitate efficient training, the model is first trained without the carry-over parameterization for 900K iterations, and then enables this mechanism in the final 100K iterations (see Section~\ref{apx:experiments:time_eval}). The training is performed on eight NVIDIA L40 GPUs with 48 GB memory.

%% file: Sections/a6_experiments.tex
\subsection{Supplementary Experiments}
\label{apx:experiments}
In this section, we present additional experimental results. \lance{Section~\ref{apx:experiments:fid} discusses a technique for tuning the hyperparameters based on relative FID values.} Section~\ref{apx:experiments:ablation} provides ablation studies on the embedding layer designs and the carry-over parameterization method. Section~\ref{apx:experiments:time_eval} reports the time cost per training iteration under different values of $\ell$. \lance{Section~\ref{apx:experiments:curves} compares the training curves of MDLM, MDLM-Prime, and ARM.} Section~\ref{apx:experiments:qualitative} presents qualitative results for both the text and image generation tasks. Finally, Section~\ref{apx:experiments:traj} provides visualization of the sampling process of MDM-Prime.

\subsubsection{Hyper-Parameter Tuning based on Relative FID}
\label{apx:experiments:fid}
\input{Figures/Appendix/fid_relfid}
\lance{Since computing the FID score is computationally intensive, we tune hyperparameters (e.g., batch size, dropout rate, scheduling coefficients $\alpha_t$, and sampling parameters) using FID scores estimated from a smaller number of samples. We refer to this approximation as relative FID. Empirically, we find that relative FID is positively correlated with the standard FID computed on the full dataset. Fig.~\ref{fig:fid_relfid} illustrates this trend, comparing relative FID computed with 2,048 samples to standard FID computed with 50,000 samples. In our experiments, we use relative FID to guide hyperparameter search for both the baseline methods (i.e., MDM and MDM-Mixture) and our proposed method, and report their best-performing results in the benchmark comparison.}

\subsubsection{Ablation Study}
\label{apx:experiments:ablation}

\input{Tables/Appendix/carry_over_zeroshot}
\input{Tables/Appendix/carry_over}
\paragraph{Carry-over Parameterization.}~In this subsection, we evaluate the effectiveness of the carry-over parameterization. Table~\ref{tab:apx:carry_over} compares the performance of MDLM-Prime trained with and without carry-over on the OWT dataset. It is observed that incorporating the carry-over parameterization consistently leads to slightly improved performance across different values of $\ell$. In contrast, when the model is evaluated using unseen data from zeroshot datasts, the difference becomes significant. A comparison is offered in Table~\ref{tab:apx:carry_over_zeroshot}. This parameterization method consistently offers performance gains on LAMBADA. In addition, MDLM-Prime without the carry-over parameterization has noticeably inferior performance (i.e., difference $>30$) on specific text domains such as PubMed. We hypothesize that the carry-over parameterization, by explicitly removing predictions of $\vy_0$ inconsistent with $\vy_t$, effectively reduces the uncertainty in the model's outputs and therefore enhances generalization to unseen domains.

\paragraph{Embedding Layer Designs.}~In this subsection, we compare the performance of MDM-Prime when trained using two embedding-merging strategies: the concatenation-based design (denoted as (o2)) and the cross-attention-based design (denoted as (o3)), both of which are detailed in Appendix~\ref{apx:architecture}. We first evaluate these approaches on the image generation task using the CIFAR-10 dataset with $\ell=2$. The concatenation-based design (o2) achieves a superior FID score of 3.26, compared to the cross-attention-based design (o3), which yields an FID score of 3.98. %On the other hand, for text generation, using the cross-attention-based design leads to a noticeable degradation in performance, with perplexity increasing from 18.04 to 57.37 for $\ell=2$ on the OWT dataset. 
The results indicate that the concatenation-based design offers better overall performance across both modalities.

\subsubsection{Runtime Evaluation}
\label{apx:experiments:time_eval}
\input{Tables/Appendix/time_eval}
\input{Tables/Appendix/co_eval_ablation}
Table~\ref{tab:time_eval} presents a comparison of the time required to optimize the model parameterized with and without the carry-over condition (i.e., Eq.~(\ref{eq:methodology:marginalization})) under different choices of $\ell$. We observe that larger values of $\ell$ incur higher computational costs for the setup with carry-over condition. The increase in runtime arises from the filter lookup table calculation required by the carry-over parameterization (see Section~\ref{apx:analysis:carry_over}). Due to the 48 GB memory limitation of our available GPUs, parallelized querying of the filter lookup table (with a batch size of $128$) often results in out-of-memory errors. To mitigate this issue, we implement the lookup operations sequentially in our text experiments, which introduces additional computational overhead as $\ell$ increases. To maintain a runtime comparable to the baseline, we first train the model without the carry-over parameterization for 900K iterations, and then enable this mechanism during the final 100K iterations. Alternatively, the carry-over parameterization can be applied only at inference time. As shown in Table~\ref{tab:apx:co_eval_ablation}, this setup—training without the carry-over parameterization but evaluating with it—achieves performance comparable to, or slightly better than, that of the model trained and evaluated with the carry-over parameterization, while avoiding the higher training cost. Nonetheless, to preserve the mathematical formality of our proposed framework, we adopt the carry-over parameterization during both training and evaluation.

\input{Tables/Appendix/nfe}

\lance{As for inference efficiency, although idle sampling steps in masked diffusion models can be executed efficiently using cached model output probabilities (since the model is timestep-independent~\cite{sahoo2024simplifieddiff}), determining the unmasked elements (i.e., resampling) during these steps still incurs computational overhead, even without extra forward passes. Consequently, the sampling cost of MDLM grows substantially when targeting a higher effective NFE (i.e., NFEs that cannot be cached). As shown in Table~\ref{tab:apx:nfe}, MDLM requires 500,000 discretized steps to achieve an effective 1,023 NFE, severely limiting sampling efficiency. In contrast, MDLM-Prime ($\ell=4$) attains the same effective NFE with only 1,044 discretized steps, highlighting its substantially improved efficiency.}

\subsubsection{Training Curves}
\label{apx:experiments:curves}
\input{Figures/Appendix/curves}
\lance{Fig.~\ref{fig:curves} compares the training and evaluation curves of MDLM, MDLM-Prime ($\ell=4$), and ARM. MDLM-Prime ($\ell=4$) achieves the lowest evaluation PPL among the three models.}

\subsubsection{Sample Quality Evaluation}
\label{apx:experiments:qualitative}

\paragraph{Image Generation.}~In this subsection, we provide additional quantitative and qualitative results to assess sample quality in image generation tasks. Table~\ref{tab:cifar_imagenet_step} presents a comparison of FID scores between MDM-Prime and its baseline (i.e., a standard MDM). The results show that MDM-Prime outperforms MDM across different number of function evaluations (NFE). In addition, Fig.~\ref{fig:imputation_all} provides a number of imputation examples, while Fig.~\ref{fig:cifar10} 
% and \ref{fig:imagenet} 
presents qualitative results showcasing multiple uncurated samples generated by MDM-Prime.

\input{Figures/Appendix/gen_ppl_eval}
\paragraph{Text Generation.}~In this subsection, we present additional sample quality evaluations for text generation tasks. Fig.~\ref{fig:apx:gen_ppl_eval} presents the generative perplexity (Gen PPL) results for ARM, MDLM, and MDLM-Prime. Gen PPL is calculated by evaluating the perplexity of generated samples using a pretrained large language model. To ensure a comprehensive assessment, we report Gen PPL obtained from both order-agnostic models (e.g., LLaDA-8B~\cite{nie2025llmdiff}) and  order-specific models (e.g., GPT-2 Large~\cite{radford2019language}). Following~\cite{zheng2024mdm}, we adopt 64-bit floating-point precision to enhance sampling accuracy. We observe opposite trends depending on the pretrained model: MDLM and MDLM-Prime outperform ARM under LLaDA-based Gen PPL, whereas ARM achieves lower Gen PPL than MDLM and MDLM-Prime when evaluated using GPT-2.

In addition, to qualitatively evaluate MDLM-Prime, we present both unconditional and conditional generation samples. Unconditional generation results are shown in Fig.~\ref{fig:uncond_text_sample}. For conditional generation, we provide the model with prefix (i.e., Fig.~\ref{fig:apx:prefix}) and suffix (i.e., Fig.~\ref{fig:apx:suffix}) texts sourced from an online article describing Rabindranath Tagore’s poems (\href{https://timesofindia.indiatimes.com/life-style/books/features/10-timeless-poems-by-rabindranath-tagore/amp_etphotostory/75593222.cms}{link}), and assign the model to generate the middle content. Figs.~\ref{fig:apx:poem_2bit}-\ref{fig:apx:poem_6bit} present some samples generated using MDLM-Prime.

\subsubsection{Visualization of Sampling Processes}
\label{apx:experiments:traj}
Unlike continuous diffusion models (e.g.,~\cite{song2021scoreflow, song2019generative}), visualizing the sampling processes of MDM and MDM-Prime is nontrivial due to the presence of masked tokens and sub-tokens. To qualitatively analyze the generation behaviors of MDM and MDM-Prime, we develop a visualization technique that illustrates the evolution of samples throughout the diffusion process. Instead of directly visualizing the latent variables (i.e., $\vx_t$ for MDM or $\vy_t$ for MDM-Prime), we visualize the model’s predictions of the final unmasked sample. Specifically, we show $\vx_0 \sim p_\theta(\cdot\, | \vx_t)$ for MDM, and $\vx_0 = f^{-1}(\vy_0)$ where $\vy_0 \sim p_\theta(\cdot\, | \vy_t)$ for MDM-Prime.

Fig.~\ref{fig:sampling_traj} shows the evolution of samples $\vx_0$ generated by MDM-Prime trained with different values of $\ell$ on CIFAR-10. For each setup, we capture a snapshot every 25 timesteps, with the total number of function evaluations (NFE) fixed at 500. The snapshots are displayed from left to right, depicting the progressive refinement from noisy initialization to the final output. 

The text generation processes are presented in Figs.~\ref{fig:text_sampling_traj_1bit}-\ref{fig:text_sampling_traj_4bit}. Since the coarse-to-fine transitions of $\vx_0$ are less visually discernible in text generation, we additionally illustrate the denoising progression using token-wise masked ratios. In the figures, darker shades of blue indicate a higher degree of masking, while lighter colors suggest that a token is closer to its final predicted state. The results demonstrate that MDLM-Prime with larger values of $\ell$ enables a more fine-grained denoising process. In contrast, the original MDLM employs a binary masking scheme, where each token is either masked or unmasked. The results thus highlight the property of the proposed Prime method.

\input{Figures/Appendix/sampling_traj}

\subsection{Limitations}
\label{apx:limitation}
A core assumption in MDM is that tokens (i.e., $x_0^1, \cdots,x_0^L$) are conditionally independent given the latent representation $\vx_t$, i.e., $x^i_0 \ind x^j_0 \mid \vx_t$ for $i\neq j$. Under this assumption, many recent works~\cite{nie2025llmdiff, sahoo2024simplifieddiff, shi2024simplifieddiff, austin2021structurediff, campbell2022ctmc, gat2024dfm, zheng2024mdm} factorize the conditional distribution as $p_\theta (\vx_0|\vx_t) = \prod_{i=1}^L p_\theta (x^i_0|\vx_t)$, which leads to the sum of log-probability terms in the evidence upper bound (i.e., Eq.~(\ref{eq:background:diffusion_elbo})). While this factorized parameterization offers significant advantages in sampling and training efficiency, some recent studies~\cite{xu2025ebmdiff, liu2025dcd} have shown that it may degrade performance due to its inability to model inter-token dependencies.

In this work, we adopt the same conditional independence assumption as prior studies~\cite{nie2025llmdiff, sahoo2024simplifieddiff, shi2024simplifieddiff, austin2021structurediff, campbell2022ctmc, gat2024dfm, zheng2024mdm}, and extend it to our sub-token formulation. Specifically, we assume $\vy^i_0 \ind \vy^j_0 \mid \vy_t$ for $i\neq j$, and accordingly factorize the conditional distribution as $p_\theta(\vy_0 | \vy_t) = \prod_{i=1}^L p_\theta(\vy^i_0 | \vy_t)$. As discussed in Section~\ref{sec:methodology}, further factorizing each sub-token sequence prediction as $p_\theta(\vy^i_0 | \vy_t) = \prod_{j=1}^\ell p_\theta(y^{i,j}_0 | \vy_t)$ leads to a clear drop in performance. Therefore, we retain the inter-token factorization (i.e., $p_\theta(\vy_0 | \vy_t) = \prod_{i=1}^L p_\theta(\vy^i_0 | \vy_t)$) to strike a balance between accuracy and computational efficiency.

\subsection{Broader Impacts and Future Works}
\label{apx:impacts}
This paper investigated whether discrete data can be effectively represented and reconstructed using sub-token representations. We proposed MDM-Prime as a simple yet effective instantiation of this idea. Given its superior performance across both image and text generation tasks, MDM-Prime holds potential for positive societal impact. To support broader scientific contributions, we outline two potential directions for extending the proposed framework and guiding future research in this area:

\paragraph{Learnable Transformations for Discrete Data.}~MDM-Prime adopts base-$b$ encoding as an invertible mapping to extend discrete data into longer sequences of sub-tokens. Analogous to normalizing flows in the continuous domain (e.g.,~\cite{Dinh2014NICENI, Dinh2016DensityEU}), which parameterize invertible transformations using carefully designed model architectures, we anticipate that this discrete transformation can likewise be parameterized and optimized during training \lance{to encode the semantics of each token}. As invertible modeling in the discrete domain remains underexplored, it presents opportunities to advance discrete generative modeling.

\paragraph{Capturing Inter-token Dependencies.}~Our current decoder implementation follows prior works~\cite{nie2025llmdiff, sahoo2024simplifieddiff, shi2024simplifieddiff, austin2021structurediff, campbell2022ctmc, gat2024dfm, zheng2024mdm}, which assume conditional independence between tokens (see Section~\ref{apx:limitation}). While recent approaches~\cite{xu2025ebmdiff, liu2025dcd} have proposed methods to relax this assumption, they require training an additional autoregressive model to guide the sampling process of MDM, resulting in substantial computational overhead. Hence, developing a more efficient method for MDM to model inter-token joint distributions represents a promising direction for future work.

\clearpage

\input{Figures/Appendix/text_sampling_traj}
\input{Figures/Appendix/text_sample}
\input{Figures/Appendix/imputation_all}
\input{Figures/Appendix/cifar10}
\input{Figures/Appendix/imagenet}

%% file: Figures/Appendix/fid_relfid.tex
\begin{wrapfigure}[]{r}{0.4\linewidth}
    \centering
    \vspace{-1.5em}
    \footnotesize
    \includegraphics[width=\linewidth]{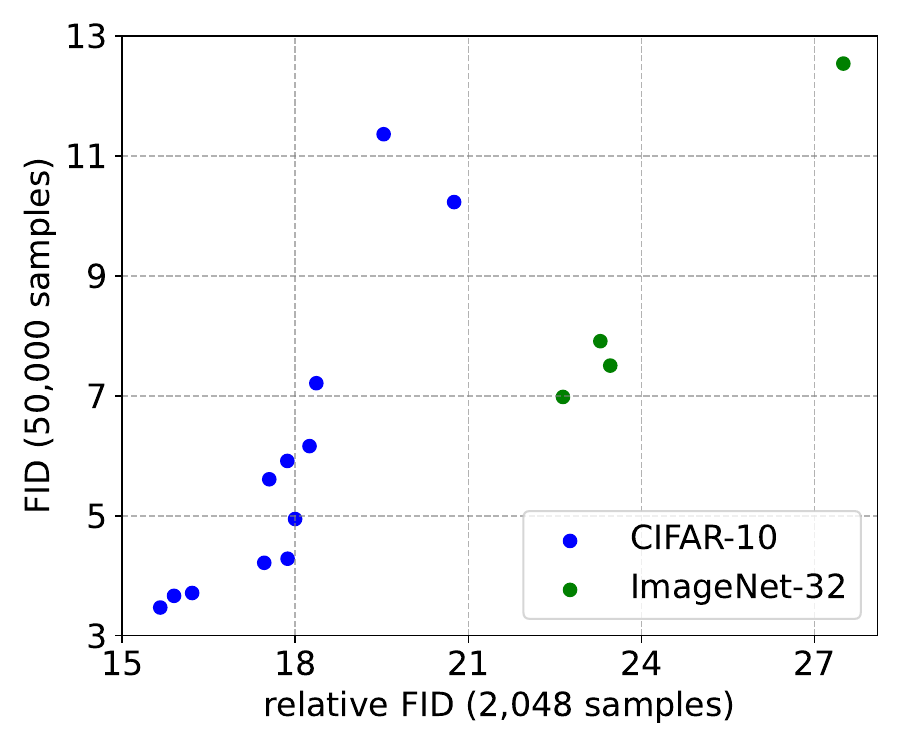}
    \vspace{-1.5em}
    \caption{FID versus relative FID.}
    \vspace{-1.75em}
    \label{fig:fid_relfid}
\end{wrapfigure}

%% file: Tables/Appendix/carry_over_zeroshot.tex
\begin{table}[t]
    \renewcommand{\arraystretch}{1}
    \newcommand{\boldtoprule}{\toprule[1.3pt]}
    \newcommand{\boldbottomrule}{\bottomrule[1.3pt]}
    \centering
    \large
    \vspace{-0.5em}
    \caption{Zero-shot validation PPL evaluated on seven textual datasets. Lower values correspond to better performance. Methods marked with * incorporate an autoregressive formulation. MDLM-Prime with carry-over parameterization achieves improved results on LAMBADA and PubMed.}
    \vspace{-0.25em}
    \resizebox{\linewidth}{!}{%
    \begin{tabular}{ccccccccc}
        \boldtoprule
                              & Carry-over  & LAMBADA & WikiText & PTB & LM1B & AG News & PubMed & ArXiv \\
        \hline
        ARM*~\cite{sahoo2024simplifieddiff} & - & 51.28 & 25.75 & 82.05  & 51.25 & \textbf{52.09}  & 49.01 & 41.73 \\
        \rowcolor{gray!15}
        MDLM~\cite{sahoo2024simplifieddiff} & \cmark & $\leq$47.52 & $\leq$32.83 & $\leq$95.26  & $\leq$67.01 & $\leq$61.15 & $\leq$41.89 & $\leq$37.37 \\
        MDLM-Prime ($\ell=2$) &        & $\leq$34.61 &  $\leq$31.94 & $\leq$91.85 & $\leq$65.25 & $\leq$63.74 &  $\leq$63.51 & $\leq$25.95 \\
        \rowcolor{gray!15}
        MDLM-Prime ($\ell=2$) & \cmark & $\leq$30.91 &  $\leq$27.93 & $\leq$74.81 & $\leq$55.50 & $\leq$63.21 & $\leq$\textbf{33.32} & $\leq$25.44 \\
        MDLM-Prime ($\ell=3$) &        & $\leq$34.55 &  $\leq$31.00 & $\leq$55.46 & $\leq$41.15 & $\leq$62.51 & $\leq$171.74 & $\leq$29.44  \\
        \rowcolor{gray!15}
        MDLM-Prime ($\ell=3$) & \cmark & $\leq$27.75 & $\leq$27.42 & $\leq$60.13 & $\leq$42.69 & $\leq$62.58 & $\leq$41.14  & $\leq$24.71 \\
        MDLM-Prime ($\ell=4$) &        & $\leq$25.79 & $\leq$\textbf{23.52} & $\leq$\textbf{49.90} & $\leq$37.70 & $\leq$59.65 & $\leq$210.96 & $\leq$\textbf{23.19} \\
        \rowcolor{gray!15}
        MDLM-Prime ($\ell=4$) & \cmark & $\leq$24.44 & $\leq$23.86 & $\leq$53.98 & $\leq$38.02 & $\leq$59.44  & $\leq$48.64 & $\leq$25.83 \\
        MDLM-Prime ($\ell=6$) &        & $\leq$25.89 & $\leq$25.86 & $\leq$52.42 & $\leq$\textbf{37.65} & $\leq$56.68  & $\leq$497.39 & $\leq$24.99 \\
        \rowcolor{gray!15}
        MDLM-Prime ($\ell=6$) & \cmark & $\leq$25.80 & $\leq$26.87 & $\leq$62.62 & $\leq$45.36 & $\leq$60.16  & $\leq$59.09 & $\leq$25.19 \\
        MDLM-Prime ($\ell=8$) &        & $\leq$32.20 & $\leq$26.01 & $\leq$53.27 & $\leq$38.05  & $\leq$65.63 & $\leq$218.99 & $\leq$28.63 \\
        \rowcolor{gray!15}
        MDLM-Prime ($\ell=8$) & \cmark & $\leq$\textbf{25.23} & $\leq$25.77 & $\leq$53.77 & $\leq$38.00 & $\leq$64.89 & $\leq$54.50 &  $\leq$25.79 \\
        \boldbottomrule
    \end{tabular}}
    \label{tab:apx:carry_over_zeroshot}
\end{table}

%% file: Tables/Appendix/carry_over.tex
\begin{wraptable}{r}{14em}
    \renewcommand{\arraystretch}{1}
    \newcommand{\boldtoprule}{\toprule[1.2pt]}
    \newcommand{\boldbottomrule}{\bottomrule[1.2pt]}
    \centering
    \small
    \vspace{-2em}
    \caption{Perplexity (PPL) evaluation on OWT. The symbol $\downarrow$ represents that lower values correspond to better performance.}
    \vspace{0.5em}
    \begin{tabular}{ccc}
        \boldtoprule
                   & \multicolumn{2}{c}{PPL ($\downarrow$)}  \\
        \cmidrule(lr){2-3}
        Carry-over & & \cmark \\
        \hline
        $\ell=2$   & $\leq$18.04 & $\leq$17.90 \\
        $\ell=3$   & $\leq$16.43 & $\leq$16.36 \\
        $\ell=4$   & $\leq$15.67 & $\leq$15.62 \\
        $\ell=6$   & $\leq$\textbf{15.43} & $\leq$\textbf{15.36} \\
        $\ell=8$   & $\leq$15.48 & $\leq$15.45 \\
        \boldbottomrule
    \end{tabular}
    \label{tab:apx:carry_over}
    \vspace{-1em}
\end{wraptable}

%% file: Tables/Appendix/time_eval.tex
\begin{table}[t]
\begin{minipage}{0.49\linewidth}
\newcommand{\boldtoprule}{\toprule[1.2pt]}
\newcommand{\boldbottomrule}{\bottomrule[1.2pt]}
\renewcommand{\arraystretch}{1.2}
\centering
\caption{The average runtime calculated based on ten training iterations. The results are reported in seconds per iteration. The 95\% confidence intervals of the results are around 0.05.}
\vspace{-0.5em}
\label{tab:time_eval}
\begin{center}
    \small
    \begin{tabular}{llcc}
        \boldtoprule
        &     & \multicolumn{2}{c}{Runtime (sec./iter.)} \\
        \cmidrule(lr){3-4}
        \multicolumn{2}{c}{Carry-over} &     & \cmark  \\
        \hline
        $\ell=1$, & $b=50,257$ &  1.06 e-1 & 1.14 e-1 \\
        $\ell=2$, & $b=225$ &  1.06 e-1 & 1.29 e-1 \\
        $\ell=3$, & $b=37$ &  1.11 e-1 & 1.34 e-1 \\
        $\ell=4$, & $b=15$ &  1.10 e-1 & 1.43 e-1 \\
        $\ell=6$, & $b=7$ &  1.13 e-1 & 1.75 e-1 \\
        $\ell=8$, & $b=4$ &  1.07 e-1 & 2.12 e-1 \\
        \boldbottomrule
    \end{tabular}
\end{center}
\end{minipage}\hfill
\begin{minipage}{0.49\linewidth}
\newcommand{\boldtoprule}{\toprule[1.1pt]}
\newcommand{\boldbottomrule}{\bottomrule[1.1pt]}
\renewcommand{\arraystretch}{1.175}
\vspace{-0.2em}
\caption{FID scores evaluated under different NFE on the CIFAR-10 and ImageNet-32 benchmarks. Lower values correspond to better performance. MDM-Prime is trained with $\ell=2$.}
\vspace{-0.35em}
\label{tab:cifar_imagenet_step}
\begin{center}
    \small
    % \resizebox{\linewidth}{!}{
    \begin{tabular}{cccc}
        \boldtoprule
            \multicolumn{4}{c}{CIFAR-10}  \\
        \hline
        NFE & 128 & 256 & 512  \\
        \hline
        MDM           & 7.55 & 5.00 & 4.66 \\
        MDM-Prime & \textbf{5.22} & \textbf{3.73} & \textbf{3.26} \\
        \hline
        \hline
         \multicolumn{4}{c}{ImageNet-32} \\
        \hline
        NFE    & 128 & 256 & 512 \\
        \hline
        MDM    & 9.55 & 8.24 & 8.12 \\
        MDM-Prime & \textbf{7.85} & \textbf{7.61} & \textbf{7.31} \\
        \boldbottomrule
    \end{tabular}%}
\end{center}
\end{minipage}
\end{table}

%% file: Tables/Appendix/co_eval_ablation.tex
\begin{table}[t]
    \renewcommand{\arraystretch}{1.1}
    \newcommand{\boldtoprule}{\toprule[1.3pt]}
    \newcommand{\boldbottomrule}{\bottomrule[1.3pt]}
    \centering
    \large
    \vspace{-0.25em}
    \caption{Ablation results on the validation perplexities of MDLM-Prime ($\ell=6$). The model is trained and evaluated with or without the carry-over parameterization. Lower perplexity indicates better performance.}
    \vspace{-0.25em}
    \resizebox{\linewidth}{!}{%
    \begin{tabular}{cccccccccc}
        \boldtoprule
        \multicolumn{2}{c}{Carry-over} & \multicolumn{8}{c}{Dataset} \\
        Train   &  Evaluation & OWT & LAMBADA & WikiText & PTB & LM1B & AG News & PubMed & ArXiv \\
        \cmidrule(lr){1-2} \cmidrule(lr){3-10}
         &        & $\leq$15.43 & $\leq$25.89 & $\leq$25.86 & $\leq$52.42 & $\leq$37.65 & $\leq$56.68  & $\leq$497.39 & $\leq$24.99 \\
         & \cmark & $\leq$15.41 & $\leq$\textbf{24.36} & $\leq$\textbf{24.25} & $\leq$\textbf{49.77} & $\leq$\textbf{37.35} & $\leq$\textbf{56.48}  & $\leq$63.58 & $\leq$\textbf{24.81} \\
         \cmark & \cmark & $\leq$\textbf{15.36} & $\leq$25.80 & $\leq$26.87 & $\leq$62.62 & $\leq$45.36 & $\leq$60.16  & $\leq$\textbf{59.09} & $\leq$25.19 \\
        \boldbottomrule
    \end{tabular}}
    \label{tab:apx:co_eval_ablation}
    \vspace{-0.5em}
\end{table}

%% file: Tables/Appendix/nfe.tex
\begin{table}[t]
    \renewcommand{\arraystretch}{1.1}
    \newcommand{\boldtoprule}{\toprule[1.2pt]}
    \newcommand{\boldbottomrule}{\bottomrule[1.2pt]}
    \centering
    % \small
    % \vspace{-2em}
    \caption{Runtime and effective NFE evaluation of MDLM and MDLM-Prime ($\ell=4$).}
    % \vspace{1em}
    \begin{tabular}{ccccccc}
        \boldtoprule
         & \multicolumn{3}{c}{MDLM} & \multicolumn{3}{c}{MDLM-Prime ($\ell=4$)} \\
        \hline
        Effective NFE & 768 & 1,003 & 1,023 & 768 & 1,003 & 1,023 \\ 
        Discretized Timesteps & 1,685 & 25,000 & 500,000 & 772 & 1,022 & 1,044 \\
        Runtime (sec.) & 8.87 & 23.30 & 159.81 & 8.72 & 12.31 & 12.58 \\
        \boldbottomrule
    \end{tabular}
    \label{tab:apx:nfe}
    % \vspace{-1em}
\end{table}

%% file: Figures/Appendix/curves.tex
\begin{figure}[t]
    \centering
    \footnotesize
    \vspace{-0.5em}
    \includegraphics[width=\linewidth]{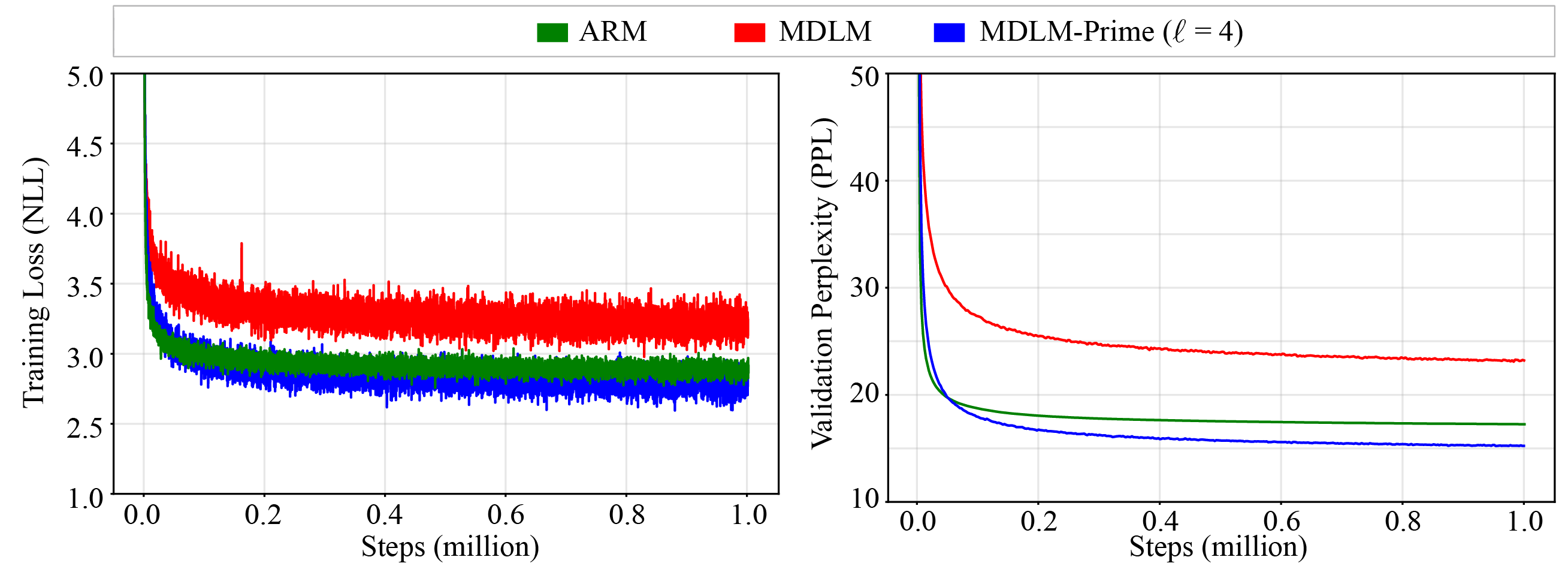}
    \vspace{-1.5em}
    \caption{The training and evaluation curves of ARM, MDM, and MDM-Prime ($\ell=4$).}
    \label{fig:curves}
    \vspace{-0.5em}
\end{figure}

%% file: Figures/Appendix/gen_ppl_eval.tex
\begin{wrapfigure}[]{r}{0.47\linewidth}
    \centering
    \vspace{-1.75em}
    \footnotesize
    \includegraphics[width=\linewidth]{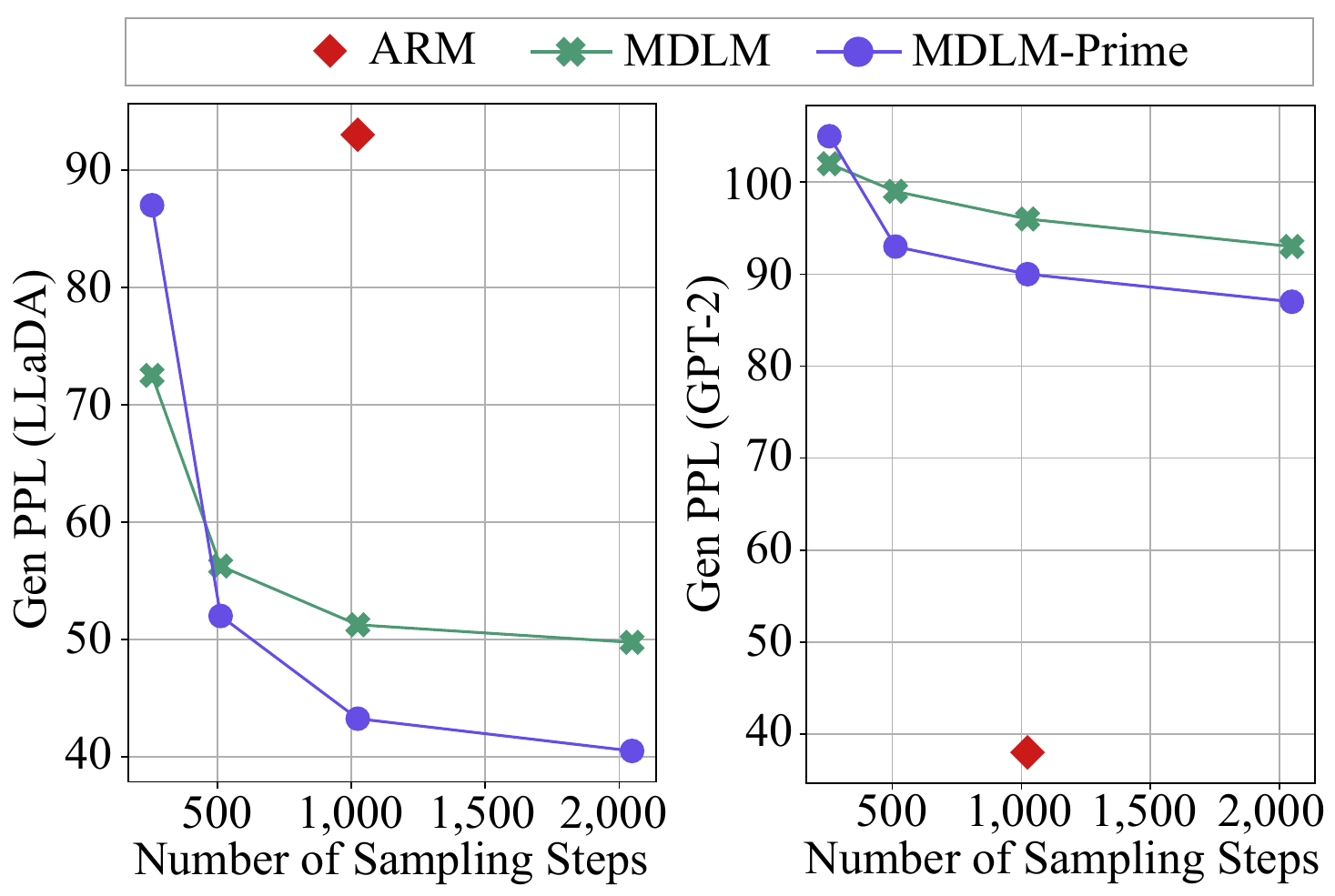}
    \vspace{-1.5em}
    \caption{Gen PPL evaluated using LLaDA-8B (left) and GPT-2 Large (right). Lower values correspond to better performance. MDLM is evaluated using their officially released code.}
    \vspace{-1em}
\label{fig:apx:gen_ppl_eval}
\end{wrapfigure}

%% file: Figures/Appendix/sampling_traj.tex
\begin{figure}[t]
    \centering
    \footnotesize
    % \vspace{-1em}
    \includegraphics[width=0.8\linewidth]{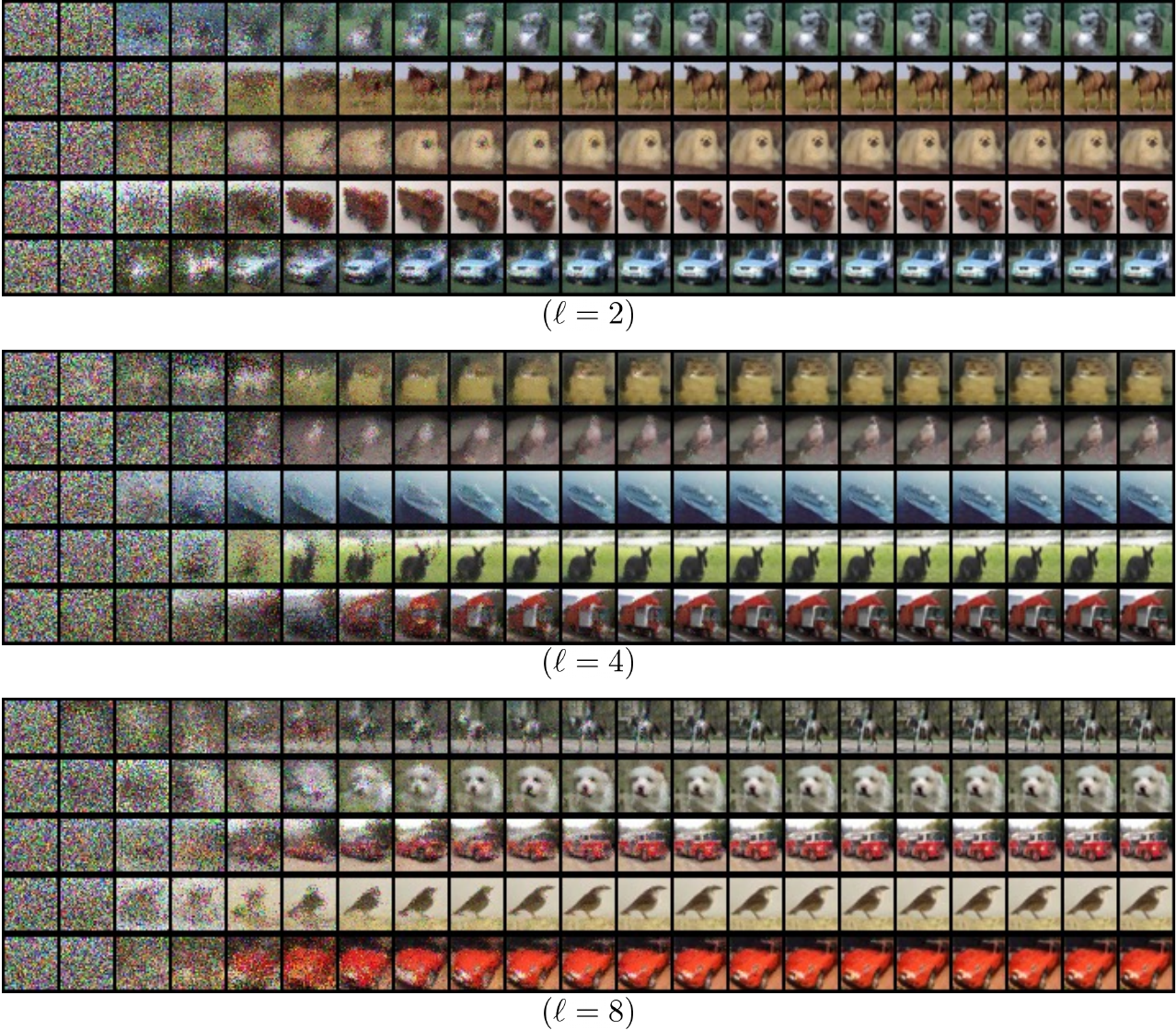}
    \vspace{-1em}
    \caption{Visualization of the sampling processes of MDM-Prime with $\ell=2$ (top), $\ell=4$ (middle), and $\ell=8$ (bottom). The models are trained on CIFAR-10.}
    \vspace{-0.5em}
    \label{fig:sampling_traj}
\end{figure}

%% file: Figures/Appendix/text_sampling_traj.tex
\begin{figure}[t]
    \centering
    \footnotesize
    \vspace{-1em}
    \includegraphics[width=\linewidth]{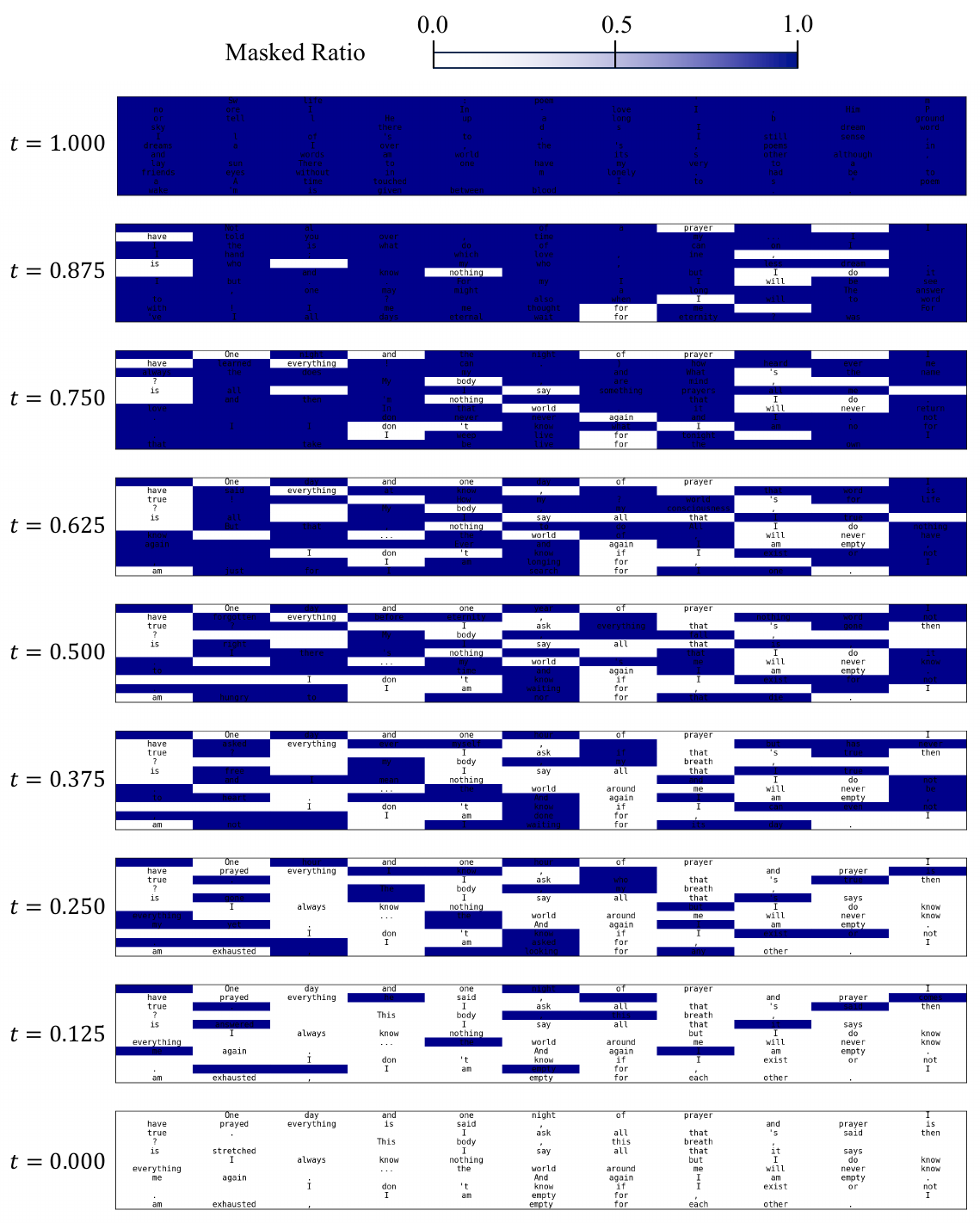}
    \vspace{-1em}
    \caption{Visualization of the sampling processes of MDLM.  The masked ratio is measured on a per-token basis, with higher values indicated by darker shades of blue. The samples are generated with prefix and suffix presented in Figs.~\ref{fig:apx:prefix} and~\ref{fig:apx:suffix}, respectively. Further experimental details are shown in Section~\ref{apx:experiments:qualitative}.}
    \vspace{-1.5em}
    \label{fig:text_sampling_traj_1bit}
\end{figure}

\begin{figure}[t]
    \centering
    \footnotesize
    \vspace{-1em}
    \includegraphics[width=\linewidth]{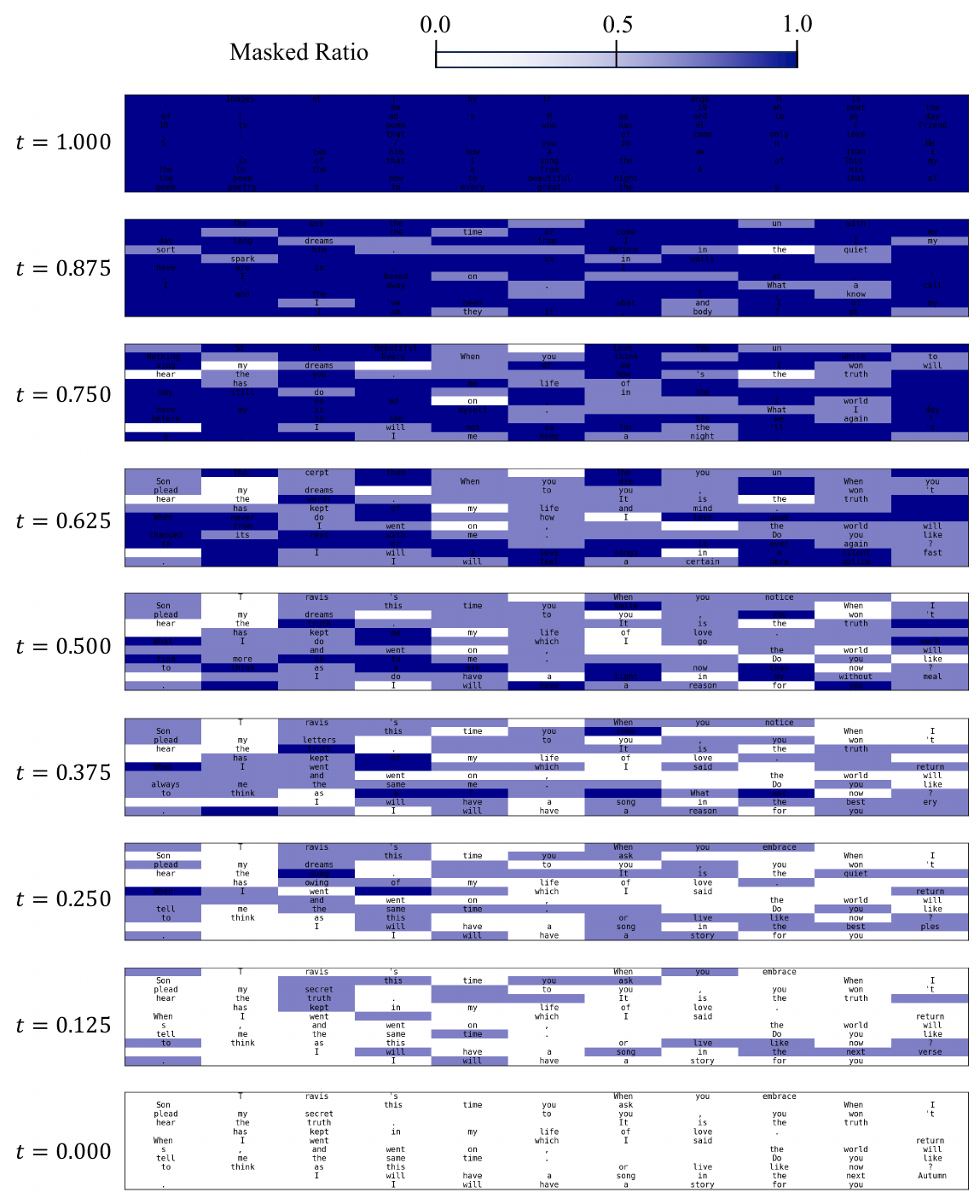}
    \vspace{-1em}
    \caption{Visualization of the sampling processes of MDLM-Prime ($\ell=2$).  The masked ratio is measured on a per-token basis, with higher values indicated by darker shades of blue. The samples are generated with prefix and suffix presented in Figs.~\ref{fig:apx:prefix} and~\ref{fig:apx:suffix}, respectively. Further experimental details are shown in Section~\ref{apx:experiments:qualitative}.}
    \vspace{-1.5em}
    \label{fig:text_sampling_traj_2bit}
\end{figure}

\begin{figure}[t]
    \centering
    \footnotesize
    \vspace{-1em}
    \includegraphics[width=\linewidth]{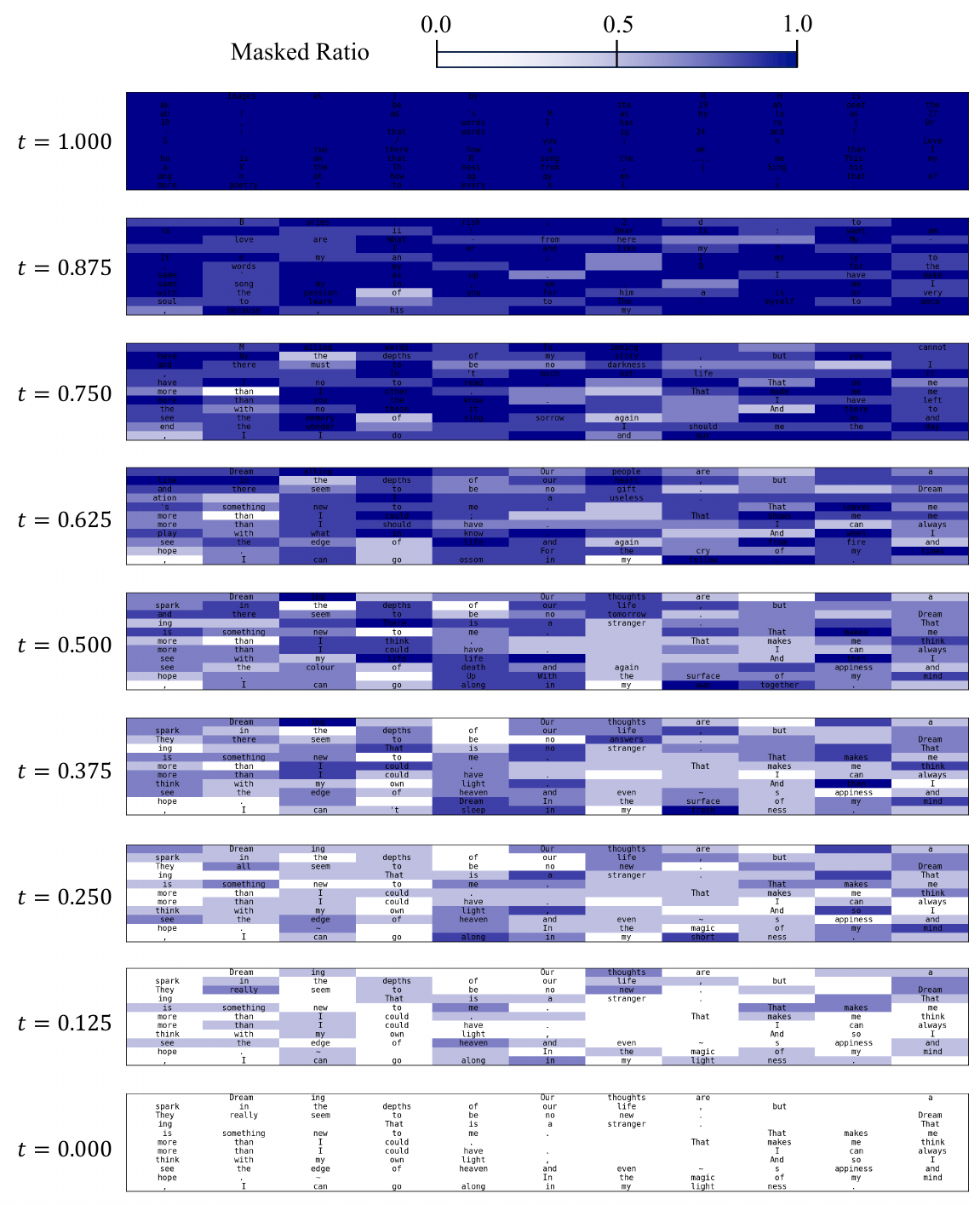}
    \vspace{-1em}
    \caption{Visualization of the sampling processes of MDLM-Prime ($\ell=4$).  The masked ratio is measured on a per-token basis, with higher values indicated by darker shades of blue. The samples are generated with prefix and suffix presented in Figs.~\ref{fig:apx:prefix} and~\ref{fig:apx:suffix}, respectively. Further experimental details are shown in Section~\ref{apx:experiments:qualitative}.}
    \vspace{-1.5em}
    \label{fig:text_sampling_traj_4bit}
\end{figure}

%% file: Figures/Appendix/text_sample.tex
\begin{figure}[t]
    \centering
    \footnotesize
    \setlength{\fboxsep}{0pt}
    \fcolorbox{gray}{white}{%
        \includegraphics[width=\linewidth]{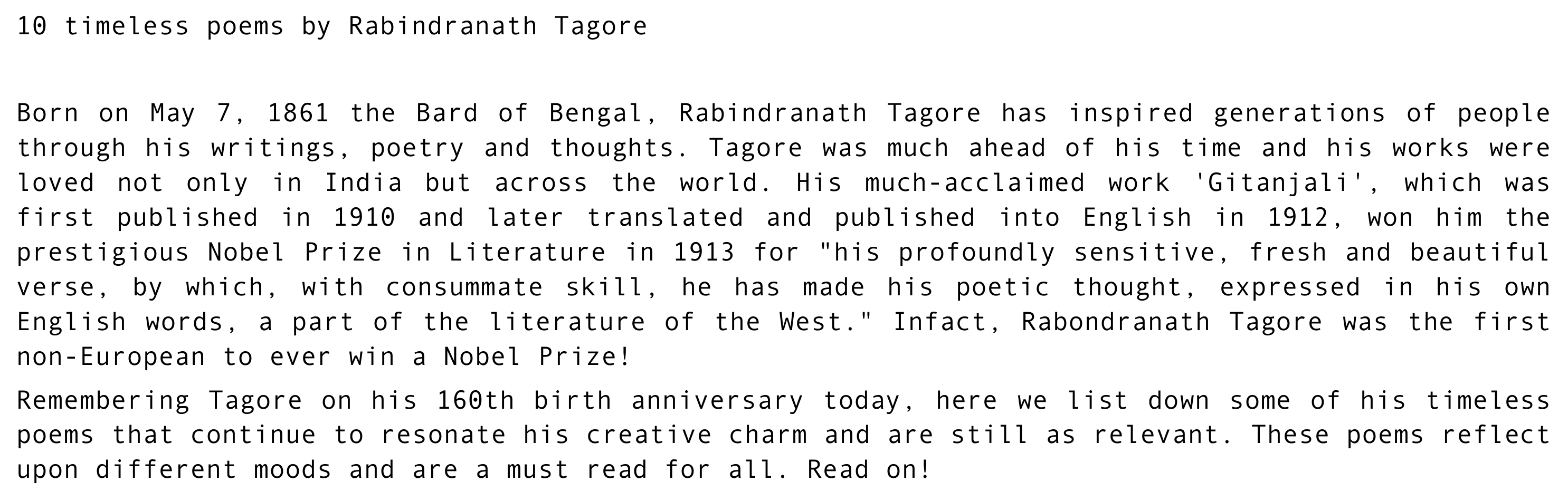}
    }
    \caption{Prefix text used for conditional sample generation.}
    \label{fig:apx:prefix}
\end{figure}

\begin{figure}[t]
    \centering
    \footnotesize
    \setlength{\fboxsep}{0pt}
    \fcolorbox{gray}{white}{%
        \includegraphics[width=\linewidth]{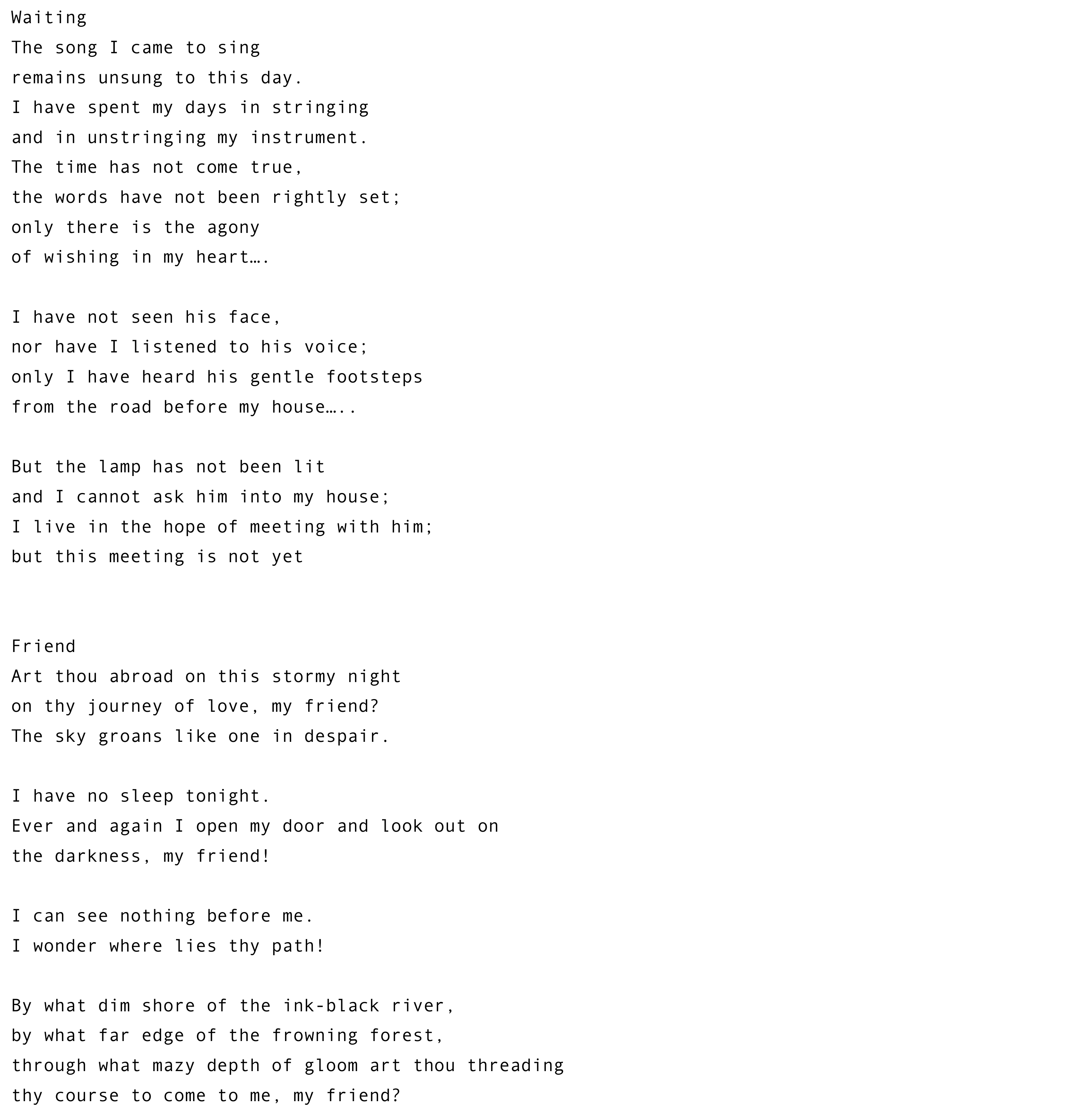}
    }
    \caption{Suffix text used for conditional sample generation.}
    \label{fig:apx:suffix}
\end{figure}

\clearpage
\begin{figure}[t]
    \centering
    \footnotesize
    \setlength{\fboxsep}{0pt}
    \fcolorbox{gray}{white}{%
        \includegraphics[width=\linewidth]{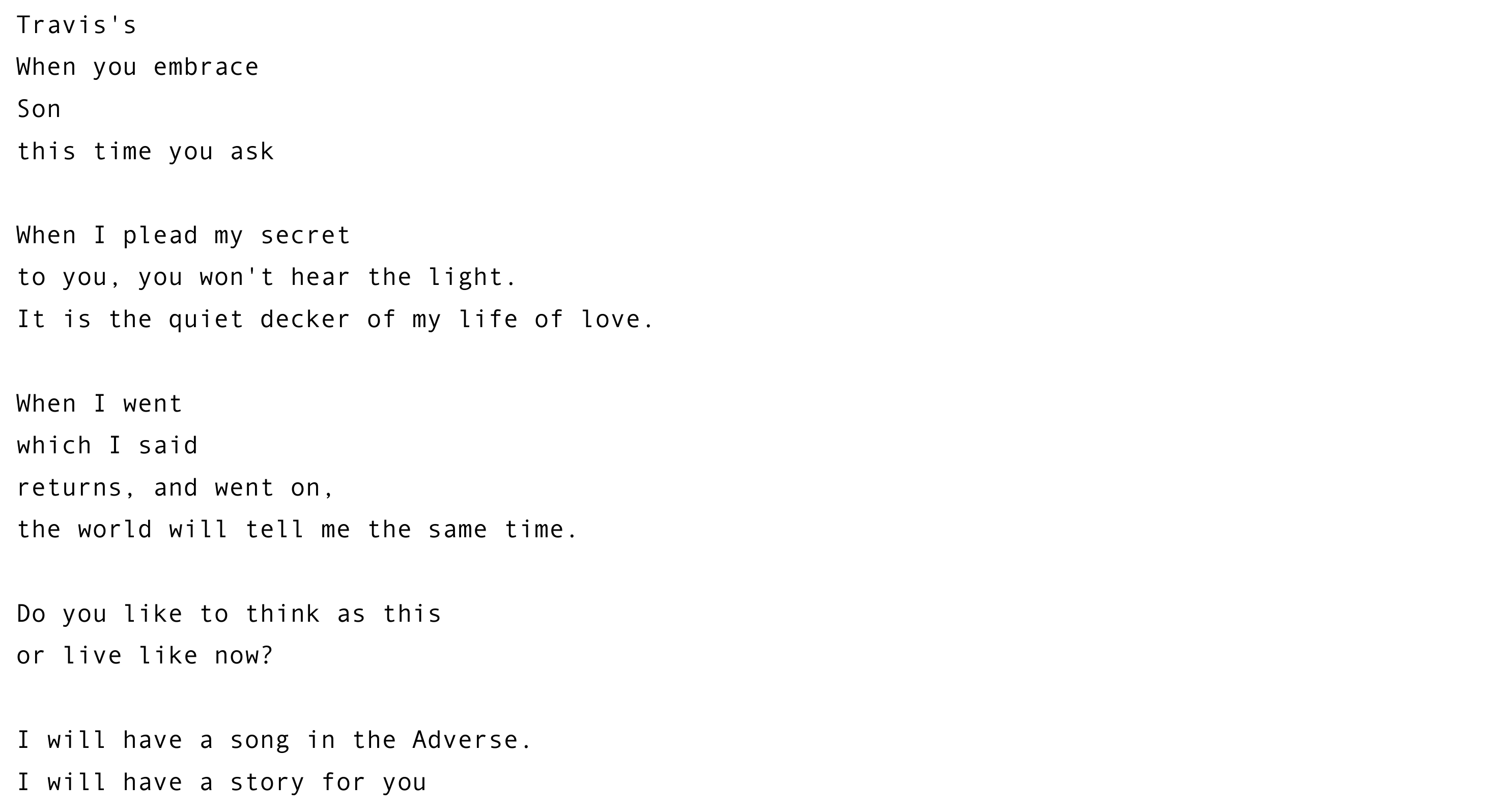}
    }
    \caption{Conditional sample generated using MDLM-Prime ($\ell=2$).}
    \label{fig:apx:poem_2bit}
\end{figure}

\begin{figure}[t]
    \centering
    \footnotesize
    \setlength{\fboxsep}{0pt}
    \fcolorbox{gray}{white}{%
        \includegraphics[width=\linewidth]{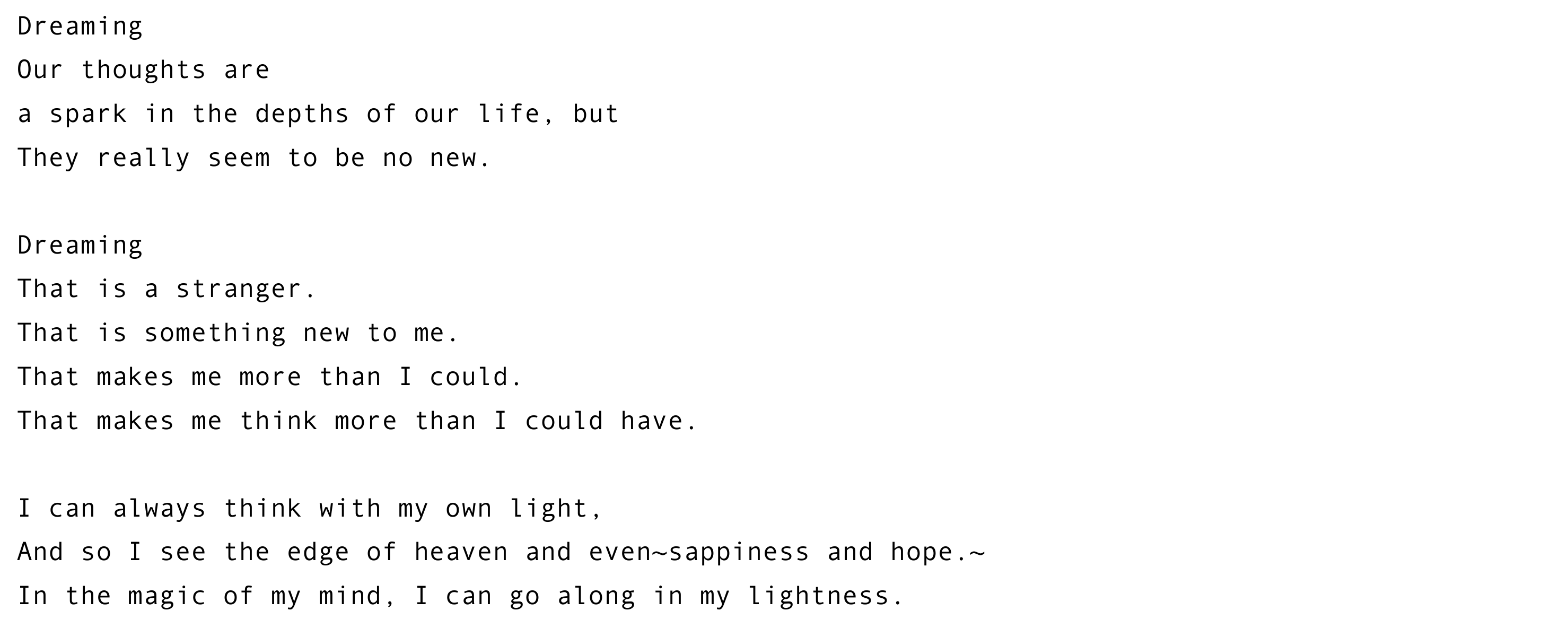}
    }
    \caption{Conditional sample generated using MDlM-Prime ($\ell=4$).}
    \label{fig:apx:poem_4bit}
\end{figure}

\begin{figure}[t]
    \centering
    \footnotesize
    \setlength{\fboxsep}{0pt}
    \fcolorbox{gray}{white}{%
        \includegraphics[width=\linewidth]{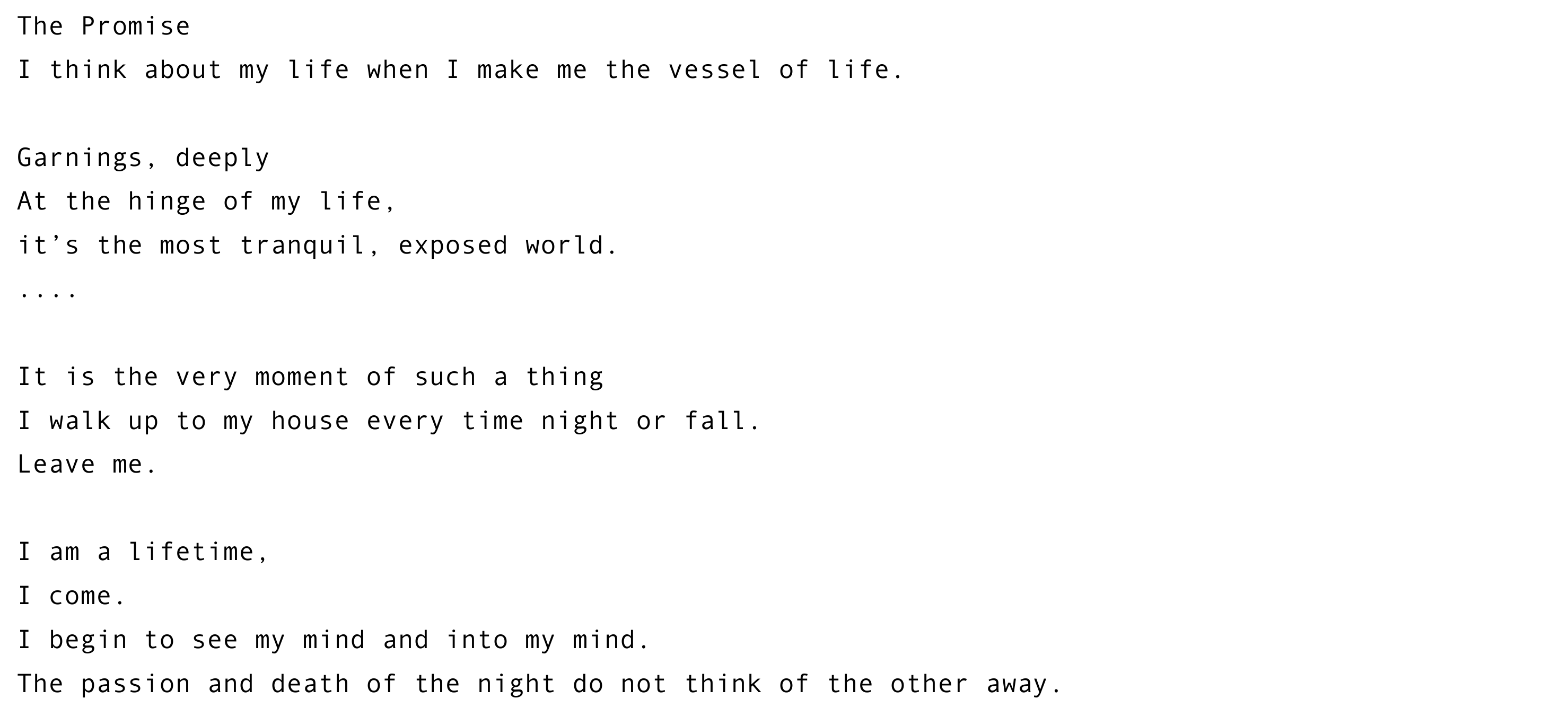}
    }
    \caption{Conditional sample generated using MDLM-Prime ($\ell=6$).}
    \label{fig:apx:poem_6bit}
\end{figure}

% \begin{figure}[t]
%     \centering
%     \footnotesize
%     \setlength{\fboxsep}{0pt}
%     \fcolorbox{gray}{white}{%
%         \includegraphics[width=\linewidth]{Figures/Files/cond_text_sample_1.pdf}
%     }
%     \caption{Conditional samples generated by MDM-Prime with 1,024 sampling steps. The conditional texts are highlighted in blue.}
%     \label{fig:cond_text_sample_1}
% \end{figure}

% \begin{figure}[t]
%     \centering
%     \footnotesize
%     \setlength{\fboxsep}{0pt} % optional: remove padding
%     \fcolorbox{gray}{white}{%
%         \includegraphics[width=\linewidth]{Figures/Files/cond_text_sample_2.pdf}
%     }
%     \caption{Conditional samples generated by MDM-Prime with 1,024 sampling steps. The conditional texts are highlighted in blue.}
%     \label{fig:cond_text_sample_2}
% \end{figure}

\begin{figure}[t]
    \centering
    \footnotesize
    \setlength{\fboxsep}{0pt} % optional: remove padding
    \fcolorbox{gray}{white}{%
        \includegraphics[width=\linewidth]{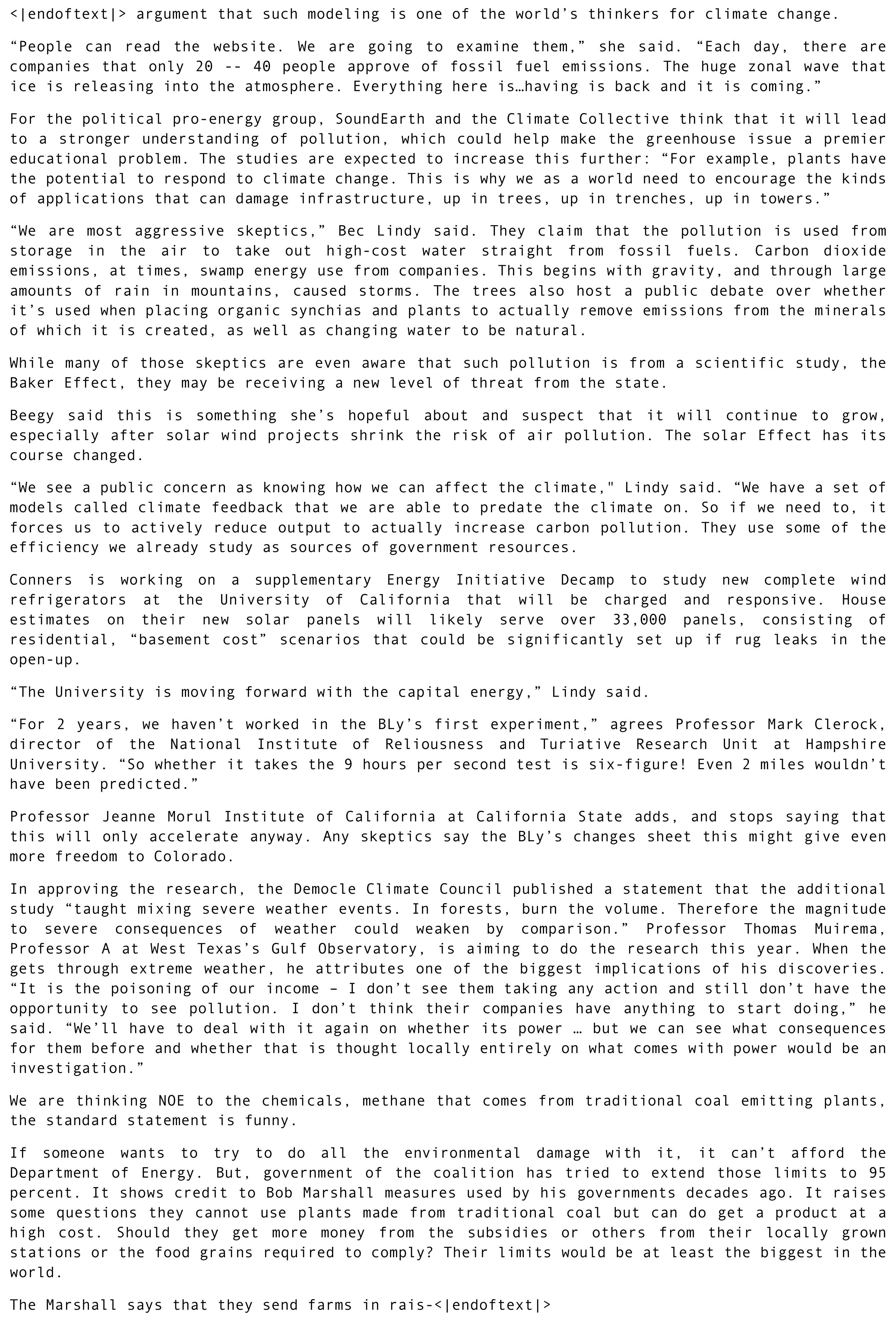}
    }
    \caption{Unconditional samples generated by MDLM-Prime with 1,024 sampling steps.}
    \label{fig:uncond_text_sample}
\end{figure}

%% file: Figures/Appendix/imputation_all.tex
\begin{figure}[t]
    \centering
    \footnotesize
    \includegraphics[width=\linewidth]{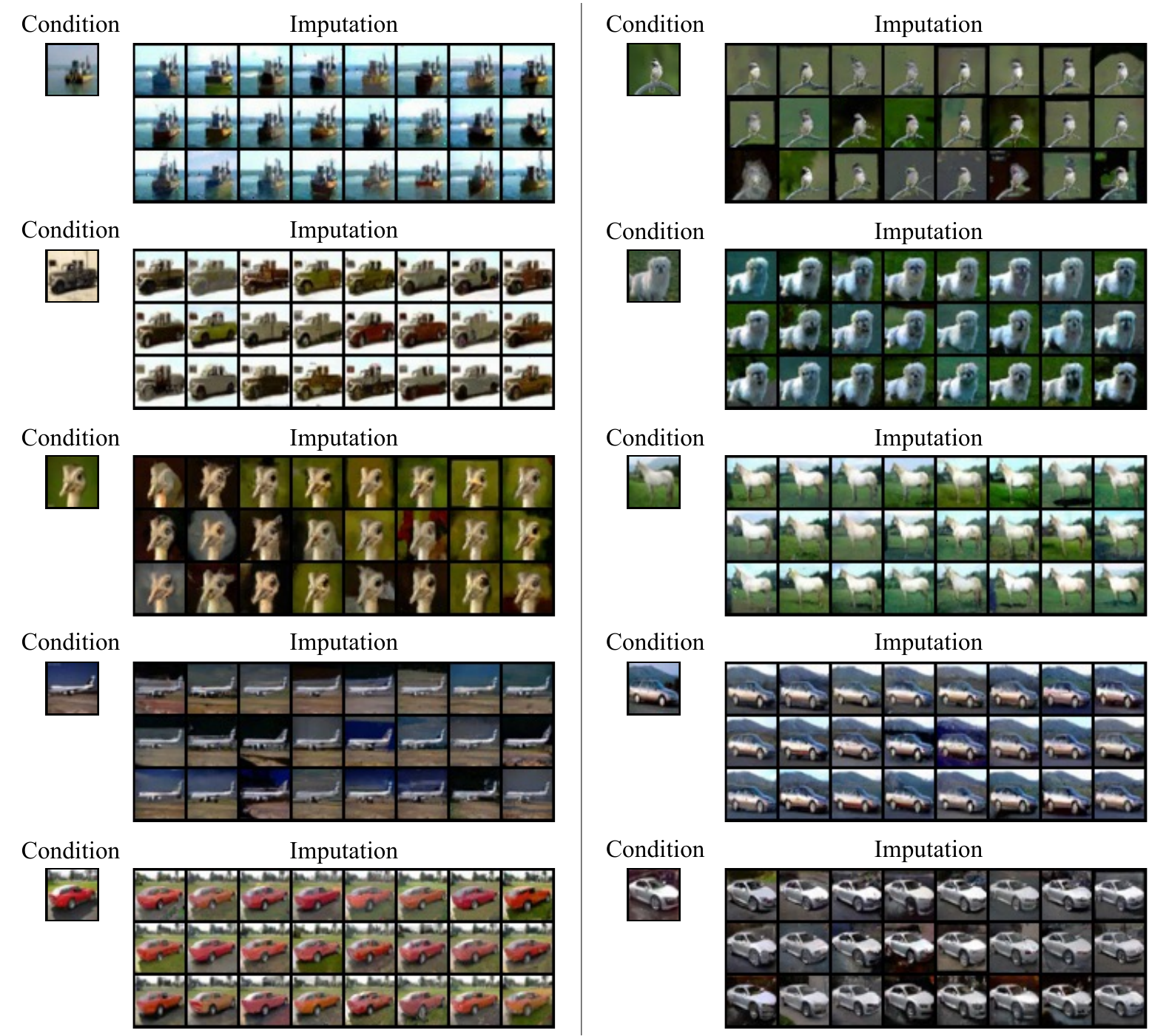}
    \caption{Uncurrated imputation results generated by MDM-Prime with CIFAR-10 images as conditions.}
    \label{fig:imputation_all}
\end{figure}

%% file: Figures/Appendix/cifar10.tex
% \begin{figure}[t]
%     \centering
%     \footnotesize
%     \includegraphics[width=0.75\linewidth]{Figures/Files/cifar10.png}
%     \caption{CIFAR-10 samples generated by MDM-Prime with NFE=512.}
%     \label{fig:cifar10}
% \end{figure}

\begin{figure}[t]
    \centering
    \footnotesize
    \includegraphics[width=0.7\linewidth]{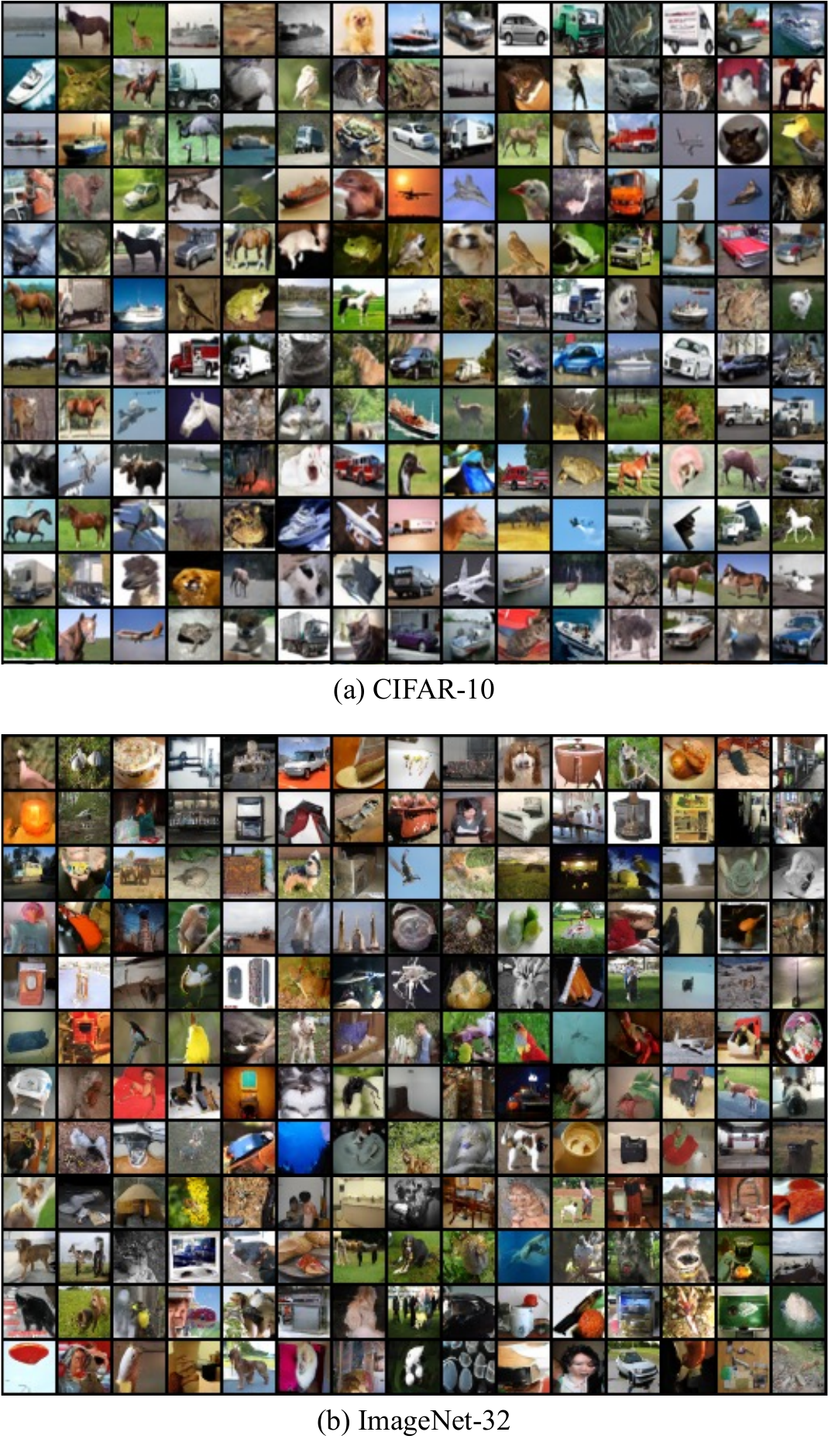}
    \caption{(a) CIFAR-10 and (b) ImageNet-32 samples generated by MDM-Prime with NFE=512.}
    \label{fig:cifar10}
\end{figure}

%% file: Figures/Appendix/imagenet.tex
% \begin{figure}[t]
%     \centering
%     \footnotesize
%     \includegraphics[width=0.75\linewidth]{Figures/Files/imagenet.png}
%     \caption{ImageNet-32 samples generated by MDM-Prime with NFE=512.}
%     \label{fig:imagenet}
% \end{figure}